\documentclass[twoside,11pt]{article}
\usepackage[abbrvbib, preprint]{jmlr2e}
\usepackage{amsfonts,amsmath,amssymb,amsthm}
\usepackage{bm,bbm}
\usepackage{booktabs}
\usepackage{graphicx}
\usepackage{here}
\usepackage{latexsym}
\usepackage{multirow}
\usepackage{newtxmath}
\usepackage{overpic}
\usepackage{subfigure}
\usepackage{xcolor}
\usepackage{todonotes}
\usepackage{hyperref}
\usepackage{comment}
\hypersetup{bookmarksnumbered=true, bookmarksopen=true, colorlinks=true, linkcolor=blue, citecolor=blue,}
\theoremstyle{plain}

\newtheorem{definition}{Definition}
\newtheorem{theorem}{Theorem}

\newtheorem{lemma}{Lemma}

\newtheorem{proposition}{Proposition}
\newtheorem{assumption}{Assumption}
\newtheorem{problem}{Problem}
\newcommand{\X}{{\protect\scalebox{0.5}{X}}}
\newcommand{\Y}{{\protect\scalebox{0.5}{Y}}}
\newcommand{\XY}{{\protect\scalebox{0.5}{X,Y}}}
\newcommand{\YX}{{\protect\scalebox{0.5}{Y|X}}}
\newcommand{\R}{\mathbb{R}}
\newcommand{\ba}{{\bm{a}}}
\newcommand{\bb}{{\bm{b}}}
\newcommand{\bi}{{\bm{i}}}
\newcommand{\bx}{{\bm{x}}}
\newcommand{\by}{{\bm{y}}}
\newcommand{\bu}{{\bm{u}}}
\newcommand{\bt}{{\bm{\theta}}}
\newcommand{\bX}{{\bm{X}}}
\newcommand{\bY}{{\bm{Y}}}

\newcommand{\calB}{{\mathcal{B}}}
\newcommand{\calC}{{\mathcal{C}}}
\newcommand{\calF}{{\mathcal{F}}}

\newcommand{\calL}{{\mathcal{L}}}
\newcommand{\calM}{{\mathcal{M}}}
\newcommand{\calN}{{\mathcal{N}}}
\newcommand{\calV}{{\mathcal{V}}}
\newcommand{\calX}{{\mathcal{X}}}
\newcommand{\calY}{{\mathcal{Y}}}

\newcommand{\rmA}{\mathrm{A}}
\newcommand{\rmB}{\mathrm{B}}

\newcommand{\rmI}{\mathrm{I}}

\newcommand{\rmM}{\mathrm{M}}
\newcommand{\rmO}{\mathrm{O}}
\newcommand{\rmP}{\mathrm{P}}
\newcommand{\rmQ}{\mathrm{Q}}
\newcommand{\rmR}{\mathrm{R}}
\newcommand{\rmV}{\mathrm{V}}
\newcommand{\opt}{\mathrm{opt}}
\newcommand{\naiseki}[1]{\langle #1 \rangle}
\newcommand{\tmin}{\mathop{\rm min}\limits}
\newcommand{\tst}{\mathop{\rm s.t.}\limits}
\newcommand{\tlimsup}{\mathop{\rm lim\,sup}\limits}
\newcommand{\tliminf}{\mathop{\rm lim\,inf}\limits}
\newcommand{\argmax}{\mathop{\rm arg\,max}\limits}

\DeclareMathOperator{\trace}{tr}
\DeclareMathOperator{\supp}{supp}
\DeclareMathOperator{\sign}{sign}
\DeclareMathOperator{\vecop}{vec}
\newcommand{\hyl}[1]{{(\hyperlink{{#1}}{{#1}})}}
\newcommand{\hyt}[1]{{\hypertarget{#1}{{\rm({#1})}}}}
\newcommand{\vo}{\phantom{1}}
\newcommand{\mr}[2]{\multirow{#1}{*}{#2}}

\newcommand{\CG}[1]{\textcolor[rgb]{0.4,0.4,0.4}{#1}}
\jmlrheading{x}{2023}{xx-xx}{xx/xx}{xx/xx}{yamasaki20xx}{Ryoya Yamasaki and Toshiyuki Tanaka}
\ShortHeadings{Optimal Kernel for Kernel-Based Modal Statistical Methods}{Yamasaki and Tanaka}
\firstpageno{1}
\begin{document}
\title{Optimal Kernel for Kernel-Based Modal Statistical Methods}
\author{%
\name Ryoya Yamasaki \email yamasaki@sys.i.kyoto-u.ac.jp\\
\addr Department of Systems Science\\
Graduate School of Informatics, Kyoto University\\
36-1 Yoshida-Honmachi, Sakyo-ku, Kyoto 606-8501 JAPAN
\AND
\name Toshiyuki Tanaka \email tt@i.kyoto-u.ac.jp\\
\addr Department of Systems Science\\
Graduate School of Informatics, Kyoto University\\
36-1 Yoshida-Honmachi, Sakyo-ku, Kyoto 606-8501 JAPAN}
\editor{xxx}
\maketitle
\begin{abstract}
Kernel-based modal statistical methods include mode estimation, regression, and clustering.
Estimation accuracy of these methods depends on the kernel used as well as the bandwidth.
We study effect of the selection of the kernel function to the estimation accuracy of these methods.
In particular, we theoretically show a (multivariate) optimal kernel that minimizes 
its analytically-obtained asymptotic error criterion when using an optimal bandwidth,
among a certain kernel class defined via the number of its sign changes.
\end{abstract}
\section{Introduction}
\label{section:Introduction}
The mode for a continuous random variable is a location statistic 
where a probability density function (PDF) takes a local maximum value.
It provides a good and easy-to-interpret summary of probabilistic events:
one can expect many sample points in the vicinity of a mode.
Besides, the mode has several desirable properties for data analysis.
One of them is that it is robust against outliers and skew or heavy-tailed distributions
as a statistic that expresses the centrality of data following a unimodal distribution.
The mean is not robust at all, which often causes troubles in real applications.
Another desirable property of the mode is that it can be naturally 
defined even in multivariate settings or non-Euclidean spaces.
For example, the median and quantiles are intrinsically difficult 
to define in these cases due to the lack of an appropriate ordering structure.
Also, the mode works well for multimodal distributions without loss of interpretability.

These attractive properties of the mode are utilized not only in location estimation 
but also in other data analysis techniques such as regression and clustering.
Among various modal statistical methods, kernel-based methods have been widely used, 
presumably because they retain the inherent goodness of the mode 
and are easy to use in that they can be implemented with simple estimation algorithms, 
such as mean-shift algorithm \citep{comaniciu2002mean, yamasaki2019ms} 
for kernel mode estimation \citep{parzen1962estimation, eddy1980optimum, romano1988weak} 
and for mode clustering \citep{casa2020},
and iteratively reweighted least squares algorithm \citep{yamasaki2020kernel} 
for modal linear regression \citep{yao2014new, kemp2019dynamic}.
Also, their good estimation accuracy is an important advantage:
for example, kernel mode estimation generally has the best estimation accuracy in mode estimation
in the sense that it attains the minimax-optimal rate regarding 
the used sample size \citep{hasminskii1979lower, donoho1991geometrizing}.

Estimation accuracy of such kernel-based methods 
depends on the kernel function used as well as the bandwidth.
Nevertheless, the selection of a kernel has not been well studied, 
and some questions remain unanswered.
For instance, for univariate kernel mode estimation, 
\citet{granovsky1991optimizing} have derived the optimal kernel, which 
minimizes the asymptotic mean squared error (AMSE) 
of a kernel mode estimate using an optimal bandwidth 
among a certain kernel class defined via the number of sign changes. 
Whether their argument can be extended to multivariate cases, 
however, has not been clarified yet.
We are also interested in optimal kernels for other kernel-based modal statistical methods,
such as modal linear regression and mode clustering, 
for which the problem of kernel selection has not been elucidated enough, as far as we are aware. 
The objective of this paper is to study the problem of kernel selection 
for various kernel-based modal statistical methods in the multivariate setting.

In the main part of this paper (Sections~\ref{section:KME}--\ref{section:Simulation}), 
with the aim of simplifying the presentation, we focus on kernel mode estimation 
\citep{parzen1962estimation, eddy1980optimum, romano1988weak} 
as a representative kernel-based modal statistical method.
In Section~\ref{section:KME}, we address multivariate extension of 
the study on the optimal kernel by \citet{granovsky1991optimizing}.
For multivariate cases, we need some additional assumptions on the structure of the kernels 
to allow systematic discussion on the optimal kernel,
so we mainly consider two commonly-used kernel classes, 
radial-basis kernels (RKs) and product kernels (PKs).
In Section~\ref{section:BAB}, we show 
basic asymptotic behaviors of the mode estimator, such as asymptotic normality of the mode estimator 
that can apply to general kernels including RKs and PKs, 
which is a modification of the existing statement in \citep{mokkadem2003law}. 
The characterization of the asymptotic behaviors leads to the AMSE of the mode estimator and the optimal bandwidth, 
providing the basis for our subsequent discussion on the optimal kernel.
We then construct theories for the optimal RK in Section~\ref{section:RBK},
and study PKs in Section~\ref{section:PK}.
As one of the consequences, the optimal RK is found to improve 
the AMSE by more than 10~\% compared with the commonly used Gaussian kernel, 
and the improvement is greater in higher dimensions.
In view of the slow nonparametric convergence rate of mode estimation, 
this improvement would significantly contribute to higher sample efficiency.
Moreover, on the basis of the consideration in Sections~\ref{section:RBK} and \ref{section:PK}, 
in Section~\ref{section:Comparison} we compare 
these two kernel classes in terms of the AMSE,
and show that, among non-negative kernels, the optimal RK is better than any PK
regardless of the underlying PDF.
Section~\ref{section:Discussions} gives some additional discussion.
In Section~\ref{section:Simulation} we show results of simulation experiments to examine 
to what extent the theories on the kernel selection (the results in Section~\ref{section:KME})
based on the asymptotics reflect the real performance in a finite sample size situation,
which supports the usefulness of our discussion.

We show that our discussion can be further adapted
to other kernel-based modal statistical methods, 
such another mode estimation method \citep{abraham2004asymptotic}, 
modal linear regression \citep{yao2014new, kemp2019dynamic}, 
and mode clustering \citep{comaniciu2002mean, casa2020}.
For these methods, we can obtain results similar to, but not completely the same 
as, those for the kernel mode estimation.
Section~\ref{section:OtherMethods} gives method-specific results on kernel selection.

The remaining part of this paper consists of a concluding section and appendix sections:
Section~\ref{section:Conclusion} summarizes our results 
and presents possible directions for future work.
Appendices give proofs of the theories on the optimal RK and of 
the theorems for the asymptotic behaviors of the kernel mode estimator 
and modal linear regression, 
which are presented newly in this paper.

\section{Optimal Kernel for Kernel Mode Estimation}
\label{section:KME}
\subsection{Basic Asymptotic Behaviors}
\label{section:BAB}
In general, for a continuous random variable, 
its mode refers either to a point where its PDF takes a local maximum value, 
or to a set of such local maximizers.
If we followed such definitions of the mode, the design of an estimator corresponding to each mode and theoretical statements regarding such an estimator will become complicated in a multimodal case.
So as to evade such a technical complication, we assume, unless stated otherwise, that the PDF has a unique global maximizer (see \hyl{A.3} in Appendix~\ref{section:ProofKME}). 
Let $\bX\in\R^d$ be a $d$-variate continuous random variable, 
and assume that it has a PDF $f$.
Then, we define the mode $\bt\in\R^d$ of $f$ as
\begin{align}
	\label{eq:def_mode}
	\bt\coloneq\argmax_{\bx\in\R^d} f(\bx).
\end{align}

Suppose that an independent and identically distributed (i.i.d.)~sample 
$\{\bX_i\in\R^d\}_{i=1}^n$ is drawn from $f$.
The kernel mode estimator (KME) is a plugin estimator
\begin{align}
	\label{eq:KMEdef}
	\bt_n\coloneq\argmax_{\bx\in\R^d} f_n(\bx),
\end{align}
where the kernel density estimator (KDE)
\begin{align}
\label{eq:KDE}%
	f_n(\bx)\coloneq\frac{1}{nh_n^d}
	\sum_{i=1}^nK\biggl(\frac{\bx-\bX_i}{h_n}\biggr)
\end{align}
is used as a surrogate for the PDF in~\eqref{eq:def_mode}, 
where $K$ is a kernel function defined on $\R^d$, 
and where $h_n>0$ is a parameter regarding the scale of the kernel and is called a bandwidth.
Here, we assume that the KDE also has a unique global maximizer.

Evaluating the KME amounts to solving the optimization problem~\eqref{eq:KMEdef} 
where the objective function $f_n$ is in general non-convex, so how one solves it 
is itself a very important problem. 
In this paper, however,
as we are mainly interested in elucidating properties of the KME, 
we assume that the KME is evaluated by a certain means, 
so that how to solve it is beyond the scope of this paper.

In what follows, for a set $S\subseteq\R$,
we let $S_{\ge0}$ denote the set of all nonnegative elements in $S$.
Thus, $\mathbb{Z}_{\ge0}=\{0,1,\ldots\}$, and $\R_{\ge0}$
represents the set of nonnegative real numbers. 
For later use, we introduce the multi-index notation:
For $\bi=(i_1,\ldots,i_d)^\top\in\mathbb{Z}_{\ge0}^d$,
let $|\bi|\coloneq\sum_{j=1}^di_j$.
It then follows that $\bi=\bm{0}_d$ and $|\bi|=0$ are equivalent, 
where $\bm{0}_d\coloneq(0,\ldots,0)^\top$ is the $d$-dimensional zero vector. 
For a multi-index $\bi\in\mathbb{Z}_{\ge0}^d$
and a vector $\bx=(x_1,\ldots,x_d)^\top$,
let $\bx^\bi \coloneq\prod_{j=1}^dx_j^{i_j}$.
The factorial of a multi-index is defined as
$\bi!\coloneq\prod_{j=1}^di_j!$.
Let $\partial^\bi \coloneq\prod_{j=1}^d\partial_j^{i_j}$,
where $\partial_j\coloneq\partial/\partial x_j$.

For a kernel function $K$ on $\R^d$
and a multi-index $\bi\in\mathbb{Z}_{\ge0}^d$, 
we let 
\begin{align}
	\calB_{d,\bi}(K)
	\coloneq \int_{\R^d} \bx^\bi K(\bx)\,d\bx
\end{align}
be the (multivariate) moment of $K$ indexed by $\bi$. 
Throughout this paper,
we assume using a \emph{$q$-th order kernel} $K$ of an even integer $q\ge2$.
Here, we say that $K$ defined on $\R^d$ is \emph{$q$-th order} 
if it satisfies the following \emph{moment condition}:
For $\bi\in\mathbb{Z}_{\ge0}^d$, 
\begin{align}
\label{eq:GenMomCon}%
	\calB_{d,\bi}(K)=
	\begin{cases}
	1,&\bi=\bm{0}_d,\\
	0,&\text{for all $\bi$ such that}\;|\bi|\in\{1,\ldots,q-1\},\\
	\text{non-zero},&\text{for some $\bi$ such that}\;|\bi|=q.
	\end{cases}
\end{align}
Also, note that the condition~\eqref{eq:GenMomCon} 
for $\bi=\bm{0}_d$ implies normalization of the kernel.

Here we consider the standard situation where the Hessian matrix $Hf(\bt)$ of the PDF at the mode is non-singular, 
where $H\coloneq\nabla\nabla^\top$. 
Existing works \citep{yamato1971sequential, ruschendorf1977consistency, mokkadem2003law} 
have revealed basic statistical properties of the KME such as consistency and asymptotic distribution.
In the following, we provide a theorem modified from that of, especially, \citep{mokkadem2003law}:
it shows the asymptotic bias (AB), variance-covariance matrix (AVC), 
and AMSE of $\bt_n$, along with asymptotic normality, 
for a more general kernel class than those covered by existing proofs.
This modification is essential for our purpose, 
since it covers the RKs that will be discussed in the next section, 
whereas the existing proofs do not.
\begin{theorem}
\label{theorem:AN}%
Assume Assumption \ref{assumption:asmAN} in Appendix~\ref{section:ProofKME} with an even integer $q\ge2$.
Then, the KME $\bt_n$ asymptotically follows a normal distribution with the following AB and AVC: 
\begin{align}
\label{eq:ABVC}%
	\mathrm{E}[\bt_n-\bt]\approx-h_n^q\rmA\bb,\quad
	\mathrm{Cov}[\bt_n]\approx\frac{1}{nh_n^{d+2}}\rmA\rmV\rmA,
\end{align}
where $\rmA\coloneq\{Hf(\bt)\}^{-1}$,
where $\bb=\bb(\bt;f,K)$ with
\begin{align}
	\bb(\bx;f, K)
	\coloneq \sum_{\bi\in\mathbb{Z}_{\ge0}^d:|\bi|=q}
	\frac{1}{\bi!} \cdot \nabla\partial^\bi f(\bx) \cdot \calB_{d,\bi}(K),
\end{align}
and where $\rmV$ is the $d\times d$ matrix defined as $\rmV\coloneq f(\bt)\calV_d(K)$ with
\begin{align}
\label{eq:FuncVofK}%
	\calV_d(K)
	\coloneq \int_{\R^d} \nabla K(\bx) \nabla K(\bx)^\top\,d\bx.
\end{align}
Moreover, the AMSE is given by
\begin{align}
\label{eq:AMSE}%
	\mathrm{E}[\|\bt_n-\bt\|^2]
	\approx h_n^{2q}\|\rmA\bb\|^2
	+\frac{1}{nh_n^{d+2}}\trace(\rmA\rmV\rmA),
\end{align}
where $\|\cdot\|$ is the Euclidean norm in $\R^d$ and 
where $\trace(\mathrm{M})$ denotes the trace of a square matrix $\mathrm{M}$.
\end{theorem}

A plausible approach to selecting the bandwidth $h_n$ and kernel $K$
is by making the AMSE~\eqref{eq:AMSE} as small as possible.
Noting $\rmA\bb\neq\bm{0}_d$ to hold under our assumptions (see \hyl{A.4} and \hyl{A.6}),
the stationary condition of the AMSE with respect to $h_n$ leads to the optimal bandwidth, 
\begin{align}
\label{eq:OptBan}%
	h_{d,q,n}^\opt
	=\left(\frac{(d+2)\trace(\rmA\rmV\rmA)}
	{2qn\|\rmA\bb\|^2}\right)^{\frac{1}{d+2q+2}},
\end{align}
which strictly minimizes the AMSE.%
\footnote{%
\label{ft:ft1}
The optimal bandwidth~\eqref{eq:OptBan} depends on the inaccessible PDF $f$ 
and its mode $\bt$, via $\rmA$, $\bb$, and $\rmV$, and thus it cannot be used in practice as it is.
However, it is possible to use a plugin estimator of the bandwidth 
which replaces these inaccessible quantities with their consistent estimates, 
and the discussion on the optimal kernel below holds also for the estimated optimal bandwidth.}
Moreover, substituting the optimal bandwidth into the AMSE, 
the bandwidth-optimized AMSE becomes
\begin{align}
\label{eq:OptAMSE}%
	\mathrm{E}[\|\bt_n-\bt\|^2\mid h_{d,q,n}^\opt]
	\approx \frac{2q}{d+2q+2}
	\|\rmA\bb\|^{\frac{2(d+2)}{d+2q+2}}
	\left(\frac{d+2}{2qn}\trace(\rmA\rmV\rmA)\right)^{\frac{2q}{d+2q+2}}.
\end{align}
The bandwidth-optimized AMSE~\eqref{eq:OptAMSE} depends
on the kernel function $K$ used in the KDE through $\{\calB_{d,\bi}(K)\}$ and $\calV_d(K)$,
so that one may consider further optimizing the bandwidth-optimized AMSE with respect to the kernel function.
However, it also depends on the sample-generating PDF $f$ via $\rmA,\bb,\rmV$, 
and moreover, the dependence on the kernel and that on the PDF are mixed in a complex way 
in the AMSE expression~\eqref{eq:OptAMSE}, 
so it seems quite difficult to study kernel optimization in the multivariate setting 
without additional assumptions. 
Furthermore, a resulting optimal kernel should depend on the PDF. 
A multivariate kernel without any structural assumptions is also difficult to implement in practice. 
We would thus introduce some structural assumptions on kernel functions,
and consider optimization of the bandwidth-optimized AMSE
under such assumptions.
Under appropriate structural assumptions on kernel functions
the optimal kernel function turns out not to depend on the PDF $f$,
as will be shown in the following sections.

\subsection{Optimal Kernel in Radial-Basis Kernels}
\label{section:RBK}
In multivariate kernel-based methods, RKs have been most commonly used.
We first consider the kernel optimization among the class of RKs.
We say that a kernel $K$ is an RK if there is a function $G$ on $\R_{\ge0}$ such that $K$ is represented as $K(\cdot)=G(\|\cdot\|)$.

For an RK $K(\cdot)=G(\|\cdot\|)$,
by means of the cartesian-to-polar coordinate transformation,
the functional $\calB_{d,\bi}(K)$ is rewritten in terms of $G$ as
\begin{align}
	\label{eq:bRK}
	\calB_{d,\bi}(K)=b_{d,\bi} B_{d,i}(G)
\end{align}
with $i=|\bi|$,
where $B_{d,i}(G)\coloneq\int_{\R_{\ge0}} x^{d-1+i}G(x)\,dx$, 
and where $b_{d,\bi}$ is a kernel-independent factor given by
\begin{align}
\begin{split}
	b_{d,\bi}
	\coloneq&\,
	\int_0^\pi \cos^{i_1}\xi_1 \sin^{d-2 + \sum_{j>1}i_j}\xi_1\,d\xi_1
	\int_0^\pi \cos^{i_2}\xi_2 \sin^{d-3 + \sum_{j>2}i_j}\xi_2\,d\xi_2 \cdots\\
	\times&\,
	\int_0^\pi \cos^{i_{d-2}}\xi_{d-2} \sin^{1 + \sum_{j>d-2}i_j}\xi_{d-2} \,d\xi_{d-2}
	\int_0^{2\pi} \cos^{i_{d-1}}\xi_{d-1} \sin^{i_d}\xi_{d-1} \,d\xi_{d-1}\\
	=&\,
	\begin{cases}
	\frac{2\prod_{j=1}^d\Gamma\bigl(\frac{1+i_j}{2}\bigr)}{\Gamma(\frac{d+i}{2})}\neq0,&i_1,\ldots,i_d\in2\mathbb{Z}_{\ge0},\\
	0,&\mbox{otherwise}.
	\end{cases}
\end{split}
\end{align}
Note that $b_{d,\bi}$, and so does $\calB_{d,\bi}(K)$, vanishes 
once any index in the multi-index $\bi$ is odd, which includes the case where $|\bi|$ is odd,
and that $b_{d,\bm{0}_d}$ is equal to $b_d\coloneq 2\pi^{d/2}/\Gamma(\frac{d}{2})$.
One then has $\bb=B_{d,i}(G)\bar{\bb}$, where
\begin{align}
\label{eq:bbar}%
	\bar{\bb}
	\coloneq
	\sum_{\bi\in\mathbb{Z}_{\ge0}^d:|\bi|=q}
	\frac{1}{\bi!}\cdot \nabla\partial^\bi f(\bt)\cdot b_{d,\bi}
	=\frac{\pi^{\frac{d}{2}}}{2^{q-1}\Gamma(\frac{d+q}{2})
    \Gamma(\frac{q}{2}+1)}
    \nabla\nabla_2^{q/2}f(\bt),
\end{align}
where $\nabla_l\coloneq\sum_{j=1}^d\frac{\partial^l}{\partial x_j^l}$
so that $\nabla_2$ denotes the Laplacian operator.

Also, the functional $\calV_d(K)$ for an RK $K(\cdot)=G(\|\cdot\|)$ reduces to
\begin{align}
	\calV_d(K)=v_d V_{d,1}(G) \rmI_d
\end{align}
with $V_{d,l}(G)\coloneq\int_{\R_{\ge0}} x^{d-1}\{G^{(l)}(x)\}^2\,dx$
and $v_d\coloneq b_{d,\{2,0,\ldots,0\}}=2\pi^{d/2}/(d \Gamma(\frac{d}{2}))$.
This implies that $\rmV=v_df(\bt)V_{d,1}(G)\rmI_d$,
that is, $\rmV$ is proportional to the identity matrix $\rmI_d$.

Therefore, the bandwidth-optimized AMSE~\eqref{eq:OptAMSE} becomes
\begin{align}
\label{eq:RKOptAMSE}%
\begin{split}
	\mathrm{E}[\|\bt_n-\bt\|^2\mid h_{d,q,n}^\opt]
	&\approx \frac{2q}{d+2q+2}
	\|\rmA\bar{\bb}\|^{\frac{2(d+2)}{d+2q+2}}
	\left(\frac{d+2}{2qn} v_d f(\bt) \trace(\rmA^2)\right)^{\frac{2q}{d+2q+2}}\\
	&\times 
	\Bigl(B_{d,q}^{2(d+2)}(G)\cdot V_{d,1}^{2q}(G)\Bigr)^{\frac{1}{d+2q+2}},
\end{split}
\end{align}
in which the kernel-dependent factor is $(B_{d,q}^{2(d+2)}\cdot V_{d,1}^{2q})^{\frac{1}{d+2q+2}}$ alone.%
\footnote{%
Although implicit in the notations here,
$B_{d,i}$ and $V_{d,1}$ are functionals of $G$.
We will use the notation $B_{d,i}(G)$ etc.~when 
we want to make the dependence on $G$ explicit,
but also use such abbreviations otherwise.}

As one can see, the bandwidth-optimized AMSE~\eqref{eq:RKOptAMSE} 
for an RK is decomposed into a product of 
the kernel-dependent factor, which we call the \emph{AMSE criterion},
and the remaining PDF-dependent part.
This fact implies that an RK minimizing the AMSE criterion 
also minimizes the AMSE, that is it is optimal, 
among the class of RKs for every PDF satisfying the requirements of Theorem~\ref{theorem:AN}.

For an RK $K(\cdot)=G(\|\cdot\|)$, as stated above, 
$\calB_{d,\bi}(K)$ vanishes for $\bi$ with odd $|\bi|$, 
so that one needs to consider the even moment condition alone, 
which can be translated into the moment condition for $G$ as follows:
\begin{align}
\label{eq:RKMomCon}%
	B_{d,i}(G)=
	\begin{cases}
	b_d^{-1},&i=0,\\
	0,&i\in I_{q-2},\\
	\text{non-zero},&i=q,
	\end{cases}
\end{align}
where $I_k$ is the index set defined as
\begin{align}
\label{eq:Ik}%
	I_k\coloneq\bigl\{k-2i:i=0,1,\ldots,\bigl\lfloor\tfrac{k}{2}\bigr\rfloor\bigr\}.
\end{align}
Thus, if one focuses on the functional form of the AMSE criterion, 
one might consider the following variational problem (we name it `type-1' because 
it depends on the first derivative of $G$ via $V_{d,1}^{2q}(G)=V_{d,0}^{2q}(G^{(1)})$).
\begin{problem}[$d$-variate, type-1, $q$-th order]
\label{problem:Prob1}
\begin{align}
	\tmin_G\quad
	&B_{d,q}^{2(d+2)}(G)\cdot V_{d,1}^{2q}(G),
\tag{\hypertarget{P1}{P1}}\\
	\mathrm{s.t.}\quad
	&G\;\text{satisfies the moment condition~\eqref{eq:RKMomCon}}.
\tag{\hypertarget{P1-1}{P1-1}}
\end{align}
\end{problem}

However, this problem has no solution, as shown in the next proposition.
\begin{proposition}
\label{prop:nosol}
	Problem~\ref{problem:Prob1} has no solution.
\end{proposition}
\begin{proof}
One has $B_{d,q}^{2(d+2)}(G)\cdot V_{d,1}^{2q}(G)\ge0$ by definition. 
It should never be equal to zero under the moment condition~\eqref{eq:RKMomCon}. 
Indeed, if one had $B_{d,q}^{2(d+2)}(G)\cdot V_{d,1}^{2q}(G)=0$, 
either $B_{d,q}(G)$ or $V_{d,1}(G)$ should be equal to zero, 
whereas $B_{d,q}(G)$ should be nonzero due to the moment condition~\eqref{eq:RKMomCon}. 
$V_{d,q}(G)$ cannot be equal to zero either, since otherwise 
$G^{(1)}$ should be equal to zero identically, implying that $G$ is a constant, 
contradicting the condition $B_{d,0}(G)=b_d^{-1}<\infty$.

We next show that the infimum of $B_{d,q}^{2(d+2)}(G)\cdot V_{d,1}^{2q}(G)$ is zero. 
Take $G_1$ and $G_2$ satisfying~\eqref{eq:RKMomCon} 
with $B_{d,q}(G_1)\not=B_{d,q}(G_2)$ 
and $V_{d,1}(G_1),V_{d,1}(G_2)<\infty$
(for example, $G_1=G_{d,q}^B$ and $G_2=G_{d,q}^E$ defined later).
Consider the linear combination $G=wG_1+(1-w)G_2$ 
of $G_1$ and $G_2$ with $w\in\R$. 
One then has $B_{d,q}(G)=wB_{d,q}(G_1)+(1-w)B_{d,q}(G_2)$, 
so that $G$ satisfies~\eqref{eq:RKMomCon} 
unless $w=w_0\coloneq \frac{B_{d,q}(G_2)}{B_{d,q}(G_2)-B_{d,q}(G_1)}$, 
in which case one has $B_{d,q}(G)=0$ so that the order of $G$ is 
strictly higher than $q$. 
On the other hand, substituting $(s,t)=(wG_1^{(1)}(x),(1-w)G_2^{(1)}(x))$ 
into the inequality $(s+t)^2\le2(s^2+t^2)$, 
multiplying both sides with $x^{d-1}$, 
and taking the integral over $x\in\R_{\ge0}$, 
one has 
\begin{align}
	V_{d,1}(G)\le2\bigl(w^2V_{d,1}(G_1)+(1-w)^2V_{d,1}(G_2)\bigr)<\infty
\end{align}
for any $w$. 
One therefore has $B_{d,q}^{2(d+2)}(G)\cdot V_{d,1}^{2q}(G)\to0$ 
as $w\to w_0$, proving that the infimum of $B_{d,q}^{2(d+2)}(G)\cdot V_{d,1}^{2q}(G)$ 
is equal to zero. 
\end{proof}

This proposition states that the approach to kernel selection via 
simply minimizing the AMSE criterion among the entire class of $q$-th order RKs 
is not a theoretically viable option. 
One may be able to discourage this approach from a practical standpoint as well. 
As shown above, one can make the AMSE criterion arbitrarily close to zero 
by considering a sequence of kernels for which $B_{d,q}$ approaches zero. 
For such a sequence of kernels, however, 
the optimal bandwidth $h_{d,q,n}^\opt\propto\{V_{1,d}/(nB_{d,q})^2\}^{\frac{1}{d+2q+2}}$ 
diverges toward infinity. 
It then causes the effect of the next leading order term of the AB, which is $O(h_n^{q+2})$, 
to be non-negligible in the MSE when the sample size $n$ is finite, 
thereby invalidating the use of the AMSE criterion to estimate 
the MSE in a finite sample situation.

The above discussion suggests that, in order to avoid diverging behaviors of the optimal bandwidth 
and to obtain results meaningful in a finite sample situation, 
one should consider optimization of the AMSE criterion 
within a certain kernel class which at least excludes those kernels with 
vanishing leading-order moments. 
Since such an approach can give a kernel with the smallest AMSE criterion at least among the considered class,
it would be superior to other kernel design methods that are not based on 
any discussion of the goodness of resulting higher-order kernels in kernel classes of the same order, 
such as jackknife \citep{schucany1977improvement, wand1990gaussian},
as well as those that optimize alternative criteria other than the AMSE criterion,
such as the minimum-variance kernels~\citep{eddy1980optimum, muller1984smooth}.
The existing researches \citep{gasser1979kernel, gasser1985kernels, granovsky1989optimality, granovsky1991optimizing} 
can be interpreted as following this approach, and they
have focused on the relationship between the order of the kernel and the number of sign changes.
More concretely, on the basis of the observation that 
any $q$-th order univariate kernel should change its sign at least $(q-2)$ times on $\R$, 
they considered kernel optimization among the class of $q$-th order kernels 
with exactly $(q-2)$ sign changes (which we call the minimum-sign-change condition). 
They have also shown that the minimum-sign-change condition excludes kernels 
with small leading-order moments, which cause the difficulties discussed above 
in the minimization of the AMSE criterion. 
In addition, and importantly, they succeeded in obtaining a closed-form solution 
of the resulting kernel optimization problem under the minimum-sign-change condition.

In this paper, we also take the same approach as the one by
\citep{gasser1979kernel, gasser1985kernels, granovsky1989optimality, granovsky1991optimizing}:
under the moment condition~\eqref{eq:RKMomCon},
one can show that $G$ changes its sign at least $(\frac{q}{2}-1)$ times on $\R_{\ge0}$; 
see Lemma~\ref{lemma:MinCro}.
Then, we consider the following modified problem to which 
the minimum-sign-change condition~\hyl{P2-3} is added.  
\begin{problem}[Modified $d$-variate, type-1, $q$-th order]
\label{problem:Prob21}%
\begin{align}
	\tmin_G\quad
	&B_{d,q}^{2(d+2)}(G)\cdot V_{d,1}^{2q}(G),
\tag{\hypertarget{P2}{P2}}\\
	\mathrm{s.t.}\quad
	&G\;\text{satisfies the moment condition~\eqref{eq:RKMomCon}},
\tag{\hypertarget{P2-1}{P2-1}}\\
	&G\;\text{is differentiable and}\;G^{(1)}\;\text{is integrable},
\tag{\hypertarget{P2-2}{P2-2}}\\
	&G\;\text{changes its sign}\;(\tfrac{q}{2}-1)\;\text{times on}\;\R_{\ge0},
\tag{\hypertarget{P2-3}{P2-3}}\\
	&V_{d,1}(G)<\infty,
\tag{\hypertarget{P2-4}{P2-4}}\\
	&x^{d+q}G(x)\to0\;\text{as}\;x\to\infty.
\tag{\hypertarget{P2-5}{P2-5}}
\end{align}
\end{problem}
\noindent%
Here, the conditions~\hyl{P2-2}--\hyl{P2-5} are technically required in the proof. 
The conditions \hyl{P2-2} and \hyl{P2-4} for making the problem well-defined and 
the condition \hyl{P2-5} which is close to $|B_{d,q}(G)|<\infty$ would be almost non-restrictive.
The minimum-sign-change condition \hyl{P2-3} for $q=2$ implies that the kernel is non-negative,
under the normalization condition in \hyl{P2-1}.
Thus, the minimum-sign-change condition~\hyl{P2-3} for higher orders $q\ge4$ 
can be regarded as a sort of extension of the non-negativity condition.

In the same way as that by \citep{gasser1979kernel, gasser1985kernels},
we can show when $q=2,4$  that this problem is solvable and the solution is provided as follows.
\begin{theorem}
\label{theorem:thm21}%
For every $d\ge1$ and $q=2, 4$,
a solution of Problem~\ref{problem:Prob21}{\footnotemark} is
\begin{align}
\label{eq:Opt-RK1}%
	G_{d,q}^B(x)\coloneq\biggl(\sum_{i\in I_{q+2}}c^B_{d,q,i}x^i\biggr)\mathbbm{1}_{x\le1},
\end{align}
where the coefficients $c^B_{d,q,i}$ with $i\in I_{q+2}$ are given by
\begin{align}
\label{eq:Opt-RK1-Coe}
	c^B_{d,q,i}=
	\frac{(-1)^{\frac{i}{2}}\Gamma(\frac{d+q+4}{2})\Gamma(\frac{d+q+i}{2})}
	{\pi^{\frac{d}{2}}\Gamma(\frac{q}{2})\Gamma(\frac{2+i}{2})
	\Gamma(\frac{d+2+i}{2})\Gamma(\frac{q+4-i}{2})}.
\end{align}
The kernel $G_{d,q}^B$ is also represented as
\begin{align}
\label{eq:Opt-RK2}%
\begin{split}
	G_{d,q}^B(x)
	&=\frac{\Gamma(\frac{d+q}{2})}{\pi^{\frac{d}{2}}\Gamma(\frac{q}{2})}
	P_{\frac{q}{2}+1}^{(\frac{d}{2},-2)}(1-2x^2)\mathbbm{1}_{x\le1}\\
	&=(-1)^{\frac{q}{2}+1}
	\frac{\Gamma(\frac{d+q}{2}+2)}{\pi^{\frac{d}{2}}\Gamma(\frac{q}{2}+2)}
	(1-x^2)^2P_{\frac{q}{2}-1}^{(2,\frac{d}{2})}(2x^2-1)\mathbbm{1}_{x\le1},
\end{split}
\end{align}
where $\Gamma(\cdot)$ is the gamma function and where $P_n^{(\alpha,\beta)}(\cdot)$ 
is the Jacobi polynomial \citep{szeg1939orthogonal} defined by 
\begin{align}
\label{eq:JACOBI}%
	P_n^{(\alpha,\beta)}(x)
	\coloneq\frac{(-1)^n}{2^nn!}(1-x)^{-\alpha}(1+x)^{-\beta}
	\frac{d^n}{dx^n}\left[(1-x)^{n+\alpha}(1+x)^{n+\beta}\right].
\end{align}
\end{theorem}
\footnotetext{
\label{ft:ft3}
Note that a solution of Problem~\ref{problem:Prob21} has freedom on its scale, 
i.e., for any $s>0$, $G(x)=s^{-d}G_{d,q}^B(x/s)$ is also a solution of Problem~\ref{problem:Prob21}.}

Alternatively, we can consider the following problem 
with another minimum-sign-change condition~\hyl{P3-3}
that is defined for the first derivative of a kernel. 
\begin{problem}[Modified $d$-variate, type-1, $q$-th order]
\label{problem:Prob22}%
\begin{align}
	\tmin_G\quad
	&B_{d,q}^{2(d+2)}(G)\cdot V_{d,1}^{2q}(G),
\tag{\hypertarget{P3}{P3}}\\
	\mathrm{s.t.}\quad
	&G\;\text{satisfies the moment condition~\eqref{eq:RKMomCon}},
\tag{\hypertarget{P3-1}{P3-1}}\\
	&G\;\text{is differentiable and}\;G^{(1)}\;\text{is integrable},
\tag{\hypertarget{P3-2}{P3-2}}\\
	&G^{(1)}\;\text{changes its sign}\;(\tfrac{q}{2}-1)\;\text{times on}\;\R_{\ge0},
\tag{\hypertarget{P3-3}{P3-3}}\\
	&V_{d,1}(G)<\infty,
\tag{\hypertarget{P3-4}{P3-4}}\\
	&\text{there exists }\delta>0\text{ s.t.~}x^{d+q+\delta}G(x)\to0\;\text{as}\;x\to\infty.
\tag{\hypertarget{P3-5}{P3-5}}
\end{align}
\end{problem}
\noindent%
The minimum-sign-change condition \hyl{P3-3} for $q=2$ implies that the kernel is non-negative and non-increasing, 
under the normalization condition in \hyl{P3-5} and the end-point condition \hyl{P3-5}.
Therefore, the minimum-sign-change condition for the first derivative \hyl{P3-3} can be seen as 
a more restrictive version of the minimum-sign-change condition for the kernel itself \hyl{P2-3}.
Additionally note that the change of \hyl{P2-5} to \hyl{P3-5} stems from the difference 
in proof techniques used for Theorems~\ref{theorem:thm21} and \ref{theorem:thm22}.

We can show that this problem has the same solution even for general even orders $q=2,4,6,\ldots$,
on the basis of the proof techniques in \citep{granovsky1989optimality, granovsky1991optimizing}.
\begin{theorem}
\label{theorem:thm22}%
For every $d\ge1$ and even $q\ge2$, a solution of Problem~\ref{problem:Prob22} 
is \eqref{eq:Opt-RK1} or equivalently \eqref{eq:Opt-RK2}.
\end{theorem}

\begin{figure}[t]
\centering
\begin{overpic}[height=2.75cm, bb=0 0 495 285]{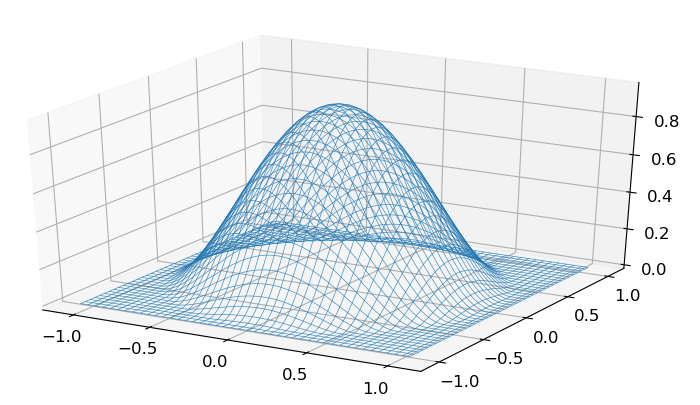}\put(3,48){\small $q=2$}\end{overpic}
\begin{overpic}[height=2.75cm, bb=0 0 495 285]{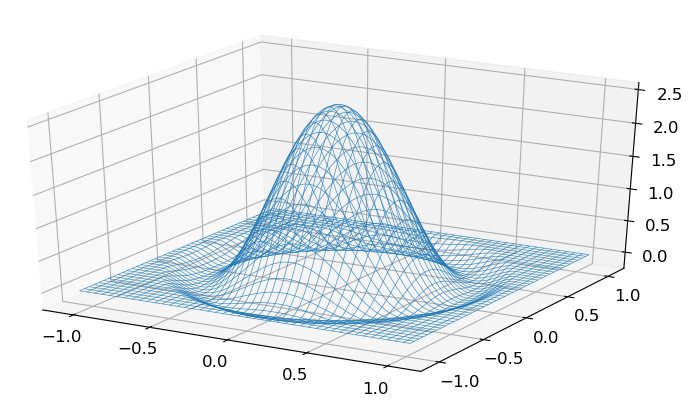}\put(3,48){\small $q=4$}\end{overpic}
\begin{overpic}[height=2.75cm, bb=0 0 495 285]{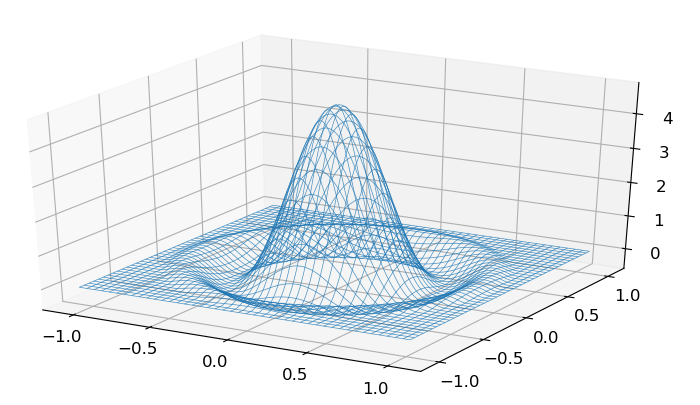}\put(3,48){\small $q=6$}\end{overpic}
\caption{The kernel $K_{d,q}^B$, $d=2$, $q=2,4,6$.}
\label{figure:Kernel-B}
\end{figure}

This type-1 optimal kernel is a truncation of a polynomial function with terms of 
degrees $0,2,\ldots,(q+2)$ and has a compact support (see Figure \ref{figure:Kernel-B}).
It should be emphasized that the optimal kernel is a truncated kernel 
even though we have not assumed truncation \emph{a priori},
and that the truncation of the optimal kernel is a consequence of the minimum-sign-change 
condition and the optimality with respect to the AMSE criterion. 
In particular, the 2nd order optimal kernel $K_{d,2}^B(\bx)\propto\{(1-\|\bx\|^2)_+\}^2$ (where 
$(\cdot)_+=\max\{\cdot,0\}$), often called a Biweight kernel, is optimal among non-negative RKs
(from Theorem~\ref{theorem:thm21}).
The kernel $G_{d,q}^B$ can be represented by the product of the 
$d$-variate Biweight RK $G_{d,2}^B$ and a polynomial function as
\begin{align}
\label{eq:Opt-RK3}%
	G_{d,q}^B(x)
	=(-1)^{\frac{q}{2}+1}
	\frac{2\Gamma(\frac{d+q}{2}+2)}{\Gamma(\frac{d}{2}+3)\Gamma(\frac{q}{2}+2)}
	P_{\frac{q}{2}-1}^{(2,\frac{d}{2})}(2x^2-1) \cdot G_{d,2}^B(x).
\end{align}
\citet{berlinet1993hierarchies} pointed out that in a variety of kernel-based estimation problems 
the optimal kernels of different orders often form a \emph{hierarchy}, 
in the sense that an optimal kernel can be represented as a product of a polynomial and 
the corresponding lowest-order kernel called the basic kernel:
See also Section 3.11 of~\citep{BerlinetThomasAgnan2004}.
According to their terminology, for example, the statement of 
Theorem~\ref{theorem:thm21} or \ref{theorem:thm22} can be concisely summarized as follows: 
Among the class of RKs, the Biweight hierarchy minimizes the AMSE of the kernel mode estimator 
for any $d$ under the minimum-sign-change condition for the kernel or its first derivative.

However, it also should be noted that the kernel optimization problem was formulated
without considering several regularity conditions other than the moment condition 
that ensure the asymptotic normality and are for deriving the AMSE.
The kernel $K_{d,q}^B$ is not twice differentiable at the edge of its support.
In this respect, $K_{d,q}^B$ does not satisfy all the sufficient conditions in Theorem~\ref{theorem:AN} 
for deriving the AMSE (see \hyl{A.7} and \hyl{A.14} in Assumption \ref{assumption:asmAN} in Appendix).
On the other hand, one can still add a small enough perturbation to $K_{d,q}^B$ to make 
it twice differentiable at the edge of its support as well, while keeping other conditions to hold,
as well as having the change of the value of the AMSE criterion arbitrarily small.
Namely, Problem~\ref{problem:Prob21} or \ref{problem:Prob22} and its solutions $K_{d,q}^B$ still tell us 
the lower bound of achievable AMSE under the minimum-sign-change condition, 
as well as the shape of kernels which will achieve AMSE close to the bound
(it also implies that the problem of minimizing the AMSE under 
all the conditions for Theorem~\ref{theorem:AN} has no solution).
From these considerations, we expect that, despite the lack of theoretical guarantee for the asymptotic normality of KME,
this trouble is not so destructive and $K_{d,q}^B$ will provide a 
practically good performance among the minimum-sign-change kernels;
See also simulation experiments in Section~\ref{section:Simulation}.

\subsection{Optimal Kernel in Product Kernels}
\label{section:PK}
In this section we consider optimization of the bandwidth-optimized 
AMSE with respect to the kernel among the class of PKs, 
where a kernel $K$ is a PK if there exists a function $G$ on $\R$ 
such that $K(\bx)=\prod_{j=1}^d G(x_j)$ for all $\bx=(x_1,\ldots,x_d)^\top$.
To simplify the discussion, we assume in this paper that $G$ as a kernel on $\R$ is symmetric and $q$-th order, 
making the resulting PK $K$ to be $q$-th order as well. 
We furthermore assume that $G$ satisfies the minimum-sign-change condition \hyl{P2-3} or \hyl{P3-3}.

Under these assumptions, simple calculations lead to
\begin{align}
	\calB_{d,\bi}(K)
	&=\begin{cases}
	1,&\bi=\bm{0}_d,\\
	2B_{1,q}(G),&\text{if one index in }\bi\text{ is }q\text{ and others are }0,\\
	0,&\text{for other }\bi\text{ such that }|\bi|\in\{1,\ldots,q\},
	\end{cases}\\
	\calV_d(K)
	&=2^d V_{1,1}(G)\cdot V^{d-1}_{1,0}(G)\rmI_d.
\end{align}
Therefore, the bandwidth-optimized AMSE~\eqref{eq:OptAMSE} becomes
\begin{align}
\label{eq:PKOptAMSE}%
\begin{split}
	\mathrm{E}[\|\bt_n-\bt\|^2\mid h_{d,q,n}^\opt]
	&\approx \frac{2q}{d+2q+2}
	\|\rmA\tilde{\bb}\|^{\frac{2(d+2)}{d+2q+2}}
	\biggl(\frac{d+2}{2qn} 2^d f(\bt) \trace(\rmA^2)\biggr)^{\frac{2q}{d+2q+2}}\\
	&\times 
	\biggl[\Bigl(B_{1,q}^6(G)\cdot V_{1,1}^{2q}(G)\Bigr)\cdot 
	\Bigl(B_{1,q}^2(G)\cdot V_{1,0}^{2q}(G)\Bigr)^{d-1}\biggr]^{\frac{1}{d+2q+2}},
\end{split}
\end{align}
where $\tilde{\bb}$ is given by
\begin{align}
\label{eq:btilde}%
	\tilde{\bb}
	\coloneq\sum_{j=1}^d\frac{1}{q!}\cdot \nabla\frac{\partial^qf(\bt)}{\partial x_j^q}\cdot2
	=\frac{2}{q!} \nabla \nabla_q f(\bt).
\end{align}
Equation~\eqref{eq:PKOptAMSE} shows that the bandwidth-optimized AMSE for the PK 
is also decomposed into a product of the kernel-dependent factor and the PDF-dependent factor. 
One can therefore take the kernel-dependent factor 
$\{(B_{1,q}^6\cdot V_{1,1}^{2q})\cdot(B_{1,q}^2\cdot V_{1,0}^{2q})^{d-1}\}^{\frac{1}{d+2q+2}}$ 
as the AMSE criterion for the PK.

We have not succeeded in deriving the optimal $G$ minimizing this criterion.
In the following, we consider instead a lower bound of the AMSE criterion, 
which will be compared with the optimal value of the criterion for the RK 
in the next section.
The bound we discuss is derived as follows:
\begin{align}
\begin{split}
	&\min_G\biggl\{\Bigl(B_{1,q}^6(G)\cdot V_{1,1}^{2q}(G)\Bigr)
	\cdot\Bigl(B_{1,q}^2(G)\cdot V_{1,0}^{2q}(G)\Bigr)^{d-1}\bigg\}\\
	&=\min_{G,G':G=G'}\bigg\{\Bigl(B_{1,q}^6(G)\cdot V_{1,1}^{2q}(G)\Bigr)
	\cdot\Bigl(B_{1,q}^2(G')\cdot V_{1,0}^{2q}(G')\Bigr)^{d-1}\bigg\}\\
	&\ge\min_G\Bigl(B_{1,q}^6(G)\cdot V_{1,1}^{2q}(G)\Bigr)
	\cdot\min_{G'}\Bigl(B_{1,q}^2(G')\cdot V_{1,0}^{2q}(G')\Bigr)^{d-1},
\end{split}
\end{align}
where the last inequality holds strictly if $d>1$ 
since the solutions of the two minimization problems involved will become different.
The first minimization problem is a univariate case of Theorem~\ref{theorem:thm21} 
or Theorem~\ref{theorem:thm22} (or \citep[Corollary~1]{granovsky1991optimizing}): 
the Biweight hierarchy $\{G_{1,q}^B\}$ provides the optimal $G$ among 
the $q$-th order kernels satisfying the minimum-sign-change condition \hyl{P3-3}.
The second minimization problem, which is an instance of the type-0 problems, 
has been solved by \citet{granovsky1989optimality}:
for even $q\ge2$, the kernel
\begin{align}
	G_{1,q}^E(x)
	=\frac{\Gamma(\frac{q+1}{2})}{\pi^{\frac{1}{2}}\Gamma(\frac{q}{2})}
	P_{\frac{q}{2}}^{(\frac{1}{2},-1)}(1-2x^2)\mathbbm{1}_{x\le1}
	=(-1)^{\frac{q}{2}+1}\frac{\Gamma(\frac{q+3}{2})}{\pi^{\frac{1}{2}}\Gamma(\frac{q}{2}+1)}
	(1-x^2)P_{\frac{q}{2}-1}^{(1,\frac{1}{2})}(2x^2-1)\mathbbm{1}_{x\le1}
\end{align}
minimizes $B_{1,q}^2\cdot V_{1,0}^{2q}$ among $q$-th order kernels that change their sign $(\frac{q}{2}-1)$ times on $\R_{\ge0}$.
The type-0 optimal kernel $K_{1,q}^E(x)=G_{1,q}^E(|x|)$ is a truncation of a polynomial function with terms of degrees $0,2,\ldots,q$, 
and forms the Epanechnikov hierarchy, 
with the Epanechnikov kernel $K_{1,2}^E\propto(1-x^2)_+$ as the basic kernel. 
We have therefore obtained a lower bound 
\begin{align}
\label{eq:PKOptAMSEBound}%
  \biggl\{\Bigl(B_{1,q}^6(K_{1,q}^B)\cdot V_{1,1}^{2q}(K_{1,q}^B)\Bigr)\cdot
  \Bigl(B_{1,q}^2(K_{1,q}^E)\cdot V_{1,0}^{2q}(K_{1,q}^E)\Bigr)^{d-1}\biggr\}^{\frac{1}{d+2q+2}}
\end{align}
of the AMSE criterion for the PK model.

\subsection{Comparison between Radial-Basis Kernels and Product Kernels}
\label{section:Comparison}
We have so far studied the RKs and the PKs separately. 
One may then ask which of these classes of kernels will perform better. 
In this section we discuss this question. 

Our approach to answering this question 
is basically grounded on a comparison between the AMSE of the optimal RK 
and that of the optimal PK. 
There are, however, difficulties in this approach: 
One is that the optimal PK, as well as its AMSE, is not known, 
which prevents us from comparing the RKs and the PKs directly in terms of the AMSE. 
Another is that the PDF-dependent factor of the AMSE takes different forms 
between the RKs and the PKs. 
Most evidently, in $\bar{\bb}$~\eqref{eq:bbar} and $\tilde{\bb}$~\eqref{eq:btilde}, 
the quantities which contribute to the AB of $\bm{\theta}_n$ for the RKs and the PKs, respectively, 
cross derivatives of the PDF are present in the former 
whereas they are absent in the latter. 
Such differences in the PDF-dependent factors 
hinder direct comparison between the RKs and the PKs 
in terms of the AMSE in the general setting, 
in a PDF-independent manner.

The only exception is the case $q=2$, 
where the PDF-dependent factor takes the same form 
for the RKs and the PKs, 
and consequently one can compare the AMSEs for RKs and PKs, regardless of the PDF.
In the case of $q=2$,
simple calculations show 
$\tilde{\bb}=\nabla \nabla_2f(\bt)$ and
$\bar{\bb}=\frac{v_d}{2} \nabla \nabla_2f(\bt)=\frac{v_d}{2}\tilde{\bb}$.
Therefore, 
the ratio of the lower bound of the AMSE for PK that derives from~\eqref{eq:PKOptAMSEBound}
to the AMSE for the optimal RK $K_{d,2}^B$ 
becomes 
\begin{align}
\label{eq:RatioBound}
	\Biggl(\frac{2^{6d+4}\bigl(B_{1,2}^6(K_{1,2}^B)\cdot V_{1,1}^4(K_{1,2}^B)\bigr)
	\cdot\bigl(B_{1,2}^2(K_{1,2}^E)\cdot V_{1,0}^4(K_{1,2}^E)\bigr)^{d-1}}
	{v_d^{2d+8}\bigl(B_{d,2}^{2(d+2)}(G_{d,2}^B)\cdot V_{d,1}^4(G_{d,2}^B)\bigr)}\Biggr)^{\frac{1}{d+6}}.
\end{align}
\begin{figure}
	\centering
	\includegraphics[height=50mm, bb=0 0 360 252]{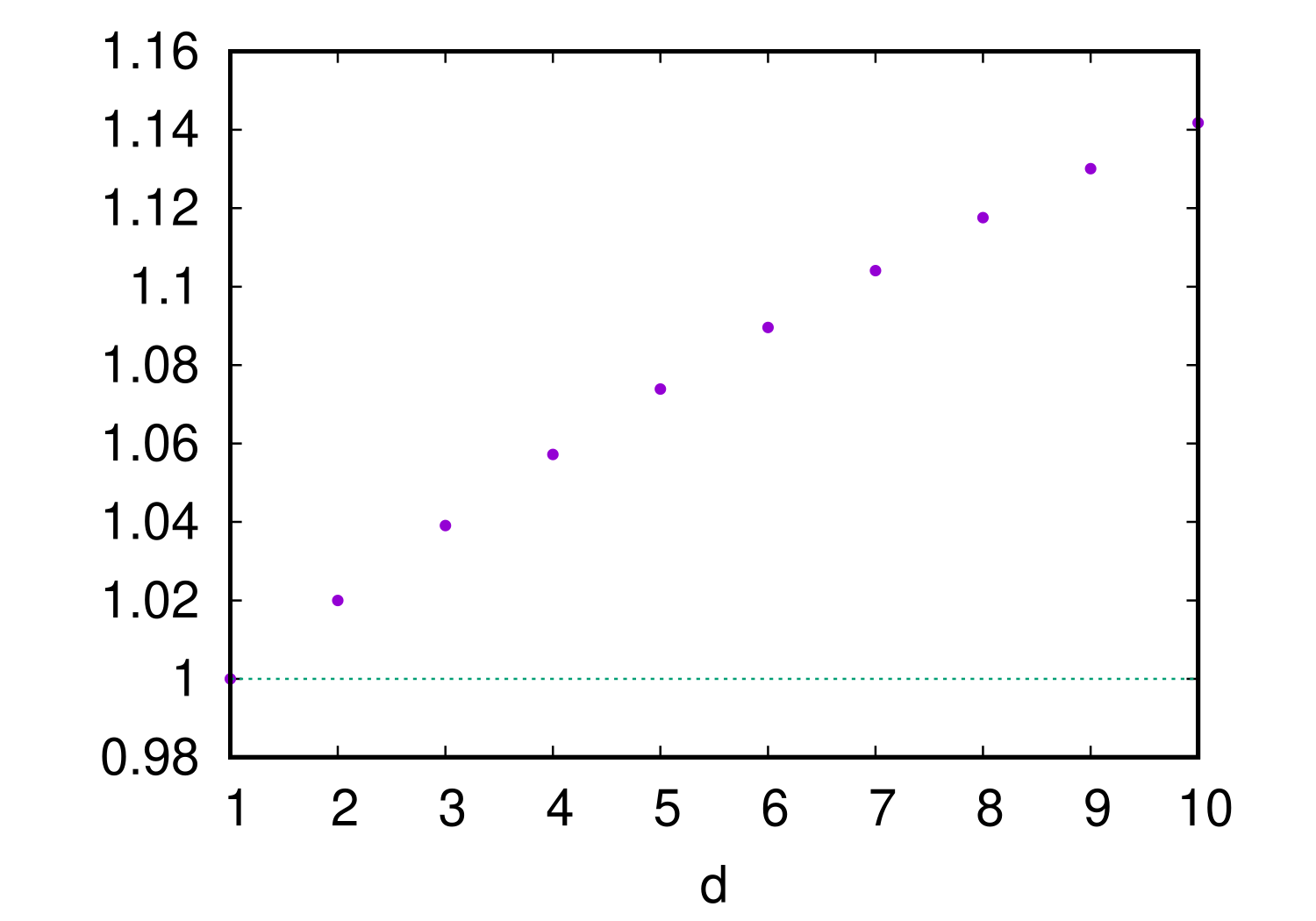}
	\caption{Ratio~\protect\eqref{eq:RatioBound} of AMSE lower bound for PK to AMSE for RK versus dimension $d$.}
	\label{fig:ratiobound}
\end{figure}
Figure~\ref{fig:ratiobound} shows how the ratio depends on the dimension $d$. 
One observes that it increases monotonically with $d$. 
%
Since the ratio is larger than 1 for $d\ge2$ 
(note that there is no distinction between the RKs and the PKs when $d=1$), 
one can conclude that the optimal second-order non-negative RK 
achieves the AMSE that is smaller than those with any second-order non-negative PK. 

It should be emphasized that the above superiority result of the optimal RK over the PKs is 
limited to the second-order ($q=2$) case. 
As mentioned above, for $q\ge4$ one cannot compare the RKs and the PKs 
in terms of the AMSE in a PDF-independent manner. 
Consequently, depending on the PDF, 
there might be some PKs with better AMSE values than the optimal RK $K_{d,q}^B$ for $q\ge4$, 
some examples of which will be shown in the simulation experiments in Section~\ref{section:Simulation}.

\subsection{Optimal Kernel in Elliptic Kernels}
\label{section:Elliptic}
Kernel methods in the multivariate setting may use a kernel function together with a linear transformation, 
in order to calibrate the difference in scales and shearing across the coordinates. 
In this section we assume for simplicity that an RK is to be combined with a linear transformation. 
For a $d\times d$ matrix $\rmP$ representing an area-preserving linear transformation in $\R^d$ (hence $|\det\rmP|=1$) 
and an RK $K$ (the base kernel), 
the transformed kernel $K_\rmP(\bx)=K(\rmP\bx)$ is often called an elliptic kernel.
In this section, we discuss the optimal base kernel $K$ and optimal transformation 
for kernel mode estimation based on an elliptic kernel.

Using the elliptic kernel $K_\rmP$ is equivalent to 
using the kernel $K$ with the transformed sample $\{\rmP\bX_i\}_{i=1}^n$ drawn from $f_\rmP$,
which is defined by $f_\rmP(\cdot)=f(\rmP^{-1}\cdot)$, 
for estimation of the mode $\rmP\bt$ of $f_\rmP$.
Multiplying the resulting mode estimate by $\rmP^{-1}$ from left yields the KME with $K_\rmP$ of $f$, 
an estimate of the mode $\bt$.
This view, together with the calculations
\begin{align}
	\{Hf_\rmP(\rmP\bt)\}^{-1}=\rmP\rmA\rmP^\top,\quad
	\nabla\nabla^{q/2}_2f_\rmP(\rmP\bt)
	=\nabla(\nabla^\top\rmQ\nabla)^{q/2}f(\bt),
\end{align}
where $\rmQ\coloneq\rmP^{-1}\rmP^{-\top}$,
reveals that the AB and AVC of the KME using $K_\rmP$ reduce to 
\begin{align}
\label{eq:AB-H}%
	&\mathrm{E}[\bt_n-\bt]\approx
	-\frac{\pi^{d/2}h_n^q}{2^{q-1}\Gamma(\frac{d+q}{2})\Gamma(\frac{q}{2}+1)}B_{d,i}(G)
	\rmA \nabla(\nabla^\top\rmQ\nabla)^{q/2}f(\bt),\\
\label{eq:AVC-H}%
	&\mathrm{Cov}[\bt_n]\approx
	\frac{v_df(\bt)}{nh_n^{d+2}}
	V_{d,1}(G)\rmA\rmQ^{-1}\rmA,
\end{align} 
and that their kernel-dependent factors remain $B_{d,q}(G)$ and $V_{d,1}(G)$ of $K(\cdot)=G(\|\cdot\|)$.
Therefore, as long as $\rmP$ (or equivalently $\rmQ$) is fixed such that the AB of the KME is non-zero, 
an argument similar to the ones in Section~\ref{section:KME} leads 
to the optimal bandwidth and the same AMSE criterion with respect to~$K$.
Thus, the kernel $K_{d,q}^B$ is optimal as a base kernel even under such a transformation $\rmP$.

One may furthermore consider optimizing $\rmQ$ (equivalently $\rmP$) via
\begin{align}
\label{eq:Opt-Prob-H}
	\tmin_\rmQ\;
	\|\rmA \nabla(\nabla^\top\rmQ\nabla)^{q/2}f(\bt)\|^{2(d+2)}
	\trace(\rmA\rmQ^{-1}\rmA)^{2q},\quad
	\mathrm{s.t.}\ \det(\rmQ)=1,
\end{align}
which comes from the linear-transformation-dependent factor of the bandwidth-optimized AMSE.
We have so far not succeeded in obtaining a closed-form solution
of the optimization problem~\eqref{eq:Opt-Prob-H},
although a numerical optimization might be possible
with plug-in estimates of $f(\bm{\theta})$ and $\rmA$.

It would be worth mentioning that
in some cases one can make the AB equal to zero by a certain choice of $\rmQ$,
whereas in some other cases one cannot make the AB equal to zero
by any choice of $\rmQ$.
Furthermore, whether one can make the AB equal to zero
via the choice of $\rmQ$ depends on the ($q+1)$-st order term
of the Taylor expansion of $f(\bx)$ around the mode $\bx=\bt$, 
and the dependence seems to be quite complicated. 
Consequently, optimization of the linear transformation in terms of the AMSE is difficult to tackle 
in a manner similar to that adopted in the previous sections, 
because the assumption that the AB does not vanish does not necessarily hold. 
Although we have not succeeded in providing a complete specification
as to when one can make the AB equal to zero via the choice of $\rmQ$, 
we describe in the following 
some cases where the AB can be made equal to zero by a certain choice of $\rmQ$ 
and some other cases where the AB cannot be equal to zero by any choice of $\rmQ$.

The following two propositions provide cases where the AB can be made zero by an appropriate choice of $\rmQ$. 
\begin{proposition}
  Assume $d\ge q+1$.
  Assume the $(q+1)$-st order term 
  of the Taylor expansion of $f(\bx)$ around
  the mode $\bx=\bt$ be of the form
  \begin{align}
    \prod_{i=1}^{q+1}\bm{a}_i^\top(\bx-\bt),
  \end{align}
  with $\{\bm{a}_i\}_{i=1}^{q+1}$ linearly independent.
  Then there is a choice of $\rmQ$ with which the AB~\eqref{eq:AB-H}
  is made equal to zero.
\end{proposition}

We show that there exists a positive definite $\rmQ$
which satisfies
\begin{align}
  \label{eq:tmppd0}
  \nabla(\nabla^\top \rmQ\nabla)^{q/2}
  \prod_{i=1}^{q+1}\ba_i^\top(\bx-\bt)=\bm{0}_d.
\end{align}
The left-hand side is calculated as
\begin{align}
  \label{eq:tmppd}
  \nabla(\nabla^\top \rmQ\nabla)^{q/2}
  \prod_{i=1}^{q+1}\ba_i^\top(\bx-\bt)
  =\sum_{\sigma}(\ba_{\sigma(1)}^\top \rmQ\ba_{\sigma(2)})
  \cdots(\ba_{\sigma(q-1)}^\top \rmQ\ba_{\sigma(q)})\ba_{\sigma(q+1)},
\end{align}
where the summation on the right-hand side is to be taken
with respect to all the permutations of $\{1,\ldots,q+1\}$.
Due to the assumption of the linear independence of
$\{\ba_i\}_{i=1}^{q+1}$,
there exists the duals $\{\bb_i\}_{i=1}^{q+1}$ 
of $\{\ba_i\}_{i=1}^{q+1}$ satisfying
$\ba_i^\top\bb_j=\mathbbm{1}_{i=j}$.
Take $\{\bb_i\}_{i=q+2,\ldots,d}$ as a basis
of the orthogonal complement of $\{\ba_i\}_{i=1}^{q+1}$.
Then $\{\bb_i\}_{i=1,\ldots,d}$ forms a basis of $\R^d$.
Let
\begin{align}
  \rmQ=\sum_{i=1}^d\bb_i\bb_i^\top\succ \rmO_d,
\end{align}
with the zero $d\times d$ matrix $\rmO_d$,
where the notation $\rmA\succ\rmB$ means that
$\rmA-\rmB$ is positive definite.
One then has $\ba_i^\top \rmQ\ba_j=\mathbbm{1}_{i=j}$
for $i,j\in\{1,\ldots,q+1\}$.
It implies that all the coefficients of $\ba_i$
on the right-hand side of~\eqref{eq:tmppd} vanish,
proving~\eqref{eq:tmppd0} to hold with the above $\rmQ$.
\begin{proposition}
  \label{prop:Prop3}
  Assume $d\ge q$.
  Assume the $(q+1)$-st order term 
  of the Taylor expansion of $f(\bx)$ around
  the mode $\bx=\bt$ be of the form
  \begin{align}
    \prod_{i=1}^{q+1}\ba_i^\top(\bx-\bt),
  \end{align}
  where $\{\ba_i\}_{i=1}^{q+1}$ spans a $q$-dimensional 
  subspace of $\R^d$. 
  Assume further that $\sum_{i=1}^{q+1}s_i\ba_i=\mathbf{0}_d$ holds with coefficients $\{s_i\}_{i=1}^{q+1}$ 
  satisfying $s_1s_2\cdots s_{q+1}\not=0$. 
  Then there is a choice of $\rmQ$ with which the AB~\eqref{eq:AB-H}
  is made equal to zero.
\end{proposition}

One can find a subset of $\{\ba_i\}_{i=1}^{q+1}$ 
of size $q$, which consists of $q$ linearly independent vectors. 
Without loss of generality we assume that $\{\ba_i\}_{i=1}^q$ are linearly independent. 
Choosing
\begin{align}
  \rmQ=\sum_{i=1}^{q}\frac{1}{s_i^2}\bb_i\bb_i^\top
  -\sum_{i,j=1\atop i<j}^{q}\frac{1}{qs_is_j}
  (\bb_i\bb_j^\top+\bb_j\bb_i^\top)
  +\sum_{i=q+2}^d\bb_i\bb_i^\top\succ \mathrm{O}_d
\end{align}
with $\{\bb_i\}_{i=1}^{q}$, 
the duals of $\{\ba_i\}_{i=1}^{q}$ 
in the subspace spanned by $\{\ba_i\}_{i=1}^{q}$, 
and with $\{\bb_i\}_{i=q+2}^d$, a basis of the orthogonal complement
of $\{\ba_i\}_{i=1}^{q}$,
one has 
\begin{align}
	\ba_i^\top\rmQ\ba_j=\left\{
		\begin{array}{ll}
			\frac{1}{s_i^2},&i=j\in\{1,\ldots,q+1\},\\
			-\frac{1}{qs_is_j},&i\not=j,i,j\in\{1,\ldots,q+1\},\\
			0,&\mbox{otherwise}.
		\end{array}\right.
\end{align}
One can thus show that the following holds:
\begin{align}
  \nabla(\nabla^\top \rmQ\nabla)^{q/2}
  \left(\prod_{i=1}^{q+1}\ba_i^\top(\bx-\bt)\right)
  &=\left(-\frac{1}{q}\right)^{q/2}\frac{1}{\prod_{i=1}^{q+1}s_i}
  \sum_{\sigma}s_{\sigma(q+1)}\ba_{\sigma(q+1)}=\bm{0}_d.
\end{align}

The following propositions provide cases where the AB cannot be equal
to zero by any choice of $\rmQ$.
\begin{proposition}
	\label{prop:Prop4}
  Assume $d\ge q$.
  Assume the $(q+1)$-st order term 
  of the Taylor expansion of $f(\bx)$ around
  the mode $\bx=\bt$ be of the form
  \begin{align}
    \ba^\top(\bx-\bt)
    \prod_{i=1}^{q/2}\frac{1}{2}(\bx-\bt)^\top
    \rmR_i(\bx-\bt),
  \end{align}
  where $\rmR_i$, $i=1,\ldots,q/2$ are positive definite. 
  Then the AB~\eqref{eq:AB-H} cannot be equal to zero
  by any choice of $\rmQ$. 
\end{proposition}

In the case of $q=2$ one can prove the following proposition.
\begin{proposition}
If, around its mode
$\bx=\bt$, $f(\bx)$ admits the expression
\begin{align}
  f(\bx)=f(\bt)
  +\frac{1}{2}(\bx-\bt)^\top\rmA^{-1}(\bx-\bt)
  +g(\bx-\bt)\ba^\top(\bx-\bt)
  +\mathrm{hot}.,
\end{align}
where $\ba$ is nonzero, 
and where
\begin{align}
  g(\bu)=\frac{1}{2}\bu^\top \rmR\bu,\quad \rmR\succ \rmO_d
\end{align}
is a positive definite quadratic form,
then there is no $\rmP$ which makes the AB to vanish.
\end{proposition}
We show that in this case no choice of $\rmQ$ which is positive definite 
will make the quantity
$\nabla\nabla^\top \rmQ\nabla(g(\bx-\bt)\ba^\top(\bx-\bt))$ vanishing.
Indeed, one can write
\begin{align}
  \label{eq:tmp}
  \nabla\nabla^\top \rmQ\nabla
  \left(g(\bx-\bt)\ba^\top(\bx-\bt)\right)
  =(2\rmR\rmQ+(\trace \rmR\rmQ)\rmI)\ba.
\end{align}
One observes that $\rmR\rmQ$ has positive eigenvalues only, 
since $\rmR$ and $\rmQ$ are both positive definite. 
We thus have $\trace \rmR\rmQ>0$, which in turn implies
that the matrix $2\rmR\rmQ+(\trace \rmR\rmQ)\rmI$ is invertible. 
Hence the quantity in~\eqref{eq:tmp} is nonvanishing
with an arbitrary choice of $\rmQ$.

Proposition~\ref{prop:Prop4} can be proven similarly. 
The term $\nabla(\nabla^\top\rmQ\nabla)^{q/2}f(\bt)$ in this case 
is represented as $\rmM(\{\rmR_i\},\rmQ)\ba$, 
where the $d\times d$ matrix $\rmM(\{\rmR_i\},\rmQ)$ 
is a sum of terms of the form $\rmR_{j_1}\rmQ\cdots\rmR_{j_n}\rmQ$ 
with coefficients which are products of $\trace\rmR_{i_1}\rmQ\cdots\rmR_{i_k}\rmQ>0$, 
and is hence positive definite. 
The term $\nabla(\nabla^\top\rmQ\nabla)^{q/2}f(\bt)$ is therefore nonzero 
irrespective of the choice of $\rmQ$.

We close this section by showing a complete specification as to when
the AB can be made equal to zero via the choice of $\rmQ$
in the most basic case $(d,q)=(2,2)$.
\begin{proposition}
  \label{prop:Prop6}
  Assume $(d,q)=(2,2)$.
  The third-order term of the Taylor expansion of $f(\bm{x})$
  around the mode $\bm{x}=\bm{\theta}$ is either:
  \begin{itemize}
  \item 0 identically, in which case the AB is equal to zero
    with an arbitrary $\rmQ$,
  \item of the form $\prod_{i=1}^3\bm{a}_i^\top(\bm{x}-\bm{\theta})$
    with $\bm{a}_1,\bm{a}_2,\bm{a}_3$,
    any two of which are linearly independent,
    in which case there is a choice of $\rmQ$ with which the AB
    is equal to zero,
  \item other than any of the above, in which case
    the AB cannot be made equal to zero by any choice of $\rmQ$.
  \end{itemize}
\end{proposition}
A proof is given in Appendix~\ref{appendix:ProofProp6}.

\section{Additional Discussion}
\label{section:Discussions}
\subsection{Degree of Goodness of Optimal Kernel and Heuristic Findings}
\label{section:Heuristic}
\begin{table}[t]
\centering
\caption{The AMSE criterion $(B_{1,q}^6\cdot V_{1,1}^{2q})^{\frac{1}{2q+3}}$ and the AMSE ratio in brackets, for the univariate KME, $q=2, 4, 6$.}
\label{table:AMSEratio}
\scalebox{0.7}{\begin{sc}\begin{tabular}{c|l|cc|c}
\toprule
$q$ & \hspace{65pt}Kernel $K$ & $B_{1,q}$ & $V_{1,1}$ & $(B_{1,q}^6\cdot V_{1,1}^{2q})^{\frac{1}{2q+3}}$\\
\midrule\midrule
\mr{10}{2}
& Biweight $K_{1,2}^B(x)=\frac{15}{16}\{(1-x^2)_+\}^2$ & $\frac{1}{14}$ & $\frac{15}{14}$ & \textbf{0.1083}~[\textbf{1.0000}]\\
& Triweight $\frac{35}{32}\{(1-x^2)_+\}^3$ & $\frac{1}{18}$ & $\frac{35}{22}$ & 0.1095~[1.0105]\\
& Tricube $\frac{70}{81}\{(1-|x|^3)_+\}^3$ & $\frac{35}{486}$ & $\frac{210}{187}$ & 0.1121~[1.0345]\\
& Cosine $\frac{\pi}{4}\cos(\frac{\pi x}{2})\mathbbm{1}_{|x|\le1}$ & $\frac{1}{2}-\frac{4}{\pi^2}$ & $\frac{\pi^4}{128}$ & 0.1135~[1.0475]\\
& Epanechnikov $K_{1,2}^E(x)=\frac{3}{4}(1-x^2)_+$ & 0.1 & 0.75 & 0.1179~[1.0883]\\
& Triangle $(1-|x|)_+$ & $\frac{1}{12}$ & 1 & 0.1188~[1.0971]\\
& Gaussian $K_{1,2}^G(x)=\frac{1}{(2\pi)^{1/2}}e^{-x^2/2}$ & 0.5 & $\frac{1}{8\pi^{1/2}}$ & 0.1213~[1.1198]\\
& Logistic $\frac{1}{e^x+2+e^{-x}}$ & $\frac{\pi^2}{6}$ & $\frac{1}{60}$ & 0.1476~[1.3629]\\
& Sech $\frac{1}{2}\mathrm{sech}(\frac{\pi x}{2})$ & 0.5 & $\frac{\pi}{24}$ & 0.1727~[1.5949]\\
& Laplace $K_{1,2}^L(x)=\frac{1}{2}e^{-|x|}$ & 1 & 0.125 & 0.3048~[2.8133]\\
\midrule
\mr{4}{4}
& $K_{1,4}^B(x)=\frac{105}{64}(1-3x^2)\{(1-x^2)_+\}^2$ & $-\frac{1}{66}$ & $\frac{525}{88}$ & \textbf{0.3729}~[\textbf{1.0000}]\\
& $K_{1,4}^E(x)=\frac{15}{32}(3-7x^2)(1-x^2)_+$ & $-\frac{1}{42}$ & $\frac{75}{16}$ & 0.4005~[1.0737]\\
& $K_{1,4}^G(x)=\frac{1}{2(2\pi)^{1/2}}(3-x^2)e^{-x^2/2}$ & -1.5 & $\frac{55}{128\pi^{1/2}}$ & 0.4451~[1.1935]\\
& $K_{1,4}^L(x)=\frac{1}{20}(12-x^2)e^{-|x|}$ & -21.6 & $\frac{313}{1600}$ & 1.6314~[4.3743]\\
\midrule
\mr{4}{6}
& $K_{1,6}^B(x)=\frac{315}{2048}(15-110x^2+143x^4)\{(1-x^2)_+\}^2$ & $\frac{1}{286}$ & $\frac{2205}{128}$ & \textbf{1.0149}~[\textbf{1.0000}]\\
& $K_{1,6}^E(x)=\frac{105}{256}(5-30x^2+33x^4)(1-x^2)_+$ & $\frac{5}{858}$ & $\frac{3675}{256}$ & 1.0760~[1.0600]\\
& $K_{1,6}^G(x)=\frac{1}{8(2\pi)^{1/2}}(15-10x^2+x^4)e^{-x^2/2}$ & 7.5 & $\frac{7105}{8192\pi^{1/2}}$ & 1.2639~[1.2453]\\
& $K_{1,6}^L(x)=\frac{1}{2384}(1560-220x^2+3x^4)e^{-|x|}$ & $\frac{196200}{149}$ & $\frac{5572345}{22733824}$ & 5.7451~[5.6609]\\
\bottomrule
\end{tabular}\end{sc}}
\end{table}

We are interested in not only specifying the optimal kernel, 
but also in how the kernel selection will affect the AMSE.
To quantify the degree of suboptimality of kernels, 
we define, for a $d$-variate $q$-th order kernel $K$,
the \emph{AMSE ratio} as the ratio of the 
bandwidth-optimized AMSE~\eqref{eq:OptAMSE} for $K$
to that for the optimal RK $K_{d,q}^B$.
The AMSE ratio depends only on $K$ if the bandwidth-optimized AMSEs for $K$ and $K_{d,q}^B$ 
share the same PDF-dependent factor, which is the case 
for RKs (including univariate ones) and PKs  with $q=2$ (see Section~\ref{section:Comparison}).
Table~\ref{table:AMSEratio} shows the comparison of the AMSE criteria
and the AMSE ratios in the univariate case.
An empirical observation is that truncated kernels, such as a Triweight and 
an Epanechnikov in addition to the optimal Biweight kernel, are better
than non-truncated kernels including a Sech, a Laplace, and $K_{1,4}^L$.
This tendency still holds even for multivariate RKs and PKs.
Figure~\ref{fig:AMSEratio_multivariateRK} shows the AMSE ratios for several multivariate RKs.
\begin{figure}
\centering
 \begin{tabular}{cc}
\begin{overpic}[height=50mm, bb=0 0 360 252]{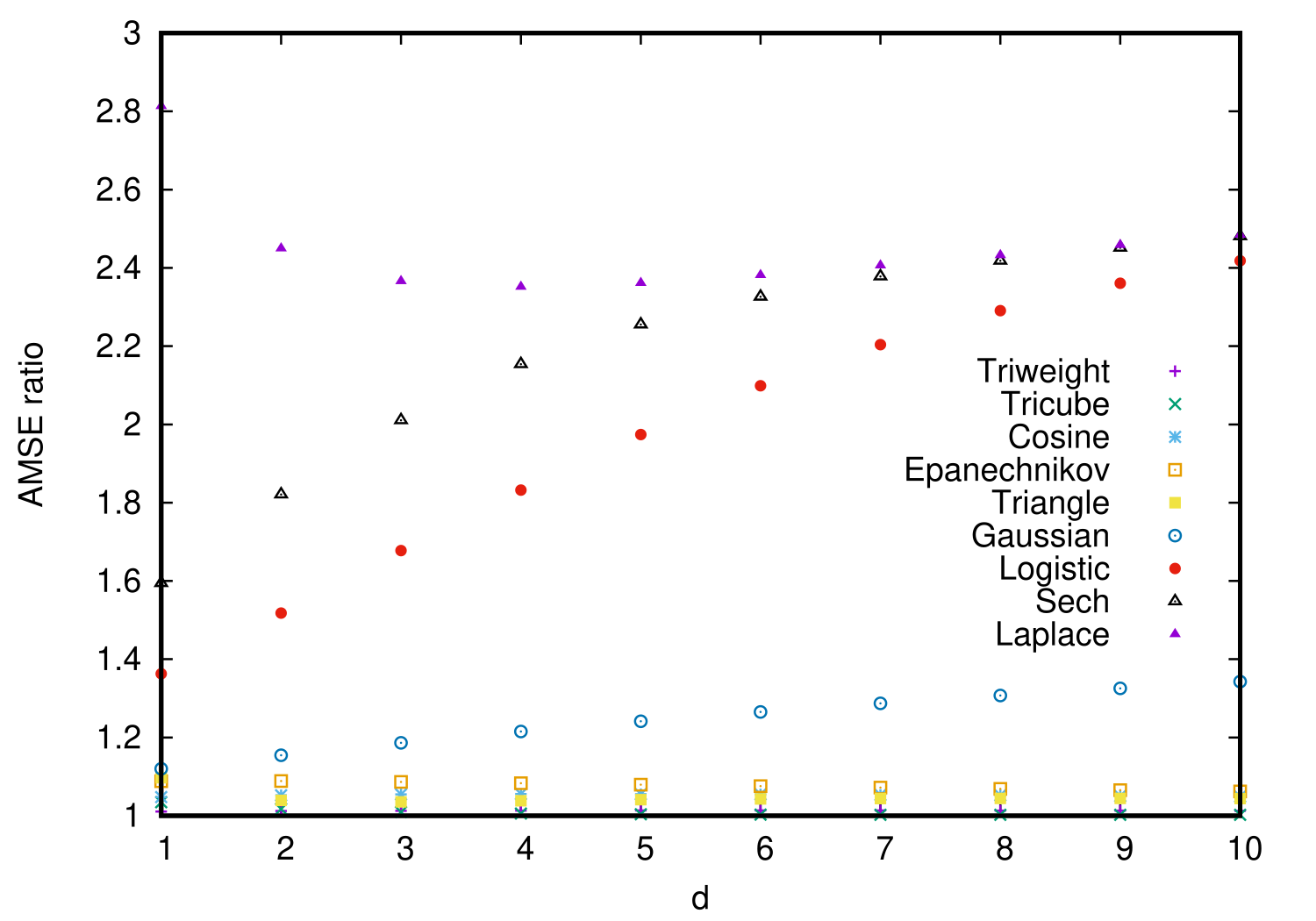}\put(18,60){\small(a)}\end{overpic}&
\begin{overpic}[height=50mm, bb=0 0 360 252]{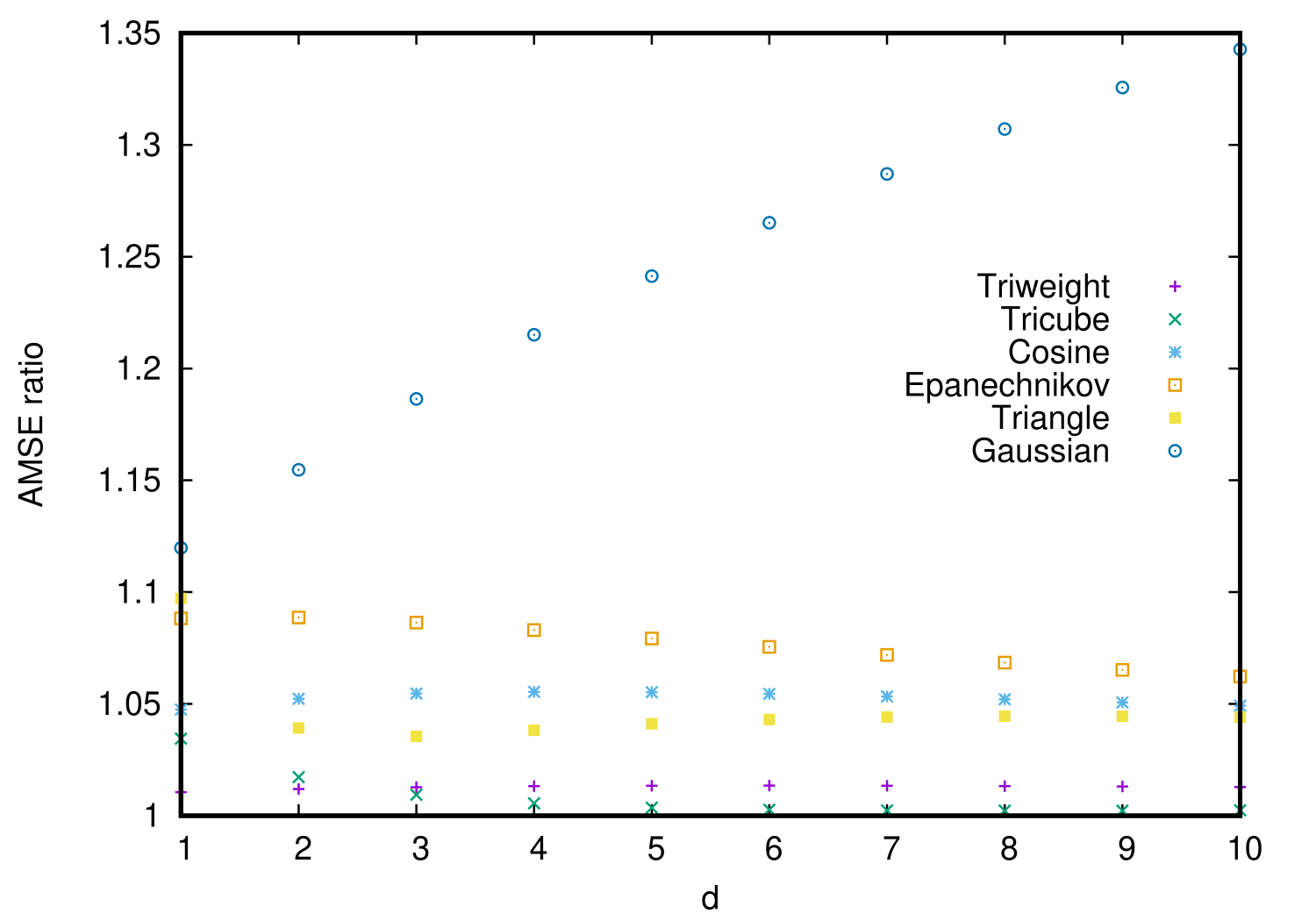}\put(18,60){\small(b)}\end{overpic}
\end{tabular}
\caption{AMSE ratios for multivariate RKs. 
(b) shows a zoom of (a), with the range of the AMSE ratio being [1,1.35], 
to better visualize the differences.}
\label{fig:AMSEratio_multivariateRK}
\end{figure}
For example, the AMSE ratio for the most frequently used Gaussian kernel $K_{d,2}^G$, 
approximately equals to 1.1198, 1.1547, 1.1864, 1.2151, 1.2413, 
1.2652, 1.2870, 1.3071, 1.3256, 1.3428 for $d=1,\ldots,10$, 
and monotonically approaches $e^2/4\approx 1.8473$ as $d$ goes to infinity.

\subsection{Other Criteria}
\label{section:Criteria}
We here mention that the optimality of the kernel $K_{d,q}^B$ is not limited to the scenario in which the AMSE is used as the criterion.
In this section, the univariate case is assumed for simplicity,
but its multivariate extension is straightforward.

\citet{grund1995minimisation} have studied the asymptotic $L^p$ error (A$L^p$E) 
$\mathrm{E}[|\theta_n-\theta|^p]$ for $p\ge1$ and the optimal bandwidth minimizing it.
Although they have clarified only the $n$-dependence of the optimal bandwidth, 
we can go one step further beyond their discussion and clarify the kernel-dependence: 
the optimal bandwidth is represented as $h_{1,q,p,n}^\opt=c_p(V_{1,1}/(nB_{1,q}^2))^{\frac{1}{2q+3}}$, 
where it is difficult to derive a closed-form expression for $c_p$, 
but it is a PDF-dependent and kernel-independent coefficient.
If this optimal bandwidth is used, then the A$L^p$E reduces to
\begin{align}
	\mathrm{E}[|\theta_n-\theta|^p\mid h_{1,q,p,n}^\opt]
	\approx 
	n^{-\frac{pq}{2q+3}}
	\Bigl(B_{1,q}^6\cdot V_{1,1}^{2q}\Bigr)^{\frac{p}{2(2q+3)}}
	\mathrm{E}\left[\biggl|\frac{2^{1/2}\{f(\theta)\}^{1/2}}{c_p^{3/2}f^{(2)}(\theta)}Z
	+\frac{2 c_p^q f^{(q+1)}(\theta)}{q!f^{(2)}(\theta)}\biggr|^p\right],
\end{align}
where $Z$ denotes the random variable following the standard normal distribution, 
and where the expectation is taken with respect to $Z$.
Since the kernel-dependent factor of the A$L^p$E is $(B_{1,q}^6\cdot V_{1,1}^{2q})$, 
which is the same as that for the AMSE, 
the kernel $K_{1,q}^B$ also minimizes the A$L^p$E among minimum-sign-change kernels, 
even when $p\neq 2$.

Also, from \citep{mokkadem2003law}, it can be found that 
the KME $\theta_n$ satisfies the law of the iterated logarithm.
In the scaling $\lim_{n\to\infty}nh_n^{2q+3}/\ln\ln n=0$, 
it holds that
\begin{align}
	\tlimsup_{n\to\infty}\sqrt{\frac{nh_n^3}{2\ln\ln n}}(\theta_n-\theta)
	=-\tliminf_{n\to\infty}\sqrt{\frac{nh_n^3}{2\ln\ln n}}(\theta_n-\theta)
	=\frac{\sqrt{f(\theta)V_{1,1}}}{|f^{(2)}(\theta)|},\quad
	\text{a.s.},
\end{align}
in other words, $\sqrt{nh_n^3/(2\ln\ln n)} (\theta_n-\theta)$ becomes relatively compact, 
and its limit set, in which the rescaled AB is included almost surely, is given by 
\begin{align}
	\Bigl[-\sqrt{f(\theta)V_{1,1}}/|f^{(2)}(\theta)|,
	\sqrt{f(\theta)V_{1,1}}/|f^{(2)}(\theta)|\Bigr].
\end{align}
Then, the combination of the bandwidth~\eqref{eq:OptBan} minimizing the AMSE and the kernel $K_{1,q}^B$ uniformly minimizes the size of the limit set of $(\theta_n-\theta)$, 
which is proportional to $\sqrt{V_{1,1}/h_n^3}\propto(B_{1,q}^6\cdot V_{1,1}^{2q})^{\frac{1}{2(2q+3)}}$.

\subsection{Singular Hessian}
\label{section:Singular}
Although it has been so far assumed that the Hessian matrix $Hf$ is 
non-singular at the mode (see~\hyl{A.4} in Assumption~\ref{assumption:asmAN}), 
this section provides a consideration on the optimal kernel when the Hessian matrix is singular.
Since the multivariate singular case is so intricate
that even its convergence rate remains an open problem, 
we suppose $d=1$ and hence $f^{(2)}(\theta)=0$ at the mode $\theta$.

\citet{vieu1996note,mokkadem2003law} have studied the singular case, 
where the PDF $f$ satisfies $f^{(i)}(\theta)=0$, $i=1,\ldots,p$ 
and $f^{(p+1)}(\theta)<0$ for $1\le p\le q-1$ instead of $f^{(2)}(\theta)<0$.
The non-singular case studied in Section~\ref{section:KME} corresponds to $p=1$.
Note that the integer $p$ should be an odd number for $f$ to take a maximum at $\theta$.

A weak convergence of the KME $\theta_n$ holds even in the singular case, in a somewhat different manner, 
where the requirements on the moments of the kernel $K$ are weakened to 
\begin{align}
\label{eq:DegMomCon}%
	\calB_{1,i}(K)=
	\begin{cases}
	1,&i=0,\\
	0,&i=p,\ldots,q-1,\\
	2B_{1,q}(K)\neq0,&i=q.
	\end{cases}
\end{align}
In other words, the conditions $\calB_{1,i}(K)=0$, $i=1,\ldots,p-1$, 
required in the moment condition~\eqref{eq:GenMomCon} in the non-singular case, 
are no longer needed in the singular case. 
Together with several additional conditions (see \citep{mokkadem2003law} for the details), 
$(\theta_n-\theta)^p$ asymptotically follows a normal distribution with 
\begin{align}
  \mathrm{E}[(\theta_n-\theta)^p]\approx -h_n^q\bar{a}\bar{b}B_{1,q},\quad
  \mathrm{var}[(\theta_n-\theta)^p]\approx\frac{2f(\theta)}{nh_n^3}\bar{a}^2V_{1,1},
\end{align}
where $\bar{a}\coloneq(\frac{1}{p!}f^{(p+1)}(\theta))^{-1}$, $\bar{b}\coloneq\frac{2}{q!}f^{(q+1)}(\theta)$.
Moreover, the bias-variance decomposition leads straightforwardly to 
the asymptotic $L^{2p}$ error (A$L^{2p}$E), instead of the AMSE, of the KME $\theta_n$: 
\begin{align}
  \mathrm{E}[(\theta_n-\theta)^{2p}]
  \approx h_n^{2q}(\bar{a}\bar{b})^2B_{1,q}^2
  +\frac{2f(\theta)}{nh_n^3}\bar{a}^2V_{1,1},
\end{align}
which has the same functional form regarding the bandwidth $h_n$ and kernel-dependent terms 
$B_{1,q}$, $V_{1,1}$ as those of the AMSE~\eqref{eq:AMSE} for the univariate non-singular case.
The optimal bandwidth minimizing the A$L^{2p}$E can be calculated
in the same way as that for the conventional one~\eqref{eq:OptBan}.
Also, the A$L^{2p}$E criterion reduces to $(B_{1,q}^6\cdot V_{1,1}^{2q})^{\frac{1}{2q+3}}$.

\begin{table}[t]
\centering
\caption{The A$L^{2p}$E criterion $(B_{1,q}^6\cdot V_{1,1}^{2q})^{\frac{1}{2q+3}}$ and 
its ratio to the smallest value in brackets for the univariate singular cases, $(p,q)=(3,4), (5,6)$.}
\label{table:tab2}
\scalebox{0.7}{\begin{sc}\begin{tabular}{l|ccc|cc}
\toprule
\hspace{8pt}Kernel $K$ & $B_{1,4}$ & $B_{1,6}$ & $V_{1,1}$ & $(B_{1,4}^6\cdot V_{1,1}^8)^{\frac{1}{11}}$ & $(B_{1,6}^6\cdot V_{1,1}^{12})^{\frac{1}{15}}$\\
\midrule\midrule
Biweight & $\frac{1}{42}$ & $\frac{5}{462}$ & $\frac{15}{14}$ & \textbf{0.1369}~[\textbf{1.0000}] & 0.1729~\vo[1.0098]\\
Triweight & $\frac{1}{66}$ & $\frac{5}{858}$ & $\frac{35}{22}$ & 0.1426~[1.0418] & 0.1851~\vo[1.0815]\\
Tricube & $\frac{1}{44}$ & $\frac{1}{104}$ & $\frac{210}{187}$ & 0.1381~[1.0089] & $\mathbf{0.1712}~\vo[\mathbf{1.0000}]$\\
Cosine & $0.0394$ & $0.0214$ & $\frac{\pi^4}{128}$ & 0.1404~[1.0257] & 0.1728~\vo[1.0095]\\
Epanechnikov & $\frac{3}{70}$ & $\frac{1}{42}$ & 0.75 & 0.1455~[1.0631] & 0.1781~\vo[1.0405]\\
Triangle & $\frac{1}{30}$ & $\frac{1}{56}$ & 1 & 0.1564~[1.1427] & 0.1999~\vo[1.1674]\\
Gaussian & 1.5 & 7.5 & $\frac{1}{8\pi^{1/2}}$ & 0.1813~[1.3246] & 0.2683~\vo[1.5675]\\
Logistic & $\frac{7\pi^4}{30}$ & $\frac{31\pi^6}{42}$ & $\frac{1}{60}$ & 0.2797~[2.0435] & 0.5223~\vo[3.0507]\\
Sech & 2.5 & 30.5 & $\frac{\pi}{24}$ & 0.3757~[2.7443] & 0.7714~[4.4614]\\
Laplace & 12 & 360 & 0.125 & 0.8548~[6.2441] & 1.9955~[11.6563]\\
\midrule
$K_{1,4}^B$ & $-\frac{1}{66}$ & $-\frac{5}{429}$ & $\frac{525}{88}$ & 0.3729~[2.7244] & 0.7033~\vo[4.1083]\\
\midrule
$K_{1,6}^B$ & 0 & $\frac{1}{286}$ & $\frac{2205}{128}$ & -- & 1.0145~\vo[5.9282]\\
\bottomrule
\end{tabular}\end{sc}}
\end{table}

The essential difference of the discussion on the optimal
kernel under the singular case from that under the non-singular case
arises from the difference between the required moment conditions~\eqref{eq:DegMomCon} and~\eqref{eq:GenMomCon}: 
As mentioned above, a kernel function is not required to satisfy 
the $i$-th moment condition for $i=1,\ldots,p-1$ in the singular case. 
For example, when $p=3$ and $q=4$, 
any symmetric 2nd order kernel does not satisfy~\eqref{eq:GenMomCon}, 
but does satisfy the singular version~\eqref{eq:DegMomCon} of the moment conditions.
The kernel $K_{1,q}^B$, of course, minimizes the A$L^{2p}$E criterion among 
the conventional $q$-th order kernels satisfying the minimum-sign-change condition, 
regardless of $p$. 
However, there is a possibility to improve the A$L^{2p}$E criterion by 
using a lower order kernel satisfying~\eqref{eq:DegMomCon}.

We investigated two simple cases $(p,q)=(3,4), (5,6)$ 
by calculating the A$L^{2p}$E criterion for several kernels, 
and report the results in Table~\ref{table:tab2}, 
where the A$L^{2p}$E criterion is given by 
$(B_{1,4}^6\cdot V_{1,1}^8)^{\frac{1}{11}}$ and $(B_{1,6}^6\cdot V_{1,1}^{12})^{\frac{1}{15}}$ 
for $(p,q)=(3,4)$ and $(5,6)$, respectively.
In the case $(p,q)=(3,4)$, where any symmetric 2nd order kernel,
as well as any 4th order kernel, fulfills the conditions~\eqref{eq:DegMomCon}, 
one can observe that a Biweight kernel $K_{1,2}^B$ is better than $K_{1,4}^B$.
In the case $(p,q)=(5,6)$, where at least three types of kernels, 
symmetric 2nd and 4th order kernels and 6th order kernels, 
satisfy the conditions~\eqref{eq:DegMomCon}, 
Table~\ref{table:tab2} shows that a Tricube kernel gives the minimal value 
of the A$L^{2p}$E criterion among the kernels we investigated.
Thus, we have confirmed that the kernel $K_{1,q}^B$ which is optimal
under the non-singular case may not be optimal under the singular case,
where a lower order kernel may improve the asymptotic estimation accuracy. 
Although we have not succeeded in deriving an optimal kernel in the singular case, 
an empirical finding that truncated kernels perform well would be useful.

\section{Simulation Experiment}
\label{section:Simulation}
The analysis in the previous sections on the optimal kernel is based on 
the asymptotic normality and the evaluation of the AMSE, 
derived from the asymptotic expansion of the KME.
While the leading-order term of the bandwidth-optimized AMSE~\eqref{eq:OptAMSE} is $O(n^{-\frac{2q}{d+2q+2}})$, 
the next-leading-order term is $O(n^{-\frac{2(q+2)}{d+2q+2}})$
if one uses symmetric kernels such as RKs and PKs.
Although the bandwidth-optimized AMSE ignores all but the leading-order term, 
those ignored terms may affect behaviors of the KME and thus its MSE when the sample size $n$ is finite.
In this section, via simulation studies, 
we examine how well the kernel selection based on the AMSE criterion 
reflects the real performance of the KME
and verify practical goodness of the optimal RK $K_{d,q}^B$ in a finite sample situation.

We tried the three cases regarding the dimension, $d=1,2,3$, 
and used synthetic i.i.d.\ sample sets of size $n=100, 400, \ldots$, 102400, 
drawn from the distribution 
$\frac{1}{2}\mathcal{N}(\bm{0}_d,\rmI_d)+\frac{1}{2}\mathcal{N}(\bm{1}_d,2\rmI_d)$, 
where $\mathcal{N}(\bm{\mu},\Sigma)$ denotes the normal distribution with mean $\bm{\mu}$ and variance-covariance matrix $\Sigma$, 
and where $\bm{1}_d$ is the all-1 $d$-vector and $\rmI_d$ is the $d\times d$ identity matrix.
The modes of the sample-generating distribution are located at $\bt\approx0.2395, 0.1514\cdot\bm{1}_2, 0.0874\cdot\bm{1}_3$ 
for $d=1,2,3$, respectively.
Because a symmetric distribution such as the normal distribution does not satisfy assumption \hyl{A.6}, 
skewed distributions were used in the experiments.
As 2nd order kernels, we examined the optimal Biweight $K_{d,2}^B$, Epanechnikov $K_{d,2}^E$, Gaussian $K_{d,2}^G$, and Laplace $K_{d,2}^L$ kernels,
as well as PKs of $K_{1,2}^B$, $K_{1,2}^E$, and $K_{1,2}^L$ if $d\ge2$.
For the higher orders $q=4$ and $6$, 
in addition to the optimal kernels $K_{d,q}^B$, 
three RKs $K_{d,q}^E$, $K_{d,q}^G$, and $K_{d,q}^L$ designed via jackknife \citep{schucany1977improvement, wand1990gaussian}, 
an alternative method for designing a higher-order kernel,
on the basis of $K_{d,2}^E$, $K_{d,2}^G$, and $K_{d,2}^L$
(see Section~\ref{section:Kernel-information}), 
and four PKs with $K_{1,q}^B$, $K_{1,q}^E$, $K_{1,q}^G$, and $K_{1,q}^L$ were also examined. 
Note that the RKs $K_{d,q}^B$, $K_{d,q}^E$, and $K_{d,q}^L$ 
and PKs of $K_{1,q}^B$, $K_{1,q}^E$ and $K_{1,q}^L$ are not twice differentiable at some points 
and thus do not exactly satisfy all the conditions of Theorem~\ref{theorem:AN} on the asymptotic behaviors of the KME.
For each kernel, we used the optimal bandwidth~\eqref{eq:OptBan} 
calculated from the sample-generating PDF and its mode\footnote{%
It should be noted that this procedure makes access to the sample-generating PDF, 
which is unavailable in practice. See Footnote~\ref{ft:ft1}.
The purpose of our simulation experiments here is not to evaluate performance in practical settings 
but rather to see how the MSE of the KME with the bandwidth optimized with respect to the AMSE 
behaves in the finite-$n$ situation, so that we used the bandwidth values exactly optimized 
with respect to the AMSE.}.
In these settings, the smallest AMSE among those of the kernels examined 
is given by the RK $K_{d,q}^B$ for $d=1$ or $q=2,6$ 
and PK with $K_{1,4}^B$ for $d=2,3$ and $q=4$. 
On the basis of 1000 trials, we calculated the MSE and its standard deviation (SD), 
and the results are reported in Table~\ref{table:SimRes},
along with the MSE ratio, defined as the ratio
of the MSE of each kernel to that for the kernel $K_{d,q}^B$ with the same $d$ and $q$ when $n=102400$.
Table~\ref{table:SimRes} also shows the results of  the Welch test with $p$-value cutoff 0.05, with the null hypothesis that the MSE of interest
be equal to the best MSE for the same $d,q,n$.

\begin{table}[t]
\centering
\caption{MSE and SD (both multiplied by $10^2$) of the KME.
For each combination of $d,q,n$, the smallest MSE is shown in bold.
Shown in gray are those which were worse than the smallest
(Welch test, $p=0.05$).}
\label{table:SimRes}
\scalebox{0.54}{\begin{sc}\begin{tabular}{c|c|cc|cccccccr}
\toprule
$d$ & $q$ & Kernel & AMSE ratio & $n=100$ & $n=400$ & $n=1600$ & $n=6400$ & $n=25600$ & $n=102400$ & MSE ratio\\
\midrule\midrule
\mr{12}{1} 
& \mr{4}{2} 
& $K_{1,2}^B$ & [\textbf{1.0000}] &
$\mathbf{4.2571\pm0.1977}$ & $\mathbf{1.6822\pm0.0737}$ & $\mathbf{0.5799\pm0.0259}$ & $\mathbf{0.2574\pm0.0110}$ & $\mathbf{0.1186\pm0.0052}$ & $\mathbf{0.0489\pm0.0021}$ & [\textbf{1.0000}]\\
&& $K_{1,2}^E$ & [1.0883] &
$4.3761\pm0.1945$ & $1.7085\pm0.0707$ & $0.6470\pm0.0285$ & $0.2772\pm0.0118$ & $0.1274\pm0.0054$ & $0.0528\pm0.0022$ & [1.0805]\\
&& $K_{1,2}^G$ & [1.1198] &
\CG{$4.9153\pm0.2414$} & \CG{$1.9430\pm0.0870$} & $0.6425\pm0.0282$ & \CG{$0.2918\pm0.0125$} & \CG{$0.1349\pm0.0058$} & \CG{$0.0553\pm0.0024$} & [1.1307]\\
&& $K_{1,2}^L$ & [2.8133] &
\CG{$10.8033\pm0.4957$} & \CG{$4.2865\pm0.2001$} & \CG{$1.4643\pm0.0606$} & \CG{$0.6984\pm0.0305$} & \CG{$0.2935\pm0.0126$} & \CG{$0.1234\pm0.0056$} & [2.5243]\\
\cmidrule{2-11}
& \mr{4}{4} 
& $K_{1,4}^B$ & [\textbf{1.0000}] &
$5.6845\pm0.2968$ & $1.9812\pm0.0877$ & $0.5978\pm0.0272$ & $0.2220\pm0.0099$ & $\mathbf{0.0876\pm0.0041}$ & $\mathbf{0.0294\pm0.0014}$ & [\textbf{1.0000}]\\
&& $K_{1,4}^E$ & [1.0737] &
$\mathbf{5.4032\pm0.2695}$ & $\mathbf{1.9559\pm0.0869}$ & $\mathbf{0.5976\pm0.0273}$ & $\mathbf{0.2217\pm0.0100}$ & $0.0882\pm0.0041$ & $0.0304\pm0.0014$ & [1.0334]\\
&& $K_{1,4}^G$ & [1.1935] &
\CG{$7.5127\pm0.3916$} & \CG{$2.5871\pm0.1156$} & \CG{$0.7530\pm0.0330$} & \CG{$0.2829\pm0.0125$} & \CG{$0.1089\pm0.0051$} & \CG{$0.0374\pm0.0017$} & [1.2712]\\
&& $K_{1,4}^L$ & [4.3743] &
\CG{$14.0314\pm0.5999$} & \CG{$6.2487\pm0.2741$} & \CG{$2.2248\pm0.0944$} & \CG{$0.9542\pm0.0428$} & \CG{$0.3592\pm0.0157$} & \CG{$0.1323\pm0.0059$} & [4.5026]\\
\cmidrule{2-11}
& \mr{4}{6} 
& $K_{1,6}^B$ & [\textbf{1.0000}] &
$7.8852\pm0.4199$ & $2.5326\pm0.1112$ & $0.7125\pm0.0320$ & $0.2374\pm0.0105$ & $0.0850\pm0.0041$ & $0.0254\pm0.0012$ & [1.0000]\\
&& $K_{1,6}^E$ & [1.0602] &
$\mathbf{7.4116\pm0.3902}$ & $\mathbf{2.3941\pm0.1070}$ & $\mathbf{0.6891\pm0.0320}$ & $\mathbf{0.2296\pm0.0103}$ & $\mathbf{0.0843\pm0.0041}$ & $\mathbf{0.0251\pm0.0012}$ & [\textbf{0.9882}]\\
&& $K_{1,6}^G$ & [1.2453] &
$10.5674\pm0.5306$ & $3.5329\pm0.1580$ & $0.9670\pm0.0418$ & $0.3282\pm0.0143$ & $0.1128\pm0.0053$ & $0.0351\pm0.0016$ & [1.3804]\\
&& $K_{1,6}^L$ & [5.6609] &
$16.9073\pm0.6946$ & $8.1200\pm0.3506$ & $2.9961\pm0.1284$ & $1.3032\pm0.0580$ & $0.4784\pm0.0213$ & $0.1700\pm0.0078$ & [6.6836]\\
\midrule
\mr{23}{2} 
& \mr{7}{2} 
& RK $K_{2,2}^B$ & [\textbf{1.0000}] &
$\mathbf{11.4818\pm0.4159}$ & $\mathbf{4.6054\pm0.1672}$ & $\mathbf{1.8538\pm0.0571}$ & $\mathbf{0.7933\pm0.0244}$ & $\mathbf{0.3510\pm0.0102}$ & $\mathbf{0.1669\pm0.0050}$ & [\textbf{1.0000}]\\
&& RK $K_{2,2}^E$ & [1.0887] &
$12.0854\pm0.4320$ & $4.8839\pm0.1710$ & \CG{$2.0408\pm0.0613$} & \CG{$0.8817\pm0.0267$} & \CG{$0.3842\pm0.0110$} & \CG{$0.1824\pm0.0057$} & [1.0930]\\
&& RK $K_{2,2}^G$ & [1.1547] &
\CG{$14.2550\pm0.5187$} & \CG{$5.4297\pm0.1925$} & \CG{$2.1603\pm0.0700$} & \CG{$0.9142\pm0.0301$} & \CG{$0.4018\pm0.0121$} & \CG{$0.1929\pm0.0056$} & [1.1557]\\
&& RK $K_{2,2}^L$& [2.4495] &
\CG{$28.7910\pm0.9371$} & \CG{$11.9038\pm0.3931$} & \CG{$4.8348\pm0.1622$} & \CG{$2.0335\pm0.0691$} & \CG{$0.8559\pm0.0271$} & \CG{$0.4139\pm0.0133$} & [2.4801]\\
&& PK of $K_{1,2}^B$ & [1.0231] &
$11.9216\pm0.4283$ & $4.7388\pm0.1700$ & $1.8922\pm0.0588$ & $0.8127\pm0.0254$ & $0.3579\pm0.0105$ & $0.1713\pm0.0051$ & [1.0265]\\
&& PK of $K_{1,2}^E$ & [1.0984] &
$12.0406\pm0.4247$ & $4.9762\pm0.1744$ & $1.9840\pm0.0600$ & \CG{$0.8672\pm0.0266$} & \CG{$0.3893\pm0.0111$} & \CG{$0.1826\pm0.0057$} & [1.0945]\\
&& PK of $K_{1,2}^L$& [2.8944] &
\CG{$29.0720\pm0.9473$} & \CG{$11.9243\pm0.3914$} & \CG{$4.9221\pm0.1607$} & \CG{$2.2091\pm0.0737$} & \CG{$0.9401\pm0.0311$} & \CG{$0.4626\pm0.0154$} & [2.7722]\\
\cmidrule{2-11}
& \mr{8}{4} 
& RK $K_{2,4}^B$ & [1.0000] &
$16.4534\pm0.6985$ & $5.4845\pm0.2056$ & $1.8539\pm0.0597$ & $0.6997\pm0.0226$ & $0.2671\pm0.0083$ & $\mathbf{0.1065\pm0.0035}$ & [\textbf{1.0000}]\\
&& RK $K_{2,4}^E$ & [1.0817] &
$\mathbf{15.1225\pm0.5873}$ & $\mathbf{5.3820\pm0.2064}$ & $\mathbf{1.8173\pm0.0575}$ & $0.6958\pm0.0225$ & $0.2774\pm0.0086$ & $0.1089\pm0.0036$ & [1.0218]\\
&& RK $K_{2,4}^G$ & [1.2552] &
\CG{$22.5890\pm0.8764$} & \CG{$7.5474\pm0.2701$} & \CG{$2.5624\pm0.0866$} & \CG{$0.9378\pm0.0317$} & \CG{$0.3452\pm0.0110$} & \CG{$0.1380\pm0.0044$} & [1.2955]\\
&& RK $K_{2,4}^L$ & [3.9734] &
\CG{$39.0583\pm1.1950$} & \CG{$18.1752\pm0.5879$} & \CG{$7.3580\pm0.2340$} & \CG{$3.0711\pm0.1027$} & \CG{$1.1322\pm0.0356$} & \CG{$0.4634\pm0.0152$} & [4.3494]\\
&& PK of $K_{1,4}^B$ & [\textbf{0.9948}] &
$16.6173\pm0.6915$ & $5.5669\pm0.2049$ & $1.8709\pm0.0616$ & $0.6935\pm0.0222$ & $\mathbf{0.2667\pm0.0084}$ & $0.1086\pm0.0036$ & [1.0191]\\
&& PK of $K_{1,4}^E$ & [1.0579] &
$15.6356\pm0.6444$ & $5.4196\pm0.2036$ & $1.8217\pm0.0599$ & $\mathbf{0.6756\pm0.0215}$ & $0.2711\pm0.0084$ & $0.1109\pm0.0037$ & [1.0409]\\
&& PK of $K_{1,4}^G$ & [1.2215] &
\CG{$22.0738\pm0.8510$} & \CG{$7.3806\pm0.2646$} & \CG{$2.4888\pm0.0845$} & \CG{$0.9088\pm0.0302$} & \CG{$0.3344\pm0.0106$} & \CG{$0.1350\pm0.0043$} & [1.2666]\\
&& PK of $K_{1,4}^L$& [4.9457] &
\CG{$38.6564\pm1.1771$} & \CG{$18.4631\pm0.5831$} & \CG{$7.5841\pm0.2377$} & \CG{$3.2581\pm0.1053$} & \CG{$1.2406\pm0.0398$} & \CG{$0.5353\pm0.0175$} & [5.0237]\\
\cmidrule{2-11}
& \mr{8}{6} 
& RK $K_{2,6}^B$ & [\textbf{1.0000}] &
$23.2642\pm0.9163$ & $7.2236\pm0.2647$ & $2.2068\pm0.0715$ & $0.7672\pm0.0248$ & $0.2628\pm0.0083$ & $0.0969\pm0.0033$ & [1.0000]\\
&& RK $K_{2,6}^E$ & [1.0700] &
$\mathbf{22.0099\pm0.9000}$ & $\mathbf{6.8727\pm0.2570}$ & $\mathbf{2.0995\pm0.0665}$ & $\mathbf{0.7287\pm0.0235}$ & $\mathbf{0.2569\pm0.0081}$ & $\mathbf{0.0961\pm0.0033}$ & [\textbf{0.9913}]\\
&& RK $K_{2,6}^G$ & [1.3282] &
\CG{$31.5057\pm1.1761$} & \CG{$10.7053\pm0.3818$} & \CG{$3.4107\pm0.1170$} & \CG{$1.1480\pm0.0379$} & \CG{$0.3806\pm0.0121$} & \CG{$0.1362\pm0.0044$} & [1.4056]\\
&& RK $K_{2,6}^L$& [5.4536] &
\CG{$45.7616\pm1.3157$} & \CG{$23.4046\pm0.7160$} & \CG{$10.3467\pm0.3170$} & \CG{$4.3718\pm0.1386$} & \CG{$1.6100\pm0.0501$} & \CG{$0.6141\pm0.0200$} & [6.3369]\\
&& PK of $K_{1,6}^B$ & [1.0419] &
$24.0760\pm0.9366$ & $7.5713\pm0.2742$ & \CG{$2.3077\pm0.0786$} & $0.7828\pm0.0243$ & $0.2714\pm0.0086$ & $0.1014\pm0.0034$ & [1.0463]\\
&& PK of $K_{1,6}^E$ & [1.0966] &
$22.6378\pm0.9059$ & $7.0804\pm0.2569$ & $2.1788\pm0.0742$ & $0.7332\pm0.0230$ & $0.2627\pm0.0082$ & $0.0991\pm0.0033$ & [1.0226]\\
&& PK of $K_{1,6}^G$ & [1.3577] &
\CG{$31.6913\pm1.1353$} & \CG{$10.8801\pm0.3880$} & \CG{$3.4745\pm0.1217$} & \CG{$1.1579\pm0.0373$} & \CG{$0.3831\pm0.0122$} & \CG{$0.1393\pm0.0046$} & [1.4379]\\
&& PK of $K_{1,6}^L$ & [7.3943] &
\CG{$46.4623\pm1.3305$} & \CG{$24.1137\pm0.7278$} & \CG{$11.0021\pm0.3326$} & \CG{$4.7613\pm0.1480$} & \CG{$1.8774\pm0.0591$} & \CG{$0.7473\pm0.0235$} & [7.7113]\\
\midrule
\mr{23}{3} 
& \mr{7}{2} 
& RK $K_{3,2}^B$ & [\textbf{1.0000}] &
$\mathbf{18.7642\pm0.6040}$ & $\mathbf{6.9539\pm0.1976}$ & $\mathbf{2.9496\pm0.0771}$ & $\mathbf{1.3446\pm0.0338}$ & $\mathbf{0.6420\pm0.0161}$ & $\mathbf{0.3099\pm0.0079}$ & [\textbf{1.0000}]\\
&& RK $K_{3,2}^E$ & [1.0864] &
$19.3876\pm0.6041$ & $7.4858\pm0.2149$ & $3.1403\pm0.0810$ & \CG{$1.4470\pm0.0358$} & \CG{$0.6943\pm0.0174$} & \CG{$0.3372\pm0.0083$} & [1.0879]\\
&& RK $K_{3,2}^G$ & [1.1864] &
\CG{$23.9051\pm0.7529$} & \CG{$8.7213\pm0.2773$} & \CG{$3.5969\pm0.0968$} & \CG{$1.6082\pm0.0404$} & \CG{$0.7779\pm0.0199$} & \CG{$0.3723\pm0.0095$} & [1.2013]\\
&& RK $K_{3,2}^L$& [2.3660] &
\CG{$53.3562\pm1.3928$} & \CG{$21.1493\pm0.6151$} & \CG{$8.8721\pm0.2516$} & \CG{$3.6972\pm0.0955$} & \CG{$1.7077\pm0.0461$} & \CG{$0.8118\pm0.0208$} & [2.6195]\\
&& PK of $K_{1,2}^B$ & [1.0448] &
$19.7222\pm0.6186$ & $7.3303\pm0.2124$ & $3.1041\pm0.0813$ & $1.4081\pm0.0351$ & $0.6781\pm0.0171$ & $0.3264\pm0.0083$ & [1.0531]\\
&& PK of $K_{1,2}^E$ & [1.1098] &
$19.8647\pm0.6456$ & \CG{$7.5905\pm0.2096$} & \CG{$3.2686\pm0.0869$} & \CG{$1.4949\pm0.0367$} & \CG{$0.7059\pm0.0179$} & \CG{$0.3438\pm0.0086$} & [1.1095]\\
&& PK of $K_{1,2}^L$& [2.9686] &
\CG{$52.7602\pm1.4184$} & \CG{$21.8585\pm0.6374$} & \CG{$9.2388\pm0.2486$} & \CG{$4.1435\pm0.1062$} & \CG{$1.9666\pm0.0503$} & \CG{$0.9162\pm0.0228$} & [2.9563]\\
\cmidrule{2-11}
& \mr{8}{4} 
& RK $K_{3,4}^B$ & [1.0000] &
$25.3986\pm0.8894$ & $8.4013\pm0.2747$ & $3.0604\pm0.0841$ & $1.2210\pm0.0318$ & $0.5073\pm0.0134$ & $0.2093\pm0.0057$ & [1.0000]\\
&& RK $K_{3,4}^E$ & [1.0861] &
$\mathbf{23.3127\pm0.8190}$ & $\mathbf{7.8901\pm0.2350}$ & $\mathbf{2.9518\pm0.0817}$ & $1.1936\pm0.0320$ & $0.5003\pm0.0133$ & $0.2084\pm0.0055$ & [0.9960]\\
&& RK $K_{3,4}^G$ & [1.3143] &
\CG{$37.2782\pm1.1748$} & \CG{$12.9172\pm0.4303$} & \CG{$4.5679\pm0.1240$} & \CG{$1.7626\pm0.0454$} & \CG{$0.7301\pm0.0194$} & \CG{$0.2955\pm0.0079$} & [1.4119]\\
&& RK $K_{3,4}^L$ & [3.9987] &
\CG{$70.6824\pm1.6955$} & \CG{$32.7709\pm0.8582$} & \CG{$14.6790\pm0.3851$} & \CG{$6.0564\pm0.1532$} & \CG{$2.4945\pm0.0655$} & \CG{$0.9916\pm0.0252$} & [4.7381]\\
&& PK of $K_{1,4}^B$ & [\textbf{0.9765}] &
\CG{$26.0719\pm0.8978$} & $8.4814\pm0.2697$ & $3.1138\pm0.0853$ & $1.2346\pm0.0318$ & $0.5133\pm0.0136$ & $0.2071\pm0.0056$ & [0.9896]\\
&& PK of $K_{1,4}^E$ & [1.0301] &
$24.4835\pm0.9137$ & $7.9345\pm0.2394$ & $2.9769\pm0.0823$ & $\mathbf{1.1889\pm0.0309}$ & $\mathbf{0.4941\pm0.0134}$ & $\mathbf{0.2045\pm0.0053}$ & [\textbf{0.9774}]\\
&& PK of $K_{1,4}^G$ & [1.2284] &
\CG{$35.2748\pm1.1315$} & \CG{$11.9953\pm0.3964$} & \CG{$4.2515\pm0.1147$} & \CG{$1.6530\pm0.0424$} & \CG{$0.6856\pm0.0182$} & \CG{$0.2761\pm0.0074$} & [1.3195]\\
&& PK of $K_{1,4}^L$& [5.4110] &
\CG{$70.0372\pm1.6944$} & \CG{$33.6666\pm0.8833$} & \CG{$15.2810\pm0.4095$} & \CG{$6.5560\pm0.1625$} & \CG{$2.8585\pm0.0761$} & \CG{$1.1353\pm0.0277$} & [5.4249]\\
\cmidrule{2-11}
& \mr{8}{6} 
& RK $K_{3,6}^B$ & [\textbf{1.0000}] &
$36.4994\pm1.2212$ & $11.6394\pm0.4029$ & $3.8464\pm0.1050$ & $1.4006\pm0.0358$ & $0.5244\pm0.0137$ & $0.1936\pm0.0051$ & [1.0000]\\
&& RK $K_{3,6}^E$ & [1.0769] &
$\mathbf{33.6589\pm1.1940}$ & $\mathbf{10.6750\pm0.3579}$ & $\mathbf{3.5802\pm0.1002}$ & $\mathbf{1.3213\pm0.0336}$ & $\mathbf{0.4910\pm0.0134}$ & $\mathbf{0.1844\pm0.0048}$ & [\textbf{0.9527}]\\
&& RK $K_{3,6}^G$ & [1.4103] &
\CG{$51.3040\pm1.4862$} & \CG{$19.3169\pm0.6060$} & \CG{$6.5221\pm0.1799$} & \CG{$2.3097\pm0.0603$} & \CG{$0.8633\pm0.0227$} & \CG{$0.3148\pm0.0082$} & [1.6260]\\
&& RK $K_{3,6}^L$ & [5.7653] &
\CG{$81.7929\pm1.8767$} & \CG{$42.6492\pm1.0293$} & \CG{$20.4169\pm0.5131$} & \CG{$9.4432\pm0.2305$} & \CG{$3.8039\pm0.0947$} & \CG{$1.3986\pm0.0339$} & [7.2246]\\
&& PK of $K_{1,6}^B$ & [1.0606] &
\CG{$37.5986\pm1.2218$} & \CG{$12.3234\pm0.3998$} & \CG{$4.1032\pm0.1129$} & \CG{$1.5007\pm0.0384$} & \CG{$0.5629\pm0.0146$} & \CG{$0.2027\pm0.0054$} & [1.0471]\\
&& PK of $K_{1,6}^E$ & [1.1092] &
$34.8445\pm1.1872$ & $11.1684\pm0.3591$ & $3.7513\pm0.1052$ & $1.3843\pm0.0355$ & $0.5196\pm0.0136$ & $0.1976\pm0.0048$ & [1.0206]\\
&& PK of $K_{1,6}^G$ & [1.4384] &
\CG{$51.3070\pm1.5192$} & \CG{$18.7785\pm0.5682$} & \CG{$6.5203\pm0.1789$} & \CG{$2.3292\pm0.0605$} & \CG{$0.8699\pm0.0227$} & \CG{$0.3156\pm0.0083$} & [1.6303]\\
&& PK of $K_{1,6}^L$& [9.1885] &
\CG{$84.8374\pm1.9266$} & \CG{$45.2915\pm1.1066$} & \CG{$22.4703\pm0.5508$} & \CG{$10.7301\pm0.2525$} & \CG{$4.6637\pm0.1161$} & \CG{$1.7211\pm0.0392$} & [8.8902]\\
\bottomrule
\end{tabular}\end{sc}}
\end{table}

For $q=2$, 
the MSE ratios were approximately close to the AMSE ratios,
which suggests that the AMSE criterion serves as a good indicator of real performance.
Also, the optimal Biweight kernel $K_{d,2}^B$ achieved the best estimation result 
for every $d$ and $n$ as expected from the asymptotic theories.
Although differences between MSEs for the Biweight kernel 
and those for other truncated kernels 
(RK $K_{d,2}^E$ and PKs of $K_{1,2}^B$ and $K_{1,2}^E$)
were not significant especially for smaller $n$,
differences between them and those for non-truncated kernels
(RKs $K_{d,2}^G$ and $K_{d,2}^L$ and PK of $K_{1,2}^L$) 
were statistically significant with $p$-value 0.05 for most combinations of $d$ and $n$.

For the higher orders $q=4, 6$,
truncated kernels performed well, whereas 
non-truncated ones gave significantly larger MSEs.
This tendency is the same as in the case $q=2$. 
However, the experimental results for the higher-order kernels 
exhibited deviations from the asymptotic theories:
except for the case of $(d,q)=(1,4)$,
even with the largest sample size $n=102400$ investigated, 
the smallest MSE values were achieved by 
kernels other than that with the minimum AMSE 
($K_{1,6}^E$ for $q=6$ when $d=1$, and PK of $K_{1,4}^B$ for $q=4$ and RK $K_{d,6}^B$ for $q=6$ when $d=2,3$).
Such deviations would be ascribed to the fact that asymptotics ignores residual terms of the AMSE as described at the beginning of this section.
The ratio between the leading-order and next-leading-order terms of the AMSE for the considered RKs and PKs is $O(n^{-\frac{4}{d+2q+2}})$, 
which gets larger as $d$ and/or $q$ increase.
Therefore, for the residual terms to be negligible, 
one needs more sample points for larger $d$ and/or $q$.
In the cases of $q=4,6$, even $n=102400$ might not have been large enough to accurately reflect 
the small difference less that about 10\ \% in the AMSE into the difference of MSE,
for the PDFs used.
However, considering that those kernels which performed inferior to the best-performing ones still performed 
close to the theoretical optimum with a difference less than 10\ \% in the AMSE criterion 
and that their experimental difference was not significant,
it can be found that even though the AMSE criterion may not be a quantitative performance index for higher-order kernels, it is still suggestive for real performance.

\section{Optimal Kernel for Other Methods}
\label{section:OtherMethods}
\subsection{In-Sample Mode Estimation}
\label{section:ISME}
A mode estimator considered in \citep{abraham2004asymptotic},
which we call the in-sample  mode estimator (ISME), 
is defined as the location of a sample point 
where a value of the KDE~\eqref{eq:KDE} becomes maximum among those at sample points:
\begin{align}
\label{eq:ISME}%
  \bar{\bt}_n\coloneq\argmax_{\bx\in\{\bX_i\}_{i=1}^n}f_n(\bx).
\end{align}
The ISME can be evaluated efficiently
with the quick-shift (QS) algorithm \citep{vedaldi2008quick}.
Although the QS algorithm has an extra tuning parameter in addition to $K$ and $h_n$,
which may affect the quality of the estimator, 
it has an advantage that it converges in a finite number of iterations irrespective of the kernel used as far as the sample size is finite,
making it computationally efficient.

\citet{abraham2004asymptotic} have proved in their Corollary~1.1 that, 
in the large-sample asymptotics, 
the ISME $\bar{\bt}_n$ converges in distribution to an asymptotic distribution of the KME $\bt_n$ 
if the bandwidth $h_n$ satisfies $n^{\frac{d-2}{2}}h_n^{\frac{d(d+2)}{2}}\ln n\to0$ as $n\to\infty$.
A simple calculation shows that,
if we use the optimal bandwidth $h_{d,q,n}^\opt\propto n^{-\frac{1}{d+2q+2}}$ in~\eqref{eq:OptBan},
the requirement for the bandwidth is fulfilled
for $(d,q)$ satisfying $(d-2)q<d+2$, that is,
for any $q$ when $d=1, 2$; $2\le q\le 4$ when $d=3$; and $q=2$ when $d=4, 5$.
Accordingly, at least for these cases,
the ISME has the same AMSE and the same AMSE criterion as the KME, 
so that the results on the optimal kernel still hold for the ISME as well: 
especially, $K_{d,q}^B$ minimizes the AMSE criterion of the ISME among the minimum-sign-change RKs,
for the above-mentioned combinations of $d$ and $q$.

\subsection{Modal Linear Regression}
\label{section:MLR}
Modal linear regression (MLR) \citep{yao2014new} aims 
to obtain a conditional mode predictor as a linear model.
In addition to intrinsic good properties rooted in a conditional mode, 
the MLR has an advantage that resulting parameter and regression line
are consistent even for a heteroscedastic or asymmetric conditional PDF, 
compared with robust M-type estimators which are not consistent in these cases \citep{baldauf2012use}.

Let $(\bX,\bY)$ be a pair of random variables in $\calX\times\calY\subset\R^{d_\X}\times\R^{d_\Y}$ following a certain joint distribution.
MLR assumes that the conditional distribution of $\bY$ conditioned on $\bX=\bx$ 
is such that the conditional PDF $f_\YX(\cdot|\bx)$ satisfies $\text{arg~max}_{\by\in\calY}f_\YX(\by|\bx)=\Theta\bx$ for any $\bx\in\calX$,
where $\Theta\in\R^{d_\Y\times d_\X}$ is an underlying MLR parameter.
For given i.i.d.~samples $\{(\bX_i,\bY_i)\in\calX\times\calY\}_{i=1}^n$ of $(\bX,\bY)$,
the MLR adopts, as the estimator of the parameter $\Theta$,
the global maximizer $\Theta_n$ of the kernel-based objective function
with argument $\Omega\in\R^{d_\Y\times d_\X}$, 
\begin{align}
\label{eq:MLR-Obj}
	O_n(\Omega)\coloneq\frac{1}{nh_n^{d_\Y}}
	\sum_{i=1}^n K\biggl(\frac{\Omega\bX_i-\bY_i}{h_n}\biggr),
\end{align}
with the kernel $K$ defined on $\R^{d_\Y}$ and the bandwidth $h_n>0$.

In \citep[Theorem~3]{kemp2019dynamic}, the asymptotic normality of the MLR has been proved.
They have considered the case where a 2nd order kernel is used 
and the AVC is relatively dominant compared with the AB 
by assuming $h_n^4\ll (nh_n)^{-(d_\Y+2)}$.
We provide an extension of their theorem which allows the use of a higher-order 
kernel and provides an explicit expression of the AB.
Note that the vectorization operator $\vecop(\cdot)$ is defined as the $nm$-vector 
$\vecop(\rmM)\coloneq(\bm{m}_1^\top,\ldots,\bm{m}_n^\top)^\top$ 
for an $m\times n$ matrix $\rmM=(\bm{m}_1,\ldots,\bm{m}_n)$, 
and that we adopt the definition that the operator $\nabla$ applied to a function with a matrix argument 
is the nabla operator with respect to the vectorization of the argument.
The same definition applies to the Hessian operator $H$ as well.

\begin{theorem}
\label{theorem:MLR}
Assume Assumption \ref{assumption:asmA.2} for an even integer $q\ge2$.
Then, the vectorization $\vecop(\Theta_n)$ of 
the MLR parameter estimator $\Theta_n$ asymptotically 
follows a normal distribution with the following AB and AVC: 
\begin{align}
  \mathrm{E}[\vecop(\Theta_n)-\vecop(\Theta)]
  \approx -h_n^q\rmA\bb,\quad
  \mathrm{Cov}[\vecop(\Theta_n)]
  \approx \frac{1}{nh_n^{d_\Y+2}}\rmA\rmV\rmA,
\end{align}
where the abbreviations $\rmA$, $\bb$, and $\rmV$ are defined by
\begin{align}
\label{eq:MLR-abbre}
\begin{split}
	\rmA&=\bigl\{\mathrm{E}[(\bX\otimes \rmI_{d_\Y})Hf_\YX(\Theta\bX|\bX)(\bX\otimes \rmI_{d_\Y})^\top]\bigr\}^{-1},\\
	\bb&=\mathrm{E}[\bX\otimes \bm{T}(\Theta\bX;\bX, f_\XY, K)],\\
	\rmV&=\mathrm{E}[f_\YX(\Theta\bX|\bX)(\bX\otimes \rmI_{d_\Y})\calV_{d_\Y}(\bX\otimes \rmI_{d_\Y})^\top].
\end{split}
\end{align}
In the above, $\bm{T}(\by;\bx, f_\XY, K)$ is defined as
\begin{align}
	\bm{T}(\by;\bx, f_\XY, K)
	\coloneq\sum_{\bi\in\mathbb{Z}_{\ge0}^{d_\Y}:|\bi|=q}
	\frac{1}{\bi!}
	\cdot \nabla \partial^\bi f_\YX(\by|\bx)
	\cdot \calB_{d_\Y,\bi}(K),
\end{align}
and $\calV_{d_\Y}$ is the $d_\Y\times d_\Y$ matrix defined via \eqref{eq:FuncVofK}.
Moreover, the AMSE is obtained as
\begin{align}
\label{eq:MLR-AMSE}
  \mathrm{E}[\|\vecop(\Theta_n)-\vecop(\Theta)\|^2]
  \approx h_n^{2q}\|\rmA\bb\|^2
  +\frac{1}{nh_n^{d_\Y+2}}\trace(\rmA\rmV\rmA).
\end{align}
\end{theorem}

By following the same line as the argument in Section~\ref{section:KME}, 
for instance, one can optimize the bandwidth $h_n$ in terms of the AMSE, 
and furthermore show that the AMSE criterion becomes 
$(B_{d_\Y,q}^{2(d_\Y+2)}\cdot V_{d_\Y,1}^{2q})^{\frac{1}{d_\Y+2q+2}}$ if an RK is used,
which implies that the kernel $K_{d_\Y,q}^B$ is optimal.
These results are generalization of \citep{yamasaki2020kernel} for $d_\Y=1$ and $q=2$.

\subsection{Mode Clustering}
\label{section:MC}
In mode clustering, 
a cluster is defined via the gradient flow $\bx'(t)=\nabla f(\bx(t))$ of the PDF $f$ regarded as a scalar field.
A limiting point of the gradient flow defines a center of the cluster corresponding to its domain of attraction.
For the mean-shift-based mode clustering \citep{comaniciu2002mean, yamasaki2019ms}, the KDE is often plugged into $f$.
Although clustering error for mode clustering using the KDE is difficult to evaluate in general dimensions,
\citet{casa2020} have analyzed clustering error rate in the univariate case in detail.

Here we review their discussion in a fairly simplified setting:
we consider the univariate case and assume that the PDF and KDE are bimodal, 
as depicted in Figure \ref{figure:CER}.
In this setting, the true clusters $C_1,C_2$ and estimated clusters $C_{n,1}, C_{n,2}$ are
\begin{align}
\begin{split}
	&C_1=\{x\in\R:x<\zeta\},\quad
	C_2=\{x\in\R:x\ge\zeta\},\\
	&C_{n,1}=\{x\in\R:x<\zeta_n\},\quad
	C_{n,2}=\{x\in\R:x\ge\zeta_n\},
\end{split}
\end{align}
with the local minimizers $\zeta$ and $\zeta_n$ of the PDF and KDE,
respectively, 
and the clustering error rate (CER) becomes
\begin{align}
	\mathrm{CER}_n
	=\sum_{j=1}^2 \Pr(X\in C_j\text{ and }X\not\in C_{n,j})
	=\begin{cases}
	\Pr(\zeta\le X<\zeta_n),&\zeta<\zeta_n,\\
	\Pr(\zeta_n\le X<\zeta),&\zeta_n<\zeta.
	\end{cases}
\end{align}
In Figure \ref{figure:CER}, the red area represents the CER.
The green area is $|\zeta_n-\zeta| \times f(\zeta)$, 
and the blue area is approximated as $|\zeta_n-\zeta| \times \frac{f^{(2)}(\zeta)}{2}|\zeta_n-\zeta|^2$, 
which is negligible in the asymptotic situation.
Therefore, the relationship `green area<red area (CER) $<$ green $+$ blue areas' implies 
that the CER asymptotically equals to the green area, $|\zeta_n-\zeta|\times f(\zeta)$: 
the asymptotic mean CER (AMCER) reduces to
\begin{align}
	\label{eq:AMCER}
	\lim_{n\to\infty}\mathrm{E}[\mathrm{CER}_n]
	=f(\zeta) \mathrm{E}[|\zeta_n-\zeta|].
\end{align}
On the basis of the fact that the form of the asymptotic distribution 
of the local minimizer $\zeta_n$ is the same as that of the mode, 
the AMCER behaves like the A$L^1$E for the mode, discussed in Section~\ref{section:Criteria}.

\begin{figure}[t]
\centering
\begin{overpic}[height=50mm, bb=0 0 665 335]{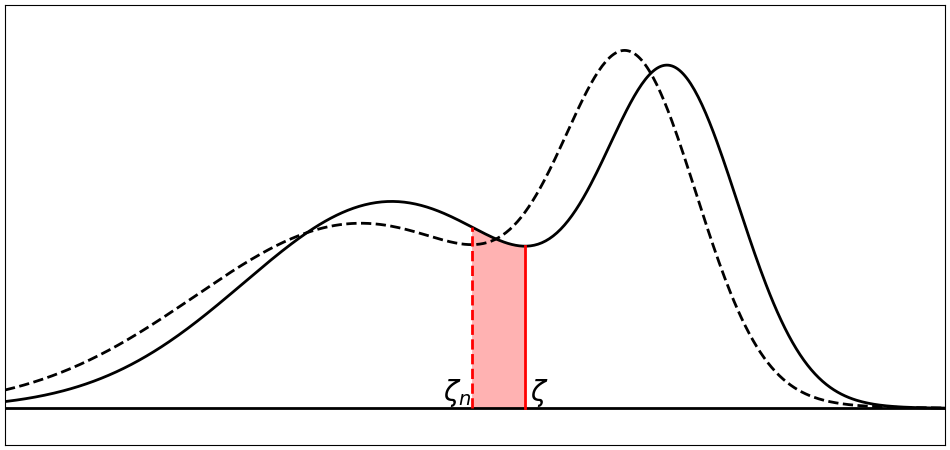}\put(3,42){\small(a)}\end{overpic}
\hspace{10pt}
\begin{overpic}[height=50mm, bb=0 0 105 335]{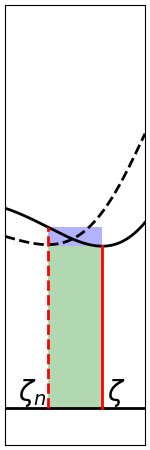}\put(5,85){\small(b)}\end{overpic}
\caption{
The black solid and dotted curves are the underlying PDF and KDE, respectively.
The red solid and dotted lines are located at local minimizers 
($\zeta$ and $\zeta_n$) of the underlying PDF and KDE, respectively.
In panel~(a), the red area represents the clustering error rate.
In panel~(b), the green area is a lower bound of the error rate, 
and the summation of the green and blue areas is an upper bound.}
\label{figure:CER}
\end{figure}

\citet{casa2020} have shown $n$-dependence of the optimal bandwidth minimizing the AMCER.
We further show the kernel dependence of the optimal bandwidth and the optimal kernel:
Since the AMCER behaves like the A$L^1$E for the mode, 
the optimal bandwidth minimizing the AMCER becomes
$h_{1,q,1,n}^\opt=c_1(V_{1,1}/(nB_{1,q})^2)^{\frac{1}{2q+3}}$, $c_1>0$,
and the resulting kernel optimization problem becomes equivalent to
the type-1 kernel optimization problem for the KME with the A$L^1$E criterion as the objective function.
Thus, the kernel $K_{1,q}^B$ minimizes the AMCER for the optimal bandwidth among the $q$-th order minimum-sign-change kernels.

\section{Conclusion}
\label{section:Conclusion}
In this paper, we have studied the kernel selection, particularly optimal kernel, 
for the kernel mode estimation, by extending the existing approach for the univariate setting by 
\citep{gasser1979kernel, gasser1985kernels, granovsky1989optimality, granovsky1991optimizing}.
This approach is based on asymptotics, and it seeks a kernel function that 
minimizes the main term of the AMSE for the optimal bandwidth among kernels (or their derivatives) 
that change their sign the minimum number of times required by their order.
Theorems~\ref{theorem:thm21} and \ref{theorem:thm22} are the main novel theoretical contribution, 
which shows that the Biweight hierarchy $\{G_{d,q}^B\}$ provides the optimal RK $K_{d,q}^B(\cdot)=G_{d,q}^B(\|\cdot\|)$
for every dimension $d\ge1$ and even order $q=2,4$ and $q\ge2$, respectively.
We have also studied the PK model, and compared it with the RK model. 
As a result, we have found that the 2nd order optimal RK $K_{d,2}^B$ 
is superior in terms of the AMSE to the PK model using any non-negative kernel. 
Simulation studies confirmed the superiority of the optimal Biweight kernel among the non-negative kernels,
as well as truncated kernels including the optimal kernel $K_{d,q}^B$, which gave better accuracy even in higher orders $q=4,6$.

The theories on the optimal kernel are also effective for 
other kernel-based modal statistical methods.
In Section~\ref{section:OtherMethods}, we show the optimal kernel for 
in-sample  mode estimation, modal linear regression, and univariate mode clustering.
Although elucidation of the optimal kernel for other methods that might be relevant such as principal curve estimation 
or density ridge estimation \citep{ozertem2011locally, sando2020modal} and multivariate mode clustering, 
in which it is difficult to represent asymptotic errors explicitly, is an open problem, 
studying them analytically or experimentally will also be interesting.

\appendix
\renewcommand{\theequation}{\Alph{section}.\arabic{equation}}
\renewcommand{\thedefinition}{\Alph{section}.\arabic{definition}}
\renewcommand{\thetheorem}{\Alph{section}.\arabic{theorem}}
\renewcommand{\theproposition}{\Alph{section}.\arabic{proposition}}
\renewcommand{\thelemma}{\Alph{section}.\arabic{lemma}}
\renewcommand{\theproblem}{\Alph{section}.\arabic{problem}}
\renewcommand{\theremark}{\Alph{section}.\arabic{remark}}
\renewcommand{\theassumption}{\Alph{section}.\arabic{assumption}}
\section{Proof of Asymptotic Behaviors of Kernel Mode Estimation}
\label{section:ProofKME}
\setcounter{equation}{0}
\setcounter{definition}{0}
\setcounter{theorem}{0}
\setcounter{proposition}{0}
\setcounter{lemma}{0}
\setcounter{problem}{0}
\setcounter{remark}{0}
\setcounter{assumption}{0}
Regularity conditions for Theorem~\ref{theorem:AN}, i.e., sufficient conditions for 
deriving the expressions~\eqref{eq:ABVC}, \eqref{eq:AMSE} of the AB, AVC, and AMSE 
along with the asymptotic normality of the KME, are listed below.
They consist of the conditions on the sample $\{\bX_i\in\R^d\}_{i=1}^n$ \hyl{A.1},
mode $\bt$ and PDF $f$ \hyl{A.2}--\hyl{A.6}, kernel $K$ \hyl{A.6}--\hyl{A.14},
and bandwidth $h_n$ \hyl{A.15} and \hyl{A.16}.
\begin{assumption}[Regularity conditions for Theorem~\ref{theorem:AN}]
\label{assumption:asmAN}
For finite $d\ge1$ and even $q\ge2$,
\begin{itemize}\setlength{\leftskip}{0.3cm}\setlength{\parskip}{0cm}\setlength{\itemsep}{0cm}
\item[\hyt{A.1}]%
	$\{\bX_i\in\R^d\}_{i=1}^n$ is a sample of i.i.d.~observations from $f$.
\item[\hyt{A.2}]%
	$f$ is $(q+2)$ times differentiable in $\R^d$
	(i.e., $f$ and $f^\bi$, $|\bi|=2$ are continuous at $\bt$).
\item[\hyt{A.3}]%
	$f$ has a unique and isolated maximizer at $\bt$
	(i.e., $f(\bx)<f(\bt)$ for all $\bx\neq\bt$, $\nabla f(\bt)=\bm{0}_d$ due to \hyl{A.2}, 
	and $\sup_{\bx\not\in N}f(\bx)<f(\bt)$ for a neighborhood $N$ of $\bt$).
\item[\hyt{A.4}]%
	$\partial^\bi f$, $|\bi|=2$ satisfies $\int|\partial^\bi f(\bx)|\,d\bx<\infty$,
	and $Hf(\bt)$ is non-singular.
\item[\hyt{A.5}]%
	$|\partial^\bi f(\bt)|<\infty$ for all $\bi$ s.t.\,$|\bi|=2,\ldots,q+1$, 
	and $\partial^\bi f$ is bounded for all $\bi$ s.t.\,$|\bi|=q+2$.
\item[\hyt{A.6}]%
	$\partial^\bi f(\bt)$, $|\bi|=q+1$ and $K$ satisfy $\bb(\bt;f,K)\neq\bm{0}_d$.
\item[\hyt{A.7}]%
	$K$ is bounded and twice differentiable in $\R^d$
	and satisfies the covering number condition, $\int |K(\bx)|\,d\bx<\infty$, 
	and $\lim_{\|\bx\|\to\infty} \|\bx\| |K(\bx)|=0$.
\item[\hyt{A.8}]%
	$\calB_{d,\bm{0}_d}(K)=1$.
\item[\hyt{A.9}]%
	$\calB_{d,\bi}(K)=0$ for all $\bi$ s.t.\,$|\bi|=1,\ldots,q-1$.
\item[\hyt{A.10}]%
	$|\calB_{d,\bi}(K)|<\infty$ for all $\bi$ s.t.\,$|\bi|=q$, 
	and $\calB_{d,\bi}(K)\neq0$ for some $\bi$ s.t.\,$|\bi|=q$.
\item[\hyt{A.11}]%
	$|\calB_{d,\bi}(K)|<\infty$ for all $\bi$ s.t.\,$|\bi|=q+1$.
\item[\hyt{A.12}]%
	$\partial_i K$ is bounded and satisfies $\int |\partial_i K(\bx)\partial_j K(\bx)|\,d\bx<\infty$, 
	$\lim_{\|\bx\|\to\infty}\|\bx\| |\partial_i K(\bx)\partial_j K(\bx)|=0$, 
	and $\int |\partial_i K(\bx)|^{2+\delta}\,d\bx<\infty$ for some $\delta>0$,
	for all $i,j=1,\ldots,d$.
\item[\hyt{A.13}]%
	$\int \nabla K(\bx)\nabla K(\bx)^\top\,d\bx=\calV_d(K)$ has a finite determinant.
\item[\hyt{A.14}]%
	$\partial^\bi K$, $|\bi|=2$ satisfies the covering number condition.
\item[\hyt{A.15}]%
	$\lim_{n\to\infty}(n h_n^{d+2q+2})^{\frac{1}{2}}=c<\infty$ (i.e., $h_n\to0$).
\item[\hyt{A.16}]%
	$\lim_{n\to\infty}n h_n^{d+4}/\ln n=\infty$.
\end{itemize}
\end{assumption}
\addtocounter{equation}{16}
\noindent%
Here, the covering number condition, that appears in \hyl{A.7} and \hyl{A.14}, is defined as follow:
\begin{definition}[{\citep[Definition~23]{pollard2012convergence} and \citep[Section~2]{mokkadem2003law}}]
\hfill
\begin{itemize}
\item 
Let $P$ be a probability measure on $S$ and $\calF$ be a class of 
functions in $\calL_1(P)\coloneq\{g:|\int_S g(\bx)\,P(d\bx)|<\infty\}$.
For each $\epsilon>0$, define the \emph{covering number} $N(\epsilon,P,\calF)$ as the smallest value of $m$ 
for which there exist functions $g_1,\ldots,g_m$ such that $\min_j \int_S |f-g_j|\,P(d\bx)\le\epsilon$
for each $f\in\calF$.
We set $N(\epsilon,P,\calF)=\infty$ if no such $m$ exists.
\item 
Let $g$ be a function defined on $\R^d$.
Let $\calF(g)=\{g(\frac{\cdot\,-\bx}{h}),h>0,\bx\in\R^d\}$ be 
the class of functions consisting of arbitrarily translated and scaled versions of $g$. 
We say that $g$ satisfies the \emph{covering number condition} 
if $g$ is bounded and integrable on $\R^d$, 
and if there exist $A>0$ and $W>0$ such that $N(\epsilon,P,\calF(g))\le A\epsilon^{-W}$ 
for any probability measure $P$ on $\R^d$ and any $\epsilon\in(0,1)$.
\end{itemize}
\end{definition}
\noindent%
For example, the RK $K(\bx)=G(\|\bx\|)$ and the PK $K(\bx)=\prod_{j=1}^d G(x_j)$ fulfill 
the covering number condition in~\hyl{A.7} if $G$ has bounded variations in both cases.

We give a proof of the asymptotic normality of the KME $\bt_n$, which is 
a modification of the proof of \citep[Theorem~2.2]{mokkadem2003law}, below.
\begin{proof}[{Proof of Theorem~\ref{theorem:AN}}]
\citet{pollard2012convergence} gives the following lemma
on uniform consistency (see his Theorem~37 in p.~34):
\begin{lemma}
\label{lemma:CNC}
Let $g$ be a function on $\R^d$ satisfying the covering number condition, 
and $\{h_n\}$ be a sequence of positive numbers satisfying $\lim_{n\to\infty}h_n=0$. 
If there exists $j\ge0$ such that $\lim_{n\to\infty}nh_n^{d+2j}/\ln n=\infty$, then 
\begin{align}
	\lim_{n\to\infty}\sup_{\bm{z}\in\R^d}\frac{1}{nh_n^{d+j}}
	\biggl|\sum_{i=1}^n g\biggl(\frac{\bm{z}-\bm{Z}_i}{h_n}\biggr)-\mathrm{E}\biggl[g\biggl(\frac{\bm{z}-\bm{Z}}{h_n}\biggr)\biggr]\biggr|=0,
	\quad\text{a.s.,}
\end{align}
where $\mathrm{E}$ is the expectation value with respect to the distribution of the random vector $\bm{Z}\in\R^d$, 
and where $\{\bm{Z}_i\in\R^d\}_{i=1}^n$ is a sample of i.i.d.~observations of $\bm{Z}$.
\end{lemma}

Also, one has the following lemma (see \citep[Theorem 1.1.1]{bochner1955harmonic}):
\begin{lemma}
\label{lemma:BOCH}%
Let $g_1$ be a function on $\R^d$ satisfying 
$\sup_{\bu\in\R^d}|g_1(\bu)|<\infty$, $\int |g_1(\bu)|\,d\bu<\infty$, and $\lim_{|\bu|\to\infty}\|\bu\||g_1(\bu)|=0$, 
$g_2$ be a function on $\R^d$ satisfying $\int |g_2(\bu)|\,d\bu<\infty$,
and $\{h_n\}$ be a sequence of positive numbers satisfying $\lim_{n\to\infty}h_n=0$.
Then one has that, for any $\bx\in\R^d$ of continuity of $g_2$, 
\begin{align}
	\lim_{n\to\infty}\frac{1}{h_n^d}\int g_1\biggl(\frac{\bu}{h_n}\biggr) g_2(\bx-\bu)\,d\bu
	=g_2(\bx)\int g_1(\bu)\,d\bu.
\end{align}
\end{lemma}

Applying Lemma~\ref{lemma:CNC} with $g=K$ and $j=0$, 
under \hyl{A.1}, the covering number condition of $K$ \hyl{A.7}, 
$h_n\to0$ \hyl{A.15}, and $n h_n^d/\ln n\to\infty$ \hyl{A.16}, ensures that
\begin{align}
	\lim_{n\to\infty}\sup_{\bx\in\R^d}
	\bigl|f_n(\bx)-\mathrm{E}[f_n(\bx)]\bigr|=0,\quad\text{a.s.}
\end{align}
Also, Lemma~\ref{lemma:BOCH} with $g_1=K$ and $g_2=f$ 
under the fact $\int|f(\bx)|\,d\bx=1<\infty$, the continuity of $f$ \hyl{A.2}, 
\hyl{A.7}, \hyl{A.8}, and $h_n\to0$ \hyl{A.15} implies that
\begin{align}
	\mathrm{E}[f_n(\bt)]
	=\frac{1}{h_n^d}\int K\biggl(\frac{\bt-\bu}{h_n}\biggr) f(\bu)\,d\bu
	=f(\bt) \int K(\bu)\,d\bu
	=f(\bt).
\end{align}
According to these results, one has $\lim_{n\to\infty}f_n(\bt)=f(\bt)$ a.s.,
which implies the consistency of $\bt_n$ to $\bt$ from
contradiction to the $(\epsilon,\delta)$-definition of limit: 
for any $\epsilon>0$, there exists a $\delta>0$ such that 
$|f_n(\bt_n)-f(\bt)|\ge\delta$ if $\|\bt_n-\bt\|\ge\epsilon$.

Because $\bt_n$ maximizes $f_n$ (i.e., $\nabla f_n(\bt_n)=\bm{0}_d$)
and $K$ and hence $f_n$ are twice differentiable \hyl{A.7},
Taylor's approximation of $\nabla f_n(\bx)$ at $\bt_n$ around $\bt$ shows
\begin{align}
\label{eq:Th1Tay}
	\bm{0}_d=\nabla f_n(\bt_n)=\nabla f_n(\bt)+Hf_n(\bt^*)(\bt_n-\bt),
\end{align}
that is, $\bt_n=\bt-\{Hf_n(\bt^*)\}^{-1}\nabla f_n(\bt)$ if $Hf_n(\bt^*)$ is invertible,
where $\bt^*$ satisfies $\|\bt^*-\bt\|\le\|\bt_n-\bt\|$.
We thus study the asymptotic behaviors of $Hf_n(\bt^*)$ and $\nabla f_n(\bt)$ below.

Applying Lemma~\ref{lemma:CNC} with $g=\partial^\bi K$, $|\bi|=2$ and $j=2$
under \hyl{A.1}, the covering number condition of $\partial^\bi K$ \hyl{A.14}, 
$h_n\to0$ \hyl{A.15}, and \hyl{A.16}, and the consistency of $\bt_n$ to $\bt$
shows that $Hf_n(\bt^*)$ converges to $\mathrm{E}[Hf_n(\bt)]$.
Moreover, integration-by-parts and Lemma~\ref{lemma:BOCH} with $g_1=K$ and $g_2=\partial^\bi f$, $|\bi|=2$ 
under the continuity of $\partial^\bi f$ \hyl{A.2}, \hyl{A.4}, \hyl{A.7}, \hyl{A.8}, and $h_n\to0$ \hyl{A.15}, 
leads that, for every $\bi$ satisfying $|\bi|=2$,
\begin{align}
	\begin{split}
	\mathrm{E}[\partial^\bi f_n(\bt)]
	&=\frac{1}{h_n^{d+2}}\int \partial^\bi K\biggl(\frac{\bt-\bu}{h_n}\biggr) f(\bu)\,d\bu
	=\frac{1}{h_n^d}\int K\biggl(\frac{\bt-\bu}{h_n}\biggr) \partial^\bi f(\bu)\,d\bu\\
	&=\partial^\bi f(\bt) \int K(\bu)\,d\bu
	=\partial^\bi f(\bt),
	\end{split}
\end{align}
that is, $\mathrm{E}[H f_n(\bt)]=H f(\bt)$.
Combining these results,
one can find that $\{Hf_n(\bt^*)\}^{-1}$ is consistent to the matrix $\rmA$.

We next show that $\nabla f_n(\bt)$ asymptotically follows the normal distribution
with the mean $h_n^q\bb$ and the variance-covariance matrix $(nh_n^{d+2})^{-1}\rmV$, 
on the basis of Lyapounov's central limit theorem.
The critical difference from the existing proof by \citep{mokkadem2003law} 
is on the assumptions~\hyl{A.6} and \hyl{A.10}, 
which affects on the calculation of the mean of $\nabla f_n(\bt)$.
On the basis of the multivariate Taylor expansion of $\partial_j f(\bt-h_n\bx)$ around $\bt$
under the setting that $q$ is even,
\begin{align}
	\partial_j f(\bt-h_n\bx)
	=\partial_j f(\bt)-\cdots
	+\biggl\{\sum_{|\bi|=q} \frac{(h_n \bx)^\bi}{\bi !} 
	\cdot \partial_j \partial^\bi f(\bt)\biggr\}
	-\biggl\{\sum_{|\bi|=q+1} \frac{(h_n \bx)^\bi}{\bi !} 
	\cdot \partial_j \partial^\bi f(\bt_\bx)\biggr\},
\end{align}
where $\bt_\bx$ lies in $\|h_n\bx\|$-neighborhood of $\bt$ and depends on $\bx$,
one has that
\begin{align}
	\begin{split}
	\mathrm{E} \bigl[\partial_j f_n(\bt)\bigr]
	&=\int \frac{1}{h_n^{d+1}} \partial_j K \biggl(\frac{\bt-\bx}{h_n}\biggr) f(\bx)\,d\bx
	=\int K(\bx) \partial_j f(\bt-h_n\bx)\,d\bx\\
	&=\int K(\bx) \biggl\{\cdots + \sum_{|\bi|=q} \frac{(h_n \bx)^\bi}{\bi !} 
	\cdot \partial_j \partial^\bi f(\bt)-\cdots\biggr\}\,d\bx.
	\end{split}
\end{align}
Because $\partial_j f(\bt)=0$ \hyl{A.3}, 
integrations of $2,\ldots,q$-th summands vanish 
due to finiteness of $\partial^\bi f(\bt)$, $|\bi|=2,\ldots,q$ \hyl{A.5} and \hyl{A.9},
and integration of the $(q+2)$-nd residual term is asymptotically ignorable due to 
boundedness of $\partial^\bi f$, $|\bi|=q+2$ \hyl{A.5}, \hyl{A.11}, and $h_n\to0$ \hyl{A.15},
one can approximate the asymptotic mean of $\nabla f_n(\bt)$ as
\begin{align}
	\frac{1}{h_n^q}\mathrm{E}\bigl[\nabla f_n(\bt)\bigr]_j
	\to \partial_j\biggl( \int K(\bx) \sum_{|\bi|=q} \frac{\bx^\bi}{\bi!} 
	\cdot \partial^\bi f(\bt)\,d\bx\biggr)
	=\bigl(\bb(\bt;f,K)\big)_j,
\end{align}
under the assumptions \hyl{A.6} and \hyl{A.10}.
Moreover, 
by using Lemma~\ref{lemma:BOCH} with $g_1=\partial_iK\partial_jK$ and $g_2=f$
under the continuity of $f$ \hyl{A.2}, \hyl{A.12}, and $h_n\to0$ \hyl{A.15},
one can calculate the variance-covariance matrix of $\nabla f_n(\bt)$ 
with the definition \hyl{A.13}:
\begin{align}
	\begin{split}
	&nh_n^{d+2}\mathrm{Cov}[\nabla f_n(\bt)]_{i,j}
	=nh_n^{d+2}\frac{1}{n}\mathrm{Cov}\biggl[
	\frac{1}{h_n^{d+1}}\partial_i K\left(\frac{\bt-\bX}{h_n}\right),
	\frac{1}{h_n^{d+1}}\partial_j K\left(\frac{\bt-\bX}{h_n}\right)\biggr]\\
	&=\frac{1}{h_n^d}\int \partial_i K\left(\frac{\bt-\bx}{h_n}\right)
	\partial_j K\left(\frac{\bt-\bx}{h_n}\right)f(\bx)\,d\bx
	\to f(\bt)[\calV_d(K)]_{i,j}=\rmV_{i,j}.
	\end{split}
\end{align}

Finally, by applying Slutsky's theorem to these pieces, 
the proof on asymptotic normality of $\{Hf_n(\bt^*)\}^{-1}\nabla f_n(\bt)$ is finished.
\end{proof}

The condition~\hyl{A.6} is added to those in \citep{mokkadem2003law}, 
to ensure that the main term of AB $\propto\rmA\bb$ does not vanish.
Also, the condition~\hyl{A.10} is a modification of condition (A5) iii) in \citep{mokkadem2003law},
the latter of which does not apply to RKs unlike their description.
However, it should be noted that the kernel $K_{d,q}^B$ does not satisfy 
the twice differentiability among the sufficient conditions on the kernel.
\citet{eddy1980optimum} has studied behaviors of the KME $\bt_n$ within 
$O((n h_n^{d+2})^{-\frac{1}{2}})$-neighborhood of the mode $\bt$ and tried the proof without assuming 
that the kernel $K$ is twice differentiable (such that the kernel $K_{d,q}^B$ fulfills their requirements),
but their proof is not rigorous in that it does not consider the possibility that the KME $\bt_n$ 
exists outside of that neighborhood.
We have considered several methods, including the proof in \citep{eddy1980optimum}, 
and then concluded that twice differentiability of $K$ appears to be necessary for deriving the AMSE.

\section{Theories on Optimal Kernel}
\setcounter{equation}{0}
\setcounter{definition}{0}
\setcounter{theorem}{0}
\setcounter{proposition}{0}
\setcounter{lemma}{0}
\setcounter{problem}{0}
\setcounter{remark}{0}
\subsection{Relationships between the Moment Condition and the Number of Sign Changes}
In this section we describe relationships between the moment condition and the number of sign changes of a kernel.
We first introduce several notations.
To represent a kernel class defined from the moment condition,
we introduce, for integers $d\ge1$, $l$, and $k$ satisfying $0\le l<k$, 
\begin{align}
\label{eq:calMdlk}
	\calM_{d,l,k}\coloneq
	\{g:B_{d,i}(g)=(-1)^l(d)_lb_d^{-1}\;(i=l),\;0\;(i\in I_{k-2}\backslash\{l\}),\;\text{non-zero}\;(i=k)\},
\end{align}
where $(n)_m\coloneq n(n+1)\cdots(n+m-1)$ denotes the Pochhammer symbol.
Also, we represent a class of functions on $\R_{\ge0}$
with the prescribed number of sign changes as
\begin{align}
	\calN_j\coloneq\{g:\text{$g$ changes its sign $j$ times on $\R_{\ge0}$}\}.
\end{align}
Here, the number of the sign changes is defined as follows:
\begin{definition}
\label{definition:SigChe}
A function $g$ defined on a finite or infinite interval $[a,b]$ 
is said to change its sign $k$ times on $[a,b]$, 
if there are $(k+1)$ subintervals $S_j=[x_{j-1},x_j]$, $j=1,\ldots,k+1$, 
where $a=x_0<x_1<\cdots<x_{k+1}=b$, such that:
\begin{itemize}\setlength{\parskip}{0cm}\setlength{\itemsep}{0cm}
\item[{(i)}] $g(x)\le0$ (or $\ge0$) for all $x\in S_j$, 
and there exists $x\in S_j$ such that $g(x)<0$ (or $>0$), 
for each $j\in\{1,\ldots,k+1\}$.
\item[{(ii)}] If $k\ge1$, $g(x)g(y)\le0$ for all $x\in S_j$, $y\in S_{j+1}$, $j=1,\ldots,k$.
\end{itemize}
\end{definition}
\noindent%
Note that a function could have an interval over which it equals to zero and 
that a point where its sign changes does not have to be uniquely determined.
Additionally, we introduce the following term:
\begin{definition}
Functions $g_1$ and $g_2$ defined on a finite or infinite interval $[a,b]$ 
are said to share the same sign-change pattern 
if the same set of subintervals in Definition~\ref{definition:SigChe} 
applies to both $g_1$ and $g_2$. 
\end{definition}

In the following, we provide a lemma (Lemma~\ref{lemma:MinCro})
about the minimum number of sign changes of a $q$-th order RK.
It is a multivariate extension of \citep[Lemma~2]{gasser1985kernels} 
stating that a univariate $q$-th order kernel changes its sign at least $(q-2)$ times on $\R$,
and becomes a basis for the condition~\hyl{P2-3} in Problem~\ref{problem:Prob21}
and condition~\hyl{P3-3} in Problem~\ref{problem:Prob22}.
The following Lemma~\ref{lemma:nonsingular} is 
for Lemmas~\ref{lemma:MinCro} and \ref{lemma:step3}.
\begin{lemma}
\label{lemma:nonsingular}
Let $S_1,\ldots,S_k$ be measurable subsets of $\R_{\ge0}$ with $\mu(S_i)>0$, 
where $\mu$ is the Lebesgue measure on $\R_{\ge0}$, 
and assume that they are ordered in the sense that for any $x\in S_i$ and any $y\in S_j$ with $i<j$ 
the inequality $x<y$ holds almost surely. 
Let $w(x)$ be a measurable function which does not change its sign inside each $S_j$, $j=1,\ldots,k$. 
Assume further that $\int_{S_j}w(x)\,dx\not=0$ holds. 
Then the following $k\times k$ matrix is non-singular:
\begin{align}
	\rmM=
	\left[\begin{array}{cccc}
	\int_{S_1}w(x)\,dx&\int_{S_2}w(x)\,dx&\cdots&\int_{S_k}w(x)\,dx\\
	\int_{S_1}w(x)x^2\,dx&\int_{S_2}w(x)x^2\,dx&\cdots&\int_{S_k}w(x)x^2\,dx\\
	\vdots&\vdots&\ddots&\vdots\\
	\int_{S_1}w(x)x^{2(k-1)}\,dx&\int_{S_2}w(x)x^{2(k-1)}\,dx&\cdots&\int_{S_k}w(x)x^{2(k-1)}\,dx
	\end{array}\right].
\end{align}
\end{lemma}
\begin{proof}
	Assume to the contrary that $\rmM$ is singular.
	Then there exists a non-zero vector $\bm{a}=(a_1,\ldots,a_k)^\top$ 
	which satisfies $\rmM^\top\bm{a}=\bm{0}_k$. 
	Rewriting it component-wise, one has 
	\begin{align}
		\label{eq:IntNull}
		\sum_{i=0}^{k-1}a_{i+1}\int_{S_j}w(x)x^{2i}\,dx
		=\int_{S_j}w(x)P(x)\,dx=0,
		\quad j=1,\ldots,k,
	\end{align}
	where we let 
	\begin{align}
		\label{eq:Poly}
		P(x)=\sum_{i=0}^{k-1}a_{i+1}x^{2i}.
	\end{align}
	Take an interval $[a,b]$ with $a=\inf S_j$ and $b=\sup S_j$. 
	Then $w(x)\mathbbm{1}_{x\in S_j}$ is integrable, 
	does not change its sign on $[a,b]$, 
	and $\int_a^bw(x)\mathbbm{1}_{x\in S_j}\,dx\not=0$. 
	From~\eqref{eq:IntNull} and the mean value theorem, there exists 
	$\alpha_j\in(a,b)$ satisfying $P(\alpha_j)=0$. 
	Since $P(x)$ is an even function of $x$, 
	one has $P(-\alpha_j)=0$. 
	The zeros $\alpha_1,\ldots,\alpha_k$ are all distinct. 
	Thus $P(x)$ is factorized as 
	\begin{align}
		P(x)=(x^2-\alpha_1^2)\cdots(x^2-\alpha_k^2)Q(x),
	\end{align}
	where $Q(x)$ is a polynomial not identically equal to 0. 
	The degree of the right-hand side is therefore at least $2k$, 
	whereas from~\eqref{eq:Poly} the degree of $P(x)$ is at most $2(k-1)$, 
	leading to contradiction. 
\end{proof}
\begin{lemma}
\label{lemma:MinCro}
For $d\ge1$ and even $q\ge2$,
if a function $g$ defined on $\R_{\ge0}$ is such that $g\in\calM_{d,0,q}$ or $g\in\calM_{d,1,q+1}$, 
then $g$ changes its sign at least $(\frac{q}{2}-1)$ times on $\R_{\ge0}$.
\end{lemma}
\begin{proof}
Assume that $g$ changes its sign $(k-1)$ times on $\R_{\ge0}$, 
and decompose $\R_{\ge0}$ into $k$ subintervals $S_1,\ldots,S_k$ 
according to the way in Definition \ref{definition:SigChe}: 
the sign of $g$ changes alternately in these subintervals.

\textit{Proof for $g\in\calM_{d,0,q}$.}
Lemma~\ref{lemma:nonsingular} 
shows that the $k\times k$ matrix
\begin{align}
	\rmM=
	\begin{bmatrix}
	\int_{S_1}x^{d+1}g(x)\,dx&\int_{S_2}x^{d+1}g(x)\,dx&\cdots&\int_{S_k}x^{d+1}g(x)\,dx\\
	\int_{S_1}x^{d+3}g(x)\,dx&\int_{S_2}x^{d+3}g(x)\,dx&\cdots&\int_{S_k}x^{d+3}g(x)\,dx\\
	\vdots&\vdots&\ddots&\vdots\\
	\int_{S_1}x^{d+2k-1}g(x)\,dx&\int_{S_2}x^{d+2k-1}g(x)\,dx&\cdots&\int_{S_k}x^{d+2k-1}g(x)\,dx
	\end{bmatrix}
\end{align}
is non-singular. 
Assume $k<\frac{q}{2}$ as opposed to the statement of this lemma.
Then, from the moment conditions $g\in\calM_{d,0,q}$, 
one would have 
\begin{align}
	\rmM\bm{1}_k=
	\begin{bmatrix}
	\int_{\mathbb{R}_{\ge0}} x^{d+1}g(x)\,dx\\
	\int_{\mathbb{R}_{\ge0}} x^{d+3}g(x)\,dx\\
	\vdots\\
	\int_{\mathbb{R}_{\ge0}} x^{d+2k-1}g(x)\,dx
	\end{bmatrix}
	=\bm{0}_k.
\end{align}
This contradicts the non-singularity of $\rmM$.
Thus, $k\ge\frac{q}{2}$ holds, i.e., 
$g$ changes its sign at least $(\frac{q}{2}-1)$ times on $\R_{\ge0}$.

\textit{Proof for $g\in\calM_{d,1,q+1}$.}
For the matrix
\begin{align}
	\rmM=
	\begin{bmatrix}
	\int_{S_1}x^{d+2}g(x)\,dx&\int_{S_2}x^{d+2}g(x)\,dx&\cdots&\int_{S_k}x^{d+2}g(x)\,dx\\
	\int_{S_1}x^{d+4}g(x)\,dx&\int_{S_2}x^{d+4}g(x)\,dx&\cdots&\int_{S_k}x^{d+4}g(x)\,dx\\
	\vdots&\vdots&\ddots&\vdots\\
	\int_{S_1}x^{d+2k}g(x)\,dx&\int_{S_2}x^{d+2k}g(x)\,dx&\cdots&\int_{S_k}x^{d+2k}g(x)\,dx
	\end{bmatrix}
\end{align}
one can take a similar proof as the case for $g\in\calM_{d,0,q}$.
\end{proof}

Also, in the case of $q=2, 4$, 
we can know the order in which the kernel changes its sign:
\begin{lemma}
\label{lemma:LemmaA.1.1}
Let $d\ge1$ and $q=2$ or $4$,
and assume that a function $G_{d,q}$ is defined on $\R_{\ge0}$ and such that $G_{d,q}\in\calM_{d,0,q}\cap\calN_{\frac{q}{2}-1}$.
Then, $G_{d,q}(x)\ge0$ for $x\in S_1\cup S_3\cup\cdots$ and $G_{d,q}(x)\le0$ for $x\in S_2\cup S_4\cup\cdots$,
where the subintervals $S_1,\ldots,S_{\frac{q}{2}}$ of $\R_{\ge0}$ are defined according to Definition~\ref{definition:SigChe}.
Moreover, $\sign(B_{d,q}(G_{d,q}))=(-1)^{q/2+1}$.
\end{lemma}
\begin{proof}
\textit{Proof for $q=2$.}
A 2nd order kernel satisfying the normalization condition and the minimum-sign-change condition is non-negative, 
and the statements $G_{d,2}(x)\ge0$ for any $x\in\R_{\ge0}$ and $\sign(B_{d,2}(G_{d,2}))=+1$ are trivial.

\textit{Proof for $q=4$.}
The sign change of minimum-sign-change 4-th order kernel $G_{d,4}$ occurs at a single point, which is denoted as $\rho>0$.
If one assumes that $G_{d,4}(x)\le 0$ for $x\in[0,\rho]$, 
then one has the following contradiction:
On the basis of the mean value theorem of integration, 
\begin{align}
	\begin{split}
	&0=B_{d,2}(G_{d,4})
	=\int_0^\infty x^{d+1}G_{d,4}(x)\,dx
	=\xi_1^2\int_0^\rho x^{d-1}G_{d,4}(x)\,dx
	+\xi_2^2 \int_\rho^\infty x^{d-1}G_{d,4}(x)\,dx\\
	&=\underbrace{\xi_1^2}_{>0} 
	\underbrace{\int_0^\infty x^{d-1}G_{d,4}(x)\,dx}_{>0~\because B_{d,0}(G_{d,4})>0}
	+\underbrace{(\xi_2^2-\xi_1^2)}_{>0} 
	\underbrace{\int_\rho^\infty x^{d-1}G_{d,4}(x)\,dx}_{>0}
	>0,
	\end{split}
\end{align}
where $0<\xi_1<\rho<\xi_2$.
Thus, $G_{d,4}(x)\ge0$ for $0\le x\le\rho$ and $G_{d,4}(x)\le0$ for $x\ge\rho$.
Also, one has
\begin{align}
	\begin{split}
	&B_{d,4}(G_{d,4})
	= \int_0^\infty x^{d+3}G_{d,4}(x)\,dx
	=\tilde{\xi}_1^2\int_0^\rho x^{d+1}G_{d,4}(x)\,dx
	+\tilde{\xi}_2^2 \int_\rho^\infty x^{d+1}G_{d,4}(x)\,dx\\
	&=\underbrace{\tilde{\xi}_1^2}_{>0} 
	\underbrace{\int_0^\infty x^{d+1}G_{d,4}(x)\,dx}_{=0~\because B_{d,2}(G_{d,4})=0}
	+\underbrace{(\tilde{\xi}_2^2-\tilde{\xi}_1^2)}_{>0} 
	\underbrace{\int_\rho^\infty x^{d+1}G_{d,4}(x)\,dx}_{<0}<0,
	\end{split}
\end{align}
where $0<\tilde{\xi}_1<\rho<\tilde{\xi}_2$.
Therefore, $\sign(B_{d,4}(G_{d,4}))=-1$ is confirmed.
\end{proof}

\subsection{Proof of Theorem~\protect\ref{theorem:thm21}}
First we show that the kernel $G_{d,q}^B$ satisfies all the conditions of Problem~\ref{problem:Prob21}.
\begin{lemma}
\label{lemma:Confirm}
The kernel $G_{d,q}^B$ satisfies all the conditions \hyl{P2-1}--\hyl{P2-5} with $G=G_{d,q}^B$.
\end{lemma}
\begin{proof}
The fact that $G_{d,q}^B$ satisfies \hyl{P2-3}, that is,
it changes its sign $(\frac{q}{2}-1)$ times on $\R_{\ge0}$,
is a direct consequence of the well-known property
of the Jacobi polynomials, that $P_n^{(\alpha,\beta)}(\cdot)$
with $\alpha,\beta>-1$ has $n$ simple zeros in the interval $(-1,1)$, 
together with the rightmost expression in~\eqref{eq:Opt-RK2} of $G_{d,q}^B$.
Differentiability of $G_{d,q}^B$ on $\R_{\ge0}$ 
and boundedness and continuity in the sense of `a.e.' of $G_{d,q}^{B(1)}$ (condition \hyl{P2-2}),
finiteness of $V_{d,1}(G_{d,q}^B)$ (condition \hyl{P2-4}), and
the behavior $x^{d+q}G_{d,q}^B(x)\to 0$ as $x\to\infty$ (condition \hyl{P2-5}) 
follow straightforwardly from the fact that,
from the rightmost expression in~\eqref{eq:Opt-RK2},
$G_{d,q}^B(x)$ is a polynomial restricted to the finite interval $x\in[0,1]$
with a double zero at $x=1$.


We thus show in the following that $G_{d,q}^B\in\calM_{d,0,q}$ holds,
that is, $G_{d,q}^B$ satisfies the moment condition~\hyl{P2-1}.
Let
\begin{align}
\begin{split}
  B_{d,i}(G_{d,q}^B)
  &=(-1)^{\frac{q}{2}+1}
  \frac{\Gamma(\frac{d+q}{2}+2)}{\pi^{\frac{d}{2}}\Gamma(\frac{q}{2}+2)}
  \int_0^1x^{d-1+i}(1-x^2)^2P_{\frac{q}{2}-1}^{(2,\frac{d}{2})}(2x^2-1)\,dx\\
  &=(-1)^{\frac{q}{2}+1}
  \frac{\Gamma(\frac{d+q}{2}+2)}{\pi^{\frac{d}{2}}\Gamma(\frac{q}{2}+2)}
  I_{d,q,i},
\end{split}
\end{align}
where
\begin{align}
  \label{eq:Idqi}
\begin{split}
  I_{d,q,i}
  \coloneq&\,\int_0^1x^{d-1+i}(1-x^2)^2P_{\frac{q}{2}-1}^{(2,\frac{d}{2})}(2x^2-1)\,dx\\
  =&\,\frac{1}{2^{\frac{d+i}{2}+3}}\int_{-1}^1(1-y)^2(1+y)^{\frac{d+i}{2}-1}P_{\frac{q}{2}-1}^{(2,\frac{d}{2})}(y)\,dy.
\end{split}
\end{align}
Using the identity~\cite[Section 2.22.2, Item 11]{PrudnikovBrychkovMarichev1986}
\begin{align}
  \int_{-1}^1(1-y)^\alpha(1+y)^{\lambda-1}
  P_n^{(\alpha,\beta)}(y)\,dy
  =(-1)^n2^{\alpha+\lambda}B(\alpha+n+1,\lambda)
  {\beta-\lambda+n\choose n},
  \quad\lambda>0,\,\alpha,\beta>-1,
\end{align}
where
\begin{align}
  {x\choose y}=\frac{\Gamma(x+1)}{\Gamma(y+1)\Gamma(x-y+1)},
  \quad x,y\in\R
\end{align}
is the binomial coefficient extended to real-valued arguments,
and where $B(\cdot,\cdot)$ denotes the Beta function, 
one has
\begin{align}
  I_{d,q,i}
  =\frac{(-1)^{\frac{q}{2}-1}}{2}
  B\Bigl(\frac{q}{2}+2,\frac{d+i}{2}\Bigr)
  {\frac{q}{2}-1-\frac{i}{2}\choose\frac{q}{2}-1}.
\end{align}
The values of $I_{d,q,i}$ for $i\in\{0,2,\ldots,q\}$ are therefore
\begin{align}
  I_{d,q,i}=\left\{
  \begin{array}{ll}
    \displaystyle\frac{(-1)^{\frac{q}{2}-1}}{2}
    B\Bigl(\frac{q}{2}+2,\frac{d}{2}\Bigr),&i=0,\\
    0,&i\in\{2,\ldots,q-2\},\\
    \displaystyle\frac{1}{2}
    B\Bigl(\frac{q}{2}+2,\frac{d+q}{2}\Bigr),&i=q.
  \end{array}\right.
\end{align}
These in turn imply that the moments of $G_{d,q}^B$ are given by
\begin{align}
  B_{d,i}(G_{d,q}^B)
  =\left\{
  \begin{array}{ll}
    \displaystyle\frac{\Gamma(\frac{d}{2})}{2\pi^{\frac{d}{2}}}
    =\frac{1}{b_d},&i=0,\\
    0,&i\in\{2,\ldots,q-2\},\\
    \displaystyle\frac{(-1)^{\frac{q}{2}+1}}{2\pi^{\frac{d}{2}}}
    \frac{\Gamma(\frac{d+q}{2})\Gamma(\frac{d+q}{2}+2)}
    {\Gamma(\frac{d}{2}+q+2)},&i=q,
  \end{array}\right.
\end{align}
showing that $G_{d,q}^B\in\calM_{d,0,q}$ holds.
(The fact that $B_{d,i}(G_{d,q}^B)=0$ for $i\in\{2,\ldots,q-2\}$ can
alternatively be understood on the basis of the expression~\eqref{eq:Idqi} and
the orthogonality relation~\eqref{eq:JacobiPOrth} of the Jacobi polynomials.)
\end{proof}

Then, we introduce two polynomials related to the kernel $K_{d,q}^B(\cdot)=G_{d,q}^B(\|\cdot\|)$, 
before the succeeding proof:
\begin{align}
\label{eq:B.PGR}
	P_{d,q}^B(x)\coloneq \frac{\Gamma(\frac{d+q}{2})}{\pi^{\frac{d}{2}}\Gamma(\frac{q}{2})}P_{\frac{q}{2}+1}^{(\frac{d}{2},-2)}(1-2x^2),\quad
	R_{d,q}^B(x)\coloneq (d-1)\frac{P^{B(1)}_{d,q}(x)}{x}+P^{B(2)}_{d,q}(x).
\end{align}
$P_{d,q}^B$ is an original polynomial of $G_{d,q}^B$ before truncated.
Also, the function $R_{d,q}^B$ is defined such that the first derivative of $x^{d-1}P_{d,q}^{B(1)}(x)$ becomes $x^{d-1}R_{d,q}^B(x)$, 
and $R_{d,q}^B$ is a polynomial function with terms of $0,2,\ldots,(q-2)$-nd degree.

In the following, we give Lemma~\ref{lemma:LemmaA.2.4} on the properties of functions $P_{d,q}^B$ and $R_{d,q}^B$.
\begin{lemma}
\label{lemma:LemmaA.1.2}
Suppose that $P(x)=a_0(x^2-a_1^2)\cdots(x^2-a_n^2)$ has coefficients $a_0\neq0$, $0<a_1\le\cdots\le a_n$.
Then, $P(x)/x^j$, $j=0,\ldots,2n$; $P^{(1)}(x)/x^j$, $j=0,\ldots,2n-2$; 
and $P^{(2)}(x)/x^j$, $j=0,\ldots,2n-4$ are strictly monotonically increasing (resp.~decreasing) for $x>a_n$, 
when $a_0>0$ (resp.~$a_0<0$).
\end{lemma}
\begin{proof}
From simple calculus, one obtains
\begin{align}
\label{eq:A.1.2P12}
	\begin{split}
	&P^{(1)}(x)=2a_0\sum_{i=1}^n x (x^2-a_1^2)\cdots(x^2-a_{i-1}^2)(x^2-a_{i+1}^2)\cdots(x^2-a_n^2),\\
	&P^{(2)}(x)=2a_0\sum_{i=1}^n (x^2-a_1^2)\cdots(x^2-a_{i-1}^2)(x^2-a_{i+1}^2)\cdots(x^2-a_n^2)\\
	&\hphantom{P^{(2)}(x)}
	+4a_0\sum_{1\le i<j\le n} x^2(x^2-a_1^2)\cdots(x^2-a_{i-1}^2)(x^2-a_{i+1}^2)\cdots(x^2-a_{j-1}^2)(x^2-a_{j+1}^2)\cdots(x^2-a_n^2).
	\end{split}
\end{align}
The first statement in~\citep[Lemma~4]{gasser1985kernels} shows the statement on $P(x)/x^j$, $j=0,\ldots,2n$ in this lemma.
Also, applying the latter (resp.~former) statement of their lemma to every summand of $P^{(1)}(x)$ (resp.~$P^{(2)}(x)$) in~\eqref{eq:A.1.2P12}, 
a result on $P^{(1)}(x)/x^j$, $j=0,\ldots,2n-2$ (resp.~$P^{(2)}(x)/x^j$, $j=0,\ldots,2n-4$) can be proved.
\end{proof}
\begin{lemma}
\label{lemma:LemmaA.2.4}
The polynomial $P_{d,q}^B$ in~\eqref{eq:B.PGR} can be represented 
as $P_{d,q}^B(x)=a_0(x^2-a_1^2)\cdots(x^2-a_{\frac{q}{2}+1}^2)$, 
$a_0>0$ (resp.~$a_0<0$) when $q=2,6,\ldots$ (resp.~$q=4,8,\ldots$), 
$0<a_1<\cdots<a_{\frac{q}{2}}=a_{\frac{q}{2}+1}=1$.
Also, for the polynomial $R_{d,q}^B$ in~\eqref{eq:B.PGR}, 
$R_{d,q}^B(x)/x^j$, $j=0,\ldots,q-2$ is monotonically increasing and positive 
(resp.~decreasing and negative) for $x>1$ when $q=2,6,\ldots$ (resp.~$q=4,8,\ldots$).
\end{lemma}
\begin{proof}
The symmetric polynomial function $P_{d,q}^B$ has $(q-2)$ zeros inside $(-1,1)$ and double zeros at $x=\pm1$, 
Thus, its coefficients become $0<a_1<\cdots<a_{\frac{q}{2}-1}<1$ and $a_{\frac{q}{2}}=a_{\frac{q}{2}+1}=1$.

Using the relationship of the Jacobi polynomial
\begin{align}
\label{eq:Poly1}
	P_n^{(\alpha,\beta)}(1)=\frac{(\alpha+1)_n}{n!},
\end{align}
which appears in~\citep[p.7]{askey1975orthogonal}, 
one has 
\begin{align}
	P_{d,q}^B(0)
	=\frac{\Gamma(\frac{d+q}{2})}{\pi^{\frac{d}{2}}\Gamma(\frac{q}{2})}P_{\frac{q}{2}+1}^{(\frac{1}{2},-2)}(1)
	=\frac{\Gamma(\frac{d+q}{2})}{\pi^{\frac{d}{2}}\Gamma(\frac{q}{2})}\frac{(\frac{3}{2})_{\frac{q}{2}+1}}{(\frac{q}{2}+1)!}>0.
\end{align}
Then, on the basis of the factorized expression of $P_{d,q}^B$, 
\begin{align}
	P_{d,q}^B(0)
	=a_0(-a_1^2)\cdots(-a_{\frac{q}{2}+1}^2)
	=(-1)^{\frac{q}{2}+1}a_0a_1^2\cdots a_{\frac{q}{2}+1}^2
	>0,
\end{align}
one has $a_0>0$ (resp.~$a_0<0$) when $q=2,6,\ldots$ (resp.~$q=4,8,\ldots$).

Finally, the increase or decrease of $R_{d,q}^B(x)/x^j$, $j=0,\ldots,q-2$ for $x>1$ is a direct corollary of Lemma~\ref{lemma:LemmaA.1.2}.
On the basis of the calculation~\eqref{eq:A.1.2P12}, 
one has
\begin{align}
	\begin{split}
	R_{d,q}^B(1)
	&=(d-1)\frac{P^{B(1)}_{d,q}(1)}{1}+P^{B(2)}_{d,q}(1)
	=(d-1)\frac{0}{1}+4 a_0 1^2 (1-a_1)^2\cdots (1-a_{\frac{q}{2}-1})^2\\
	&\propto a_0 (1-a_1)^2\cdots (1-a_{\frac{q}{2}-1})^2
	\end{split}
\end{align}
which, together with the above-mentioned fluctuation of $R_{d,q}^B(x)/x^0=R_{d,q}^B(x)$, ensures that the sign of $R_{d,q}^B(x)$, $x>1$ and that of $a_0$ match with each other.
\end{proof}

Finally, we provide a proof of Theorem~\ref{theorem:thm21} below:
\begin{proof}[Proof of Theorem~\ref{theorem:thm21}]
Suppose that $G_{d,q}=G_{d,q}^B+Q_{d,q}:\R_{\ge0}\to\R$ 
satisfies the conditions \hyl{P2-1}--\hyl{P2-5} and 
has the same $q$-th moment $B_{d,q}$ as that of $G_{d,q}^B$.
For this setting, 
because $G_{d,q}^B$ satisfies \hyl{P2-1}--\hyl{P2-5} (shown in Lemma~\ref{lemma:Confirm}),
the disturbance $Q_{d,q}$ needs to satisfy the moment conditions
\begin{align}
\label{eq:B.mom}
	B_{d,i}(Q_{d,q})=0,\quad i\in I_q.
\end{align}
From the radial symmetry, odd-order moment conditions are automatically fulfilled.
Additionally, it should be noted that $Q_{d,q}$ does not need to have a compact support, 
which makes the kernel's optimality not limited to just the class of truncation kernels.
In contrast, the following proof requires the condition~\hyl{P2-5} (i.e., $\lim_{x\to\infty}x^{d+q}Q_{d,q}(x)=0$).

The following calculations are performed to derive the inequality~\eqref{eq:B.ineq}.
On the basis of the integral by part, 
$P_{d,q}^{B(1)}(0)=0$, $\lim_{x\to\infty}x^{d+q}Q_{d,q}(x)=0$, 
combination of the moment conditions~\eqref{eq:B.mom} and 
the fact that the polynomial $R_{d,q}^B$ defined in~\eqref{eq:B.PGR} has $0,2,\ldots,q$-th degree terms, 
one has
\begin{align}
\label{eq:B.1}
	\begin{split}
	\int_0^\infty x^{d-1}P_{d,q}^{B(1)}(x)Q_{d,q}^{(1)}(x)\,dx
	&=\biggl[x^{d-1}P_{d,q}^{B(1)}(x)Q_{d,q}(x)\biggr]_0^\infty
	-\int_0^\infty x^{d-1}R_{d,q}^B(x)Q_{d,q}(x)\,dx\\
	&=0.
	\end{split}
\end{align}
Secondly, 
integral by part, together with 
$P_{d,q}^{B(1)}(1)=0$ and $\lim_{x\to\infty}x^{d+q}Q_{d,q}(x)=0$, 
shows
\begin{align}
\label{eq:B.2}
	\begin{split}
	\int_1^\infty x^{d-1}P_{d,q}^{B(1)}(x)Q_{d,q}^{(1)}(x)\,dx
	&=\biggl[x^{d-1}P_{d,q}^{B(1)}(x)Q_{d,q}(x)\biggr]_1^\infty
	-\int_1^\infty x^{d-1}R_{d,q}^B(x)Q_{d,q}(x)\,dx\\
	&=	-\int_1^\infty x^{d-1}R_{d,q}^B(x)Q_{d,q}(x)\,dx.
	\end{split}
\end{align}
Collecting these pieces, one can obtain the following inequality:
\begin{align}
	\label{eq:B.ineq}
	\begin{split}
	&V_{d,1}(G_{d,q})-V_{d,1}(G_{d,q}^B)
	=V_{d,1}(Q_{d,q})
	+2\int_0^\infty x^{d-1}G_{d,q}^{B (1)}(u)Q_{d,q}^{(1)}(x)\,dx\\
	&\ge 2\int_0^\infty x^{d-1}G_{d,q}^{B (1)}(u)Q_{d,q}^{(1)}(x)\,dx
	= 2\int_0^\infty x^{d-1}\{G_{d,q}^{B (1)}(x)-P_{d,q}^{B (1)}(x)\}Q_{d,q}^{(1)}(x)\,dx\\
	&= -2\int_1^\infty x^{d-1}P_{d,q}^{B(1)}(x)Q_{d,q}^{(1)}(x)\,dx
	= 2\int_1^\infty x^{d-1}R_{d,q}^B(x)Q_{d,q}(x)\,dx,
	\end{split}
\end{align}
where the equality holds just when $Q^{(1)}_{d,q}(x)\equiv 0$, 
which holds only if $Q_{d,q}(x)\equiv 0$ due to the moment condition $B_{d,0}(Q_{d,q})=0$.

Thus, since the kernels $G_{d,q}$ and $G_{d,q}^B$ have the same value of $q$-th moment, 
in order to prove the optimality of $G_{d,q}^B$ with respect to the AMSE criterion, 
it is sufficient to prove
\begin{align}
	\label{eq:B.TODO}
	\int_1^\infty x^{d-1}R_{d,q}^B(x)Q_{d,q}(x)\,dx\ge0
\end{align}
for $Q_{d,q}(x)\not\equiv 0$,
instead of showing $V_{d,1}(G_{d,q})>V_{d,1}(G_{d,q}^B)$.
We prove the inequality~\eqref{eq:B.TODO} for order $q=2,4$, below.

\textit{Proof for $q=2$.}
First, we provide the proof for $q=2$.
Lemma~\ref{lemma:LemmaA.2.4} shows $R_{d,q}^B(x)>0$ for $x>1$.
Also, since $G_{d,2}$ is non-negative due to the minimum-sign-change condition and $G_{d,2}^B(x)=0$ for $x>1$, 
the disturbance $Q_{d,2}(x)=G_{d,2}(x)$, $x>1$ has to be non-negative.
Thus,~\eqref{eq:B.TODO} holds.

\textit{Proof for $q=4$.}
We next show the theorem for $q=4$.
The only sign change of minimum-sign-change 4-th order kernel $G_{d,4}$ goes from 
positive to negative at some point $\rho>0$ from Lemma~\ref{lemma:LemmaA.1.1}.

If $\rho\le1$, then $G_{d,4}(x)=Q_{d,4}(x)$, $x\ge1\ge\rho$ is non-positive.
Also, Lemma~\ref{lemma:LemmaA.2.4} shows that $R_{d,4}^B(x)<0$ for $x\ge1$.
Thus, the inequality~\eqref{eq:B.TODO} holds.

Next we consider the other case with $\rho>1$, where the followings are satisfied:
\begin{align}
\label{eq:B.3}
	&G_{d,4}(x)\ge 0\text{ for }x\le\rho\,
	(\text{which implies }Q_{d,4}(x)\ge0\text{ for }x\in[1,\rho]),\\
\label{eq:B.4}
	&G_{d,4}(x)=G_{d,4}^B(x)+Q_{d,4}(x)=Q_{d,4}(x)\le 0\text{ for }x\ge\rho(>1).
\end{align}
From the equation 
\begin{align}
	0
	=B_{d,2}(G_{d,4})
	=\underbrace{\int_0^1 x^{d+1}G_{d,4}(x)\,dx}_{\ge0~\because\eqref{eq:B.3}}
	+\int_1^\infty x^{d+1}Q_{d,4}(x)\,dx,
\end{align}
one has 
\begin{align}
\label{eq:B.5}
	\int_1^\infty x^{d+1}Q_{d,4}(x)\,dx\le 0,
\end{align}
which leads us to obtain
\begin{align}
\label{eq:B.6}
	\int_\rho^\infty x^{d+1}Q_{d,4}(x)\,dx
	=\int_1^\infty x^{d+1}Q_{d,4}(x)\,dx
	-\int_1^\rho x^{d+1}Q_{d,4}(x)\,dx
	\le0.
\end{align}
On the basis of the mean value theorem of integration, one gets 
\begin{align}
	\begin{split}
	&\int_1^\infty x^{d-1}R_{d,4}^B(x)Q_{d,4}(x)\,dx
	=\frac{R_{d,4}^B(\zeta_1)}{\zeta_1^2} \int_1^\rho x^{d+1}Q_{d,4}(x)\,dx
	+\frac{R_{d,4}^B(\zeta_2)}{\zeta_2^2} \int_\rho^\infty x^{d+1}Q_{d,4}(x)\,dx\\
	&=\underbrace{\frac{R_{d,4}^B(\zeta_1)}{\zeta_1^2}}_{<0}
	\underbrace{\int_1^\infty x^{d+1}Q_{d,4}(x)\,dx}_{\le0~\because \eqref{eq:B.5}}
	+\underbrace{\biggl(\frac{R_{d,4}^B(\zeta_2)}{\zeta_2^2}-\frac{R_{d,4}^B(\zeta_1)}{\zeta_1^2}\biggr)}_{<0}
	\underbrace{\int_\rho^\infty x^{d+1}Q_{d,4}(x)\,dx}_{\le0~\because \eqref{eq:B.6}}
	\ge0,
	\end{split}
\end{align}
for $1<\zeta_1<\zeta_2$,
where we use $R_{d,4}^B(1)<0$ and that $R_{d,4}^B(x)/x^2$ is monotonically decreasing for $x>1$, 
which result from Lemma~\ref{lemma:LemmaA.2.4}.
This inequality implies the completion of the proof.
\end{proof}


\subsection{Proof of Theorem~\protect\ref{theorem:thm22}}
\subsubsection{Reformulation of Problem~\protect\ref{problem:Prob22}}
In this section, we provide a proof of Theorem~\ref{theorem:thm22} that is 
a multivariate RK extension of the existing results by~\citep{granovsky1991optimizing}.
We here reformulate Problem~\protect\ref{problem:Prob22} to Problem~\ref{problem:Prob4}, 
to which we can almost directly apply the proof in the existing works.

Problem~\ref{problem:Prob22} includes both the function $G$ and derivative $G^{(1)}$
in its formulation, which makes the problem difficult to handle.
%
Under the condition~\hyl{P3-5} (i.e., $x^{d+q}G(x)\to0$ as $x\to\infty$), 
integration-by-parts yields the equation,
\begin{align}
\label{eq:IntPar}
	\begin{split}
	B_{d,i}(G)
	&=\int_0^\infty x^{d-1+i}G(x)\,dx
	=\biggl[\frac{1}{d+i}x^{d+i}G(x)\biggr]_0^\infty
	-\frac{1}{d+i}\int_0^\infty x^{d+i}G^{(1)}(x)\,dx\\
	&=-\frac{1}{d+i}B_{d,i+1}(G^{(1)}),
	\quad\text{for}\;i\in I_q.
	\end{split}
\end{align}
This equation allows us to rewrite the objective function \hyl{P3} as a functional of $G^{(1)}$ only, 
and shows the following lemma that allows us to represent the moment condition \hyl{P3-1} in terms of $G^{(1)}$.
\begin{lemma}
\label{lemma:equiv12}
For $d\ge1$ and even $q\ge2$,
assume that $x^{d+q}G(x)\to0$ holds as $x\to\infty$
and that $G^{(1)}$ is continuous a.e.
Then the following two conditions are equivalent:
\begin{itemize}
\item[(i)] $G \in \calM_{d,0,q}$ and $G$ is differentiable.
\item[(ii)] $G^{(1)}\in\calM_{d,1,q+1}$.
\end{itemize}
\end{lemma}
\noindent%
Thus, the following problem is equivalent to Problem~\ref{problem:Prob22}
in the sense that a solution of one of the problems is a solution of the other.
\begin{problem}
\label{problem:Prob3}%
Find $G^*(x)=-\int_x^\infty F^*(t)\,dt$, where
$F^*$ is a solution of the following optimization problem,
with $G(x)\coloneq-\int_x^\infty F(t)\,dt$:
\begin{align}
	\tmin_F\quad
	&B^{2(d+2)}_{d,q+1}(F)\cdot V^{2q}_{d,0}(F),
\tag{\hypertarget{P4}{P4}}\\
	\tst\quad
	&F\in\calM_{d,1,q+1},
\tag{\hypertarget{P4-1}{P4-1}}\\
	&F\;\text{is integrable},
\tag{\hypertarget{P4-2}{P4-2}}\\
	&F\in\calN_{\frac{q}{2}-1},
\tag{\hypertarget{P4-3}{P4-3}}\\
	&V_{d,0}(F)<\infty,
\tag{\hypertarget{P4-4}{P4-4}}\\
	&\text{there exists }\delta>0\text{ s.t.~}x^{d+q+\delta}G(x)\to0\;\text{as}\;x\to\infty.
\tag{\hypertarget{P4-5}{P4-5}}
\end{align}
\end{problem}

The kernel optimization problem \ref{problem:Prob3} still has scale indeterminacy of a solution,
as described also in Footnote~\ref{ft:ft3}.
Indeed, one has the following lemma.
\begin{lemma}
\label{lemma:equiv23}
Let $d\ge1$ and $q=2,4,\ldots$
For a function $F$ and $s>0$, 
we define a scale-transformed function $F_s(\cdot)\coloneq s^{-(d+1)} F(\cdot/s)$.
Then, one has
\begin{align}
	B_{d,i}(F_s)=s^{i-1} B_{d,i}(F),\quad
	V_{d,0}(F_s)=s^{-(d+2)}V_{d,0}(F).
\end{align}
\end{lemma}
\noindent%
This lemma immediately shows that the objective function~\hyl{P4}
is scale-invariant, that is, for any $s>0$ one has
\begin{align}
	B_{d,q+1}^{2(d+2)}(F_s)\cdot V_{d,0}^{2q}(F_s)
	=B_{d,q+1}^{2(d+2)}(F)\cdot V_{d,0}^{2q}(F).
\end{align}
It also shows that the set $\calM_{d,1,k}$ is scale-invariant,
that is, $F\in\calM_{d,1,k}$ implies $F_s\in\calM_{d,1,k}$ for any $s>0$.
The following Problem~\ref{problem:Prob4} resolves the scale indeterminacy.
\begin{problem}
\label{problem:Prob4}%
Find $G^*(x)=-\int_x^\infty F^*(t)\,dt$, where
$F^*$ is a solution of the following optimization problem for a given constant $C>0$,
with $G(x)\coloneq-\int_x^\infty F(t)\,dt$:
\begin{align}
	\tmin_F\quad
	&V_{d,0}(F),
\tag{\hypertarget{P5}{P5}}\\
	\tst\quad
	&F\in\calM_{d,1,q+1},
\tag{\hypertarget{P5-1}{P5-1}}\\
	&F\;\text{is integrable},
\tag{\hypertarget{P5-2}{P5-2}}\\
	&F\in\calN_{\frac{q}{2}-1},
\tag{\hypertarget{P5-3}{P5-3}}\\
	&V_{d,0}(F)<\infty,
\tag{\hypertarget{P5-4}{P5-4}}\\
	&\text{there exists }\delta>0\text{ s.t.~}x^{d+q+\delta}G(x)\to0\;\text{as}\;x\to\infty,
\tag{\hypertarget{P5-5}{P5-5}}\\
	&|B_{d,q+1}(F)|=(d+q)C.
\tag{\hypertarget{P5-6}{P5-6}}
\end{align}
\end{problem}

\subsubsection{Auxiliary Lemmas}
\label{sec:Hilbert}
Here we introduce notations and lemmas that will be used for proofs in the next sections.
Let $d$ be a positive integer. 
For functions $f,g$ defined on $\R_{\ge0}=[0,\infty)$, define
\begin{align}
  \langle f,g\rangle_d\coloneq \int_0^\infty x^{d-1}f(x)g(x)\,dx.
\end{align}
Let $L^2(\R_{\ge0},x^{d-1})$ be the space of square-integrable functions
on $\R_{\ge0}$ with weight $x^{d-1}$.
This is a Hilbert space equipped with the inner product
$\langle\cdot,\cdot\rangle_d$.
Note however that, for any nonnegative integer $i$,
one has $x^i\not\in L^2(\R_{\ge0},x^{d-1})$, as the integral
of $(x^i)^2$ over $\R_{\ge0}$ with weight $x^{d-1}$ diverges.

With a measurable subset $S$ of $\R_{\ge0}$, define 
\begin{align}
  \langle f,g\rangle_{S,d}\coloneq \int_S x^{d-1}f(x)g(x)\,dx,
\end{align}
and let $L^2(S,x^{d-1})$ be the space of square-integrable functions
on $S$ with weight $x^{d-1}$.
Let $D$ be a compact measurable subset of $\R_{\ge0}$,
and let $D^c=\R_{\ge0}\backslash D$.
One has that $L^2(D,x^{d-1})$ and $L^2(D^c,x^{d-1})$ are both
subspaces of $L^2(\R_{\ge0},x^{d-1})$, and that
the map $V:f\to f\mathbbm{1}_{x\in D}\oplus f\mathbbm{1}_{x\in D^c}$
from $L^2(\R_{\ge0},x^{d-1})$ to
$L^2(D,x^{d-1})\oplus L^2(D^c,x^{d-1})$ is
an isomorphism~\cite[page 25]{Conway2007}.
We therefore identify $L^2(\R_{\ge0},x^{d-1})$
and $L^2(D,x^{d-1})\oplus L^2(D^c,x^{d-1})$ via this isomorphism.

Let $I$ be a finite subset of $\mathbb{Z}_{\ge0}$, and let
\begin{align}
  M_{D,I}\coloneq \mathop{\mathrm{span}}\{x^i\mathbbm{1}_{x\in D},i\in I\}.
\end{align}
For fixed $D$ and $I$, $M_{D,I}$ is a subspace of $L^2(D,x^{d-1})$
with dimension equal to $|I|$.
\begin{lemma}
\label{lemma:oc}
  Assume that a function $f_0\in L^2(\R_{\ge0},x^{d-1})$
  satisfies $\langle f_0,f\rangle_d=0$
  for any $f$ with $\supp(f)\subset D$
  and $\langle f,x^i\rangle_d=0$ for all $i\in I$.
  Then $f_0\mathbbm{1}_{x\in D}\in M_{D,I}$.
\end{lemma}
\begin{proof}
In view of the above direct-sum decomposition
of $L^2(\R_{\ge0},x^{d-1})$, the subspace $M_{D,I}$
is regarded as the restriction of 
\begin{align}
  \bar{M}_{D,I}=\mathop{\mathrm{span}}\{x^i\mathbbm{1}_{x\in D},i\in I\}
  \oplus L^2(D^c,x^{d-1})
\end{align}
to $L^2(D,x^{d-1})$. 
The orthogonal complement of $\bar{M}_{D,I}$ in $L^2(\R_{\ge0},x^{d-1})$
is given, via Lemma~\ref{lemma:orthcomp} below, by
\begin{align}
\begin{split}
  \bar{M}_{D,I}^\perp
  &=(\mathop{\mathrm{span}}\{x^i\mathbbm{1}_{x\in D},i\in I\})^\perp
  \cap L^2(D^c,x^{d-1})^\perp\\
  &=\{f\in L^2(\R_{\ge0},x^{d-1}):\langle f,x^i\rangle_{D,d}=0
  \;\text{for all}\;i\in I\}\cap
  \{f:\supp(f)\in D\}\\
  &=\{f\in L^2(\R_{\ge0},x^{d-1}):
  \supp(f)\subset D,\ \langle f,x^i\rangle_d=0
  \;\text{for all}\;i\in I\}.
\end{split}
\end{align}

From the assumption of this lemma,
the function $f_0$ belongs to the orthogonal complement
$\bar{M}_{D,I}^{\perp\perp}$ of $\bar{M}_{D,I}^{\perp}$.
On the other hand, $\bar{M}_{D,I}$ is a closed subspace
in $L^2(\R_{\ge0},x^{d-1})$
thanks to~\cite[Theorem 1.42]{Rudin1991},
since $M_D$ is finite-dimensional.
One can therefore conclude that $\bar{M}_{D,I}^{\perp\perp}=\bar{M}_{D,I}$
holds, implying $f_0\in\bar{M}_{D,I}$, and the statement of the lemma follows. 
\end{proof}
\begin{lemma}
  \label{lemma:orthcomp}
  Let $H$ be a Hilbert space, and let $A,B$ be
  subsets of $H$ satisfying $0\in A\cap B$.
  Then $(A+B)^\perp=A^\perp\cap B^\perp$ holds.
\end{lemma}
\begin{proof}
  Since $0\in B$, one has $A\subset A+B$ and hence $(A+B)^\perp\subset A^\perp$.
  One similarly has $(A+B)^\perp\subset B^\perp$,
  and thus $(A+B)^\perp\subset A^\perp\cap B^\perp$.
  On the other hand, take $x\in A^\perp\cap B^\perp$ and $y\in A+B$.
  Noting that $y$ is represented as $y=a+b$ with $a\in A$ and $b\in B$,
  one has
  \begin{align}
    \langle x,y\rangle
    =\langle x,a\rangle+\langle x,b\rangle
    =0,
  \end{align}
  where the final equality is due to $x\in A^\perp$ and $x\in B^\perp$,
  implying that $x\in(A+B)^\perp$ holds.
  This proves $A^\perp\cap B^\perp\subset(A+B)^\perp$. 
\end{proof}
\noindent%
It should be noted that the above lemma holds
regardless of the dimensionality or the closedness of $A,B$.

\subsubsection{Prove that the optimal solution is a polynomial kernel}
The following lemma is an RK extension of \citep[Theorem]{granovsky1989optimality},
and it states the functional form that an optimal solution of Problem~\ref{problem:Prob4} should take.
\begin{lemma}
\label{lemma:GraThm}
For $d\ge1$ and even $q\ge2$,
a solution of Problem~\ref{problem:Prob4} takes a form of $\bar{p}_{q+1,\tau}(\cdot)\coloneq p_{q+1}(\cdot)\mathbbm{1}_{\cdot\le\tau}$,
where $\tau>0$ is decided from a constant $C$ in Problem~\ref{problem:Prob4}, 
and where $p_{q+1}$ is a polynomial function with terms of degree in $I_{q+1}$
and $\bar{p}_{q+1,\tau}$ has no discontinuity (i.e., $p_{q+1}(\tau)=0$).
\end{lemma}
\begin{proof}[Proof of Lemma~\ref{lemma:GraThm}]
In this proof, we work with Problem~\ref{problem:Prob4}.
To briefly represent classes of kernel functions that satisfy various combinations of the conditions, 
we introduce a new notation: we let $\calC_{i_1,\ldots,i_j}$ denote a class of functions 
satisfying the conditions $(\mathrm{P5\text{-}}i_1),\ldots,(\mathrm{P5\text{-}}i_j)$ in Problem~\ref{problem:Prob4}.
$C_4$ equals to $L^2(\R_{\ge0},x^{d-1})$.
Moreover, we use an abbreviation $\calB\coloneq\calC_{1,2,3,4,5,6}$.

\paragraph{Step~1}
%
Define $V=\inf_{f\in\calB}V_{d,0}(f)$, 
and let $\{f_n\in\calB\}$ be a sequence of functions satisfying 
$V_{d,0}(f_n)-V<\epsilon_n$ with positive numbers $\{\epsilon_n\}$ satisfying $\epsilon_n\to0$ as $n\to\infty$.
The function set $\{f\in\calC_4:V_{d,0}(f)\le\text{const}\}$ is weakly compact 
as a bounded set in the Hilbert space $\calC_4$ equipped with 
the scalar product $\naiseki{\cdot,\cdot}_d$.
It is thus found that there exists a subsequence of $f_n$ 
weakly converging to a function $f_*\in\calC_4$.
Note that this convergence particularly shows $\naiseki{f_n,f_*}_d\to\naiseki{f_*,f_*}_d$,
and one has
\begin{align}
	\naiseki{f_*,f_*}_d^2
	=\lim_{n\to\infty}\naiseki{f_*,f_n}_d^2
	\le\lim_{n\to\infty}\naiseki{f_n,f_n}_d \naiseki{f_*,f_*}_d
	=V\naiseki{f_*,f_*}_d,
\end{align}
which implies $\naiseki{f_*,f_*}_d=V$.

In the rest of Step~1, we show that $f_*$ satisfies \hyl{P5-1}, \hyl{P5-3}, and \hyl{P5-6}:
$f_n\in\calC_{1,6}$ implies that
\begin{align}
	\int_0^\infty x^{d-1+i}f_n(x)\,dx=c_i,\quad i\in I_{q+1},
\end{align}
where $c_i$, $i\in I_{q+1}$ are $i$-dependent constants.
Introducing $g_n(x)\coloneq-\int_x^\infty f_n(t)\,dt$ and a $x$-independent constant $c>0$,
from the condition \hyl{P5-5} for $G=g_n$ (that is, $|g_n(x)|=o(x^{-(d+q+\delta)})$),
one has
\begin{align}
\begin{split}
	&\lim_{M\to\infty}\biggl|\int_M^\infty x^{d-1+i}f_n(x)\,dx\biggr|
	\le\lim_{M\to\infty}\biggl\{\biggl|\biggl[x^{d-1+i}g_n(x)\biggr]_M^\infty\biggr| 
	+\int_M^\infty (d-1+i)x^{d-2+i}|g_n(x)|\,dx\biggr\}\\
	&\le \lim_{M\to\infty}c\int_M^\infty x^{-(q+2+\delta-i)}\,dx
	\le \lim_{M\to\infty}c\int_M^\infty x^{-(1+\delta)}\,dx=0,\quad i\in I_{q+1}.
\end{split}
\end{align}
Because $f_n$ converges to $f_*$ and 
$x^i\mathbbm{1}_{x\le M}\in\calC_4$, 
one has
\begin{align}
	\int_0^M x^{d-1+i}f_*(x)\,dx
	=\lim_{n\to\infty}\int_0^M x^{d-1+i}f_n(x)\,dx
	=\lim_{n\to\infty}\left\{ c_i - \int_M^\infty x^{d-1+i}f_n(x)\,dx\right\}.
\end{align}
From these results, one has
\begin{align}
	\int_{\R_{\ge0}} x^{d-1+i}f_*(x)\,dx
	=\lim_{M\to\infty} \int_0^M x^{d-1+i}f_*(x)\,dx
	=c_i,\quad i\in I_q,
\end{align}
which implies that $f_*$ satisfies \hyl{P5-1} and \hyl{P5-6} for $F=f_*$.
Moreover, since $f_n\in\calN_{\frac{q}{2}-1}$ it also holds that $f_*\in\calN_{\frac{q}{2}-1}$ \hyl{P5-3}, 
otherwise $f_*$ and $f_n$ have a interval (say $S$) where their signs do not coincide 
(so $\naiseki{\mathbbm{1}_{x\in S},f_n}_d\neq\naiseki{\mathbbm{1}_{x\in S},f_*}_d$ with $\mathbbm{1}_{x\in S}\in\calC_4$), 
which contradicts that $f_n$ converges to $f_*$.

\paragraph{Step~2}
On the basis of the fact that $f_*\in\calC_{1,3,4,6}$ proved in Step~1,
we here show a functional form of $f_*$.
We introduce the function class 
\begin{align}
\label{eq:PerMom}
	\calF=\calC_4\cap\{f:B_{d,i}(f)=0\;\text{for all}\;i\in I_{q+1}\}.
\end{align}
Then, one has that $f_*+\delta f\in\calC_{1,3,4,6}$ 
for $\delta$ with a sufficiently small absolute value and 
any bounded non-identically-zero function $f\in\calF$ satisfying $\supp(f)\subset D_*$,
where $\supp(f)$ denotes the support of $f$ and $D_*=\supp(f_*)$.
For the optimality of $f_*$, $\naiseki{f_*,f}_d=0$ for all $f\in\calF$ is necessary,
because $|\delta|\gg\delta^2$, 
\begin{align}
	V_{d,0}(\bar{f})-V_{d,0}(f_*)
	=2\delta\naiseki{f_*,f}_d+\delta^2 V_{d,0}(f)\ge0
\end{align}
and $f_*$ is not identically zero since $V=V_{d,0}(f_*)$ is away from 0 due to the proof by 
contradiction with the condition \hyl{P5-3} for $F=f_*$ and Lemma~\ref{lemma:MinCro}.
This implies that $f_*$ should lie in the orthogonal complement of
$\calF\cap\{f:\supp(f)\subset D_*\}$, that is, 
$f_*\in\mathop{\mathrm{span}}\{x^i\mathbbm{1}_{x\in D_*},i\in I_{q+1}\}$, 
which is proved by applying Lemma~\ref{lemma:oc}.
The result $f_*\in\mathop{\mathrm{span}}\{x^i\mathbbm{1}_{x\in D_*},i\in I_{q+1}\}$ 
implies that $f_*$ takes a form of $p_{q+1}(x)\mathbbm{1}_{x\in D_*}$, 
where $p_{q+1}$ is a polynomial with terms of degree in $I_{q+1}$.
Also, the boundedness of $V_{d,0}(f_*)$,
in addition to the functional form of $f_*$ derived above,
concludes that the support $D_*$ of $f_*$ is bounded.

\paragraph{Step~3}
Hereafter we show that an optimal solution $f_*$ has no discontinuity 
and that, for a polynomial $p_{q+1}$ satisfying $p_{q+1}(\tau)=0$, 
$f_*$ takes a form
\begin{align}
\label{eq:LEMB5}
	f_*(x)=\bar{p}_{q+1,\tau}(x)\coloneq p_{q+1}(x)\mathbbm{1}_{x\le\tau}.
\end{align}
Then, we prove the following lemma.
\begin{lemma}
	\label{lemma:step3}
	Let $S$ be a subset of $\R_{\ge0}$ 
	with $\mu(S)\in(0,\infty)$, where $\mu$ denotes the Lebesgue measure 
	on $\R_{\ge0}$.
	Let $f_*$ be a function defined in \eqref{eq:LEMB5} and $D_*$ be the support of $f_*$. 
	Assume that $p_{q+1}(x)$ takes the same sign over $S$, 
	and that $f_*(\cdot)$ and $p_{q+1}(\cdot)\mathbbm{1}_{\cdot\in S\cup D_*}$ 
	share the same sign-change pattern. 
	Then either $\mu(S\cap D_*)$ or $\mu(S\cap D_*^c)$ is zero.
\end{lemma}
\begin{proof}[Proof of Lemma~\ref{lemma:step3}]
	Assume to the contrary that both $\mu(S\cap D_*)$ and $\mu(S\cap D_*^c)$ 
	are strictly positive. 
	Since $\mu(S\cap D_*)>0$, one can take an ordered partition 
	$\{P_j\}_{j=1,\ldots,k+1}$ of $S\cap D_*$ with $\mu(P_j)\in(0,\infty)$ for all $j=1,\ldots,k+1$, 
	with which, for $x\in P_i$ and $y\in P_j$ with $i<j$ 
	one has $x<y$ almost surely. 
	We let $\rmM=(m_{ij})$ and $\bm{b}=(b_1,\ldots,b_{k+1})^\top$ with 
	\begin{align}
		&m_{ij}\coloneq\int_{P_j}x^{d+2i-2}f_*(x)\,dx,\quad
		i,j=1,2,\ldots,k+1,\\
		&b_i\coloneq\int_{S\cap D_*^c}x^{d+2i-2}\,dx,\quad
		i=1,2,\ldots,k+1.
	\end{align}
	One has $0<b_i<\infty$ due to the assumption $0<\mu(S\cap D_*^c)\le\mu(S)<\infty$. 
	Furthermore, let $\bm{c}=(c_1,\ldots,c_{k+1})^\top$ and define 
	\begin{align}
		f_{\bm{c}}(x)=\sum_{j=1}^{k+1}c_jp_{q+1}(x)\mathbbm{1}_{x\in P_j}.
	\end{align}
	Note that $f_{\bm{c}}$ is equivalent to $f_*$ on $S$ 
	when $\bm{c}=\bm{1}_{k+1}$.

	We now want to set $\bm{c}$ so that $f_{\bm{c}}(x)$ and 
	the function $\mathbbm{1}_{x\in S\cap D_*^c}$ 
	share the same moments of orders in $I_{q+1}$. 
	It will be achieved if $\bm{c}$ satisfies $\rmM\bm{c}=\bm{b}$, 
	since the $(2i-1)$-st moments of $f_{\bm{c}}(x)$ 
	and $\mathbbm{1}_{x\in S\cap D_*^c}$ are $\sum_{j=1}^{k+1}m_{ij}c_j$ 
	and $b_i$, respectively. 
	Applying Lemma~\ref{lemma:nonsingular} 
	with $w(x)=x^df_*(x)$ shows that $\rmM$ is invertible, 
	so that the assertion $\bm{c}=\rmM^{-1}\bm{b}$ 
	makes the moments of $f_{\bm{c}}(x)$ 
	and $\mathbbm{1}_{x\in S\cap D_*^c}$ to coincide.

	With the above $\bm{c}$, 
	let $f(x)=s\cdot (\mathbbm{1}_{x\in S\cap D_*^c}-f_{\bm{c}}(x))$, 
	where $s\in\{-1,1\}$ is the sign of $p_{q+1}$ on $S$. 
	One has $f\in\calF$ by construction, 
	which implies $\langle p_{q+1},f\rangle_d=0$. 
	On the other hand, one has 
	\begin{align}
		\int_{S\cap D_*^c}x^{d-1}p_{q+1}(x)f(x)\,dx
		=\int_{S\cap D_*^c}x^{d-1}p_{q+1}(x)s\,dx>0.
	\end{align}
	Since $f$ vanishes outside $S$, we have 
	\begin{align}
		\langle f_*,f\rangle_d
		=\int_{S\cap D_*}x^{d-1}p_{q+1}(x)f(x)\,dx
		=\langle p_{q+1},f\rangle_d-\int_{S\cap D_*^c}x^{d-1}p_{q+1}(x)s\,dx<0.
	\end{align}
	For a sufficiently small $\delta>0$, the function $f_*+\delta f$ 
	has the same sign $s$ as $p_{q+1}$ on $S$, 
	and it is equal to $f_*$ outside $S$. 
	Therefore, $f_*$ and $f_*+\delta f$ 
	share the same sign-change pattern, 
	and one consequently has $f_*+\delta f\in\calC_{1,3,4,6}$.
	The above inequality $\langle f_*,f\rangle_d<0$ 
	implies that for a sufficiently small $\delta>0$ 
	one has $V_{d,0}(f_*+\delta f)<V_{d,0}(f_*)$, 
	leading to contradiction. 
\end{proof}
\noindent%
We consider, as in Definition \ref{definition:SigChe}, 
subintervals partitioning $\R_{\ge0}$ according to the sign of $p_{q+1}$. 
Since $p_{q+1}$ is a polynomial of degree $(q+1)$ 
with odd-degree terms only, 
the number of the above subintervals is at most $(\frac{q}{2}+1)$. 
Lemma~\ref{lemma:step3} states that, for each of the subintervals, 
the support $D_*$ of $f_*$ either contains its interior as a whole or not, 
and cannot contain only some fraction of it in terms of the Lebesgue measure.
Since $f_*\in\calN_{\frac{q}{2}-1}$ and $D_*$ is bounded, 
the only possibility is therefore that the closure of $D_*$ 
equals to $[0,\tau]$, where $\tau$ is the largest zero of $p_{q+1}$, 
that is, $f_*=\bar{p}_{q+1,\tau}$.

In Lemma~\ref{lemma:GraLem2} below, 
we will show that such a function $f_*$ uniquely exists
and that function also satisfies the conditions \hyl{P5-2} and \hyl{P5-5}
(besides the conditions confirmed in Step~1).
Thus we conclude that function of the shape $\bar{p}_{q+1,\tau}$ is a solution of Problem~\ref{problem:Prob4}.
\end{proof}

\subsubsection{Prove that the optimal polynomial kernel is uniquely determined}
The following lemma is an RK extension of \citep[Lemma~2]{granovsky1989optimality},
and shows that a polynomial function suggested by the above lemma is uniquely determined.
\begin{lemma}
\label{lemma:GraLem2}
For $d\ge1$ and even $q\ge2$,
there exists a unique constant $C\neq0$ and a unique polynomial $p_{q+1}$ with terms of degree in $I_{q+1}$, 
such that $\bar{p}_{q+1,1}(\cdot)\coloneq p_{q+1}(\cdot)\mathbbm{1}_{\cdot\le1}$ satisfies
\begin{align}
\label{eq:LemReq}
	\bar{p}_{q+1,1}\in\calM_{d,1,q+1},\quad
	\int_{\R_{\ge0}}x^{d+q}\bar{p}_{q+1,1}(x)\,dx=C,\quad
	p_{q+1}(1)=0.
\end{align}
Also, the resulting polynomial $p_{q+1}$ is such that $\bar{p}_{q+1,1}=G_{d,q}^{B(1)}$.
\end{lemma}
\noindent%
Here, suppose that $C$ is given such that $\tau$ gets to 1 for the sake of simplicity of description, 
since $\tau$ in $\bar{p}_{q+1,\tau}$ and $|C|$ in \hyl{P5-6} have a one-to-one correspondence
and it is possible, as can be seen from Lemma~\ref{lemma:equiv23} on the scale transformation.
Then, we will show that $C$ satisfying the conditions (instead of $\tau$) is uniquely determined.
\begin{proof}[Proof of Lemma~\ref{lemma:GraLem2}]
For odd $j\ge1$, 
we introduce one of the normalized Jacobi polynomials on $[0,1]$: 
\begin{align}
	Q_j(x)=C_{d,j} x P_{\frac{j-1}{2}}^{(\frac{d}{2},0)}(1-2x^2),
\end{align}
where $C_{d,j}=\sqrt{d+2j}$ is a normalization constant,
and define $\bar{Q}_j(x)=Q_j(x)\mathbbm{1}_{x\le1}$.
It is straightforward to show that
$\{\bar{Q}_j\}_{j=1,3,\ldots}$ satisfies
the orthogonality relation 
\begin{align}
  \label{eq:JacobiPOrth}
\begin{split}
	\langle\bar{Q}_j,\bar{Q}_k\rangle_d
	&=\int_0^1 x^{d-1} Q_j(x)Q_k(x)\,dx\\
  &=C_{d,j}C_{d,k}\int_0^1x^{d+1}P_{\frac{j-1}{2}}^{(\frac{d}{2},0)}(1-2x^2)
  P_{\frac{k-1}{2}}^{(\frac{d}{2},0)}(1-2x^2)\,dx\\
  &=\frac{C_{d,j}C_{d,k}}{2^{\frac{d}{2}+2}}
  \int_{-1}^1(1-t)^{\frac{d}{2}}
  P_{\frac{j-1}{2}}^{(\frac{d}{2},0)}(t)
  P_{\frac{k-1}{2}}^{(\frac{d}{2},0)}(t)\,dt\\
  &=\mathbbm{1}_{j=k}
	\quad\text{for odd}\;j, k\ge1,
\end{split}
\end{align}
where the last equality follows from the orthogonality relation
of the Jacobi polynomials:
\begin{align}
\int_{-1}^1(1-t)^\alpha(1+t)^\beta
P_m^{(\alpha,\beta)}(t)P_n^{(\alpha,\beta)}(t)\,dt
=\frac{2^{\alpha+\beta+1}}{2n+\alpha+\beta+1}
\frac{\Gamma(n+\alpha+1)\Gamma(n+\beta+1)}{\Gamma(n+\alpha+\beta+1)\Gamma(n+1)}
\mathbbm{1}_{m=n}.
\end{align}
Since the Jacobi polynomial $P_m^{(\alpha,\beta)}$
is a polynomial of degree $m$,
$Q_j$ is a polynomial with terms of degree in $I_j$,
so that one can write 
\begin{align}
	Q_j(x)=\sum_{i\in I_j} c_{j,i} x^i,
\end{align}
where $c_{j,j}\neq0$.
Furthermore, since any polynomial of degree $m$ is represented
as a linear combination of $\{P_n^{(\alpha,\beta)}\}_{n=0,1,\ldots,m}$,
any polynomial with terms of degree in $I_k$ for odd $k$ 
is represented as a linear combination of $\{Q_j\}_{j=1,3,\ldots,k}$.
This in particular means that, for odd $j\ge3$, $x^i$ with $i\in I_{j-2}$
is represented as a linear combination of $\{Q_j\}_{j=1,3,\ldots,i}$,
so that the orthogonality~\eqref{eq:JacobiPOrth}
implies that $\bar{Q}_j$ satisfies the moment conditions
$B_{d,i}(\bar{Q}_j)=0$, $i\in I_{j-2}$.
%


%
Any polynomial with terms of degree in $I_k$
can be represented as a linear combination of $Q_j$, $j=1,3,\ldots,k$.
So, the polynomial in the statement of the theorem can be written as
\begin{align}
	p_{q+1}(x)=\sum_{j\in I_{q+1}} w_j Q_j(x).
\end{align}
The coefficient $w_j$ is evaluated,
via the moment conditions of $\bar{p}_{q+1,1}$, as 
\begin{align}
\begin{split}
	w_j
	&=\int_0^1 x^{d-1} p_{q+1}(x)Q_j(x)\,dx 
	=\int_0^1 x^{d-1} p_{q+1}(x)\biggl(\sum_{i\in I_j} c_{j,i} x^i\biggr)\,dx\\
	&=\sum_{i\in I_j} c_{j,i} B_{d,i}(\bar{p}_{q+1,1})
	=\begin{cases}
	c_{j,1}\bar{C},&j\in I_{q-1},\\
	c_{q+1,1}\bar{C}+c_{q+1,q+1}C,&j=q+1,
	\end{cases}
\end{split}
\end{align}
where $\bar{C}\coloneq B_{d,1}(\bar{p}_{q+1,1})=-db_d^{-1}<0$.
Therefore, $p_{q+1}$ has a unique representation,
\begin{align}
\label{eq:unique}
	p_{q+1}(x)=\bar{C}\sum_{j\in I_{q+1}} c_{j,1} Q_j(x)+ C c_{q+1,q+1}Q_{q+1}(x),
\end{align}
where $C$ is variable and other components are fixed.
Considering additionally $p_{q+1}(1)=0$ determines $C$ uniquely.

One can calculate \eqref{eq:unique} to show $\bar{p}_{q+1,1}=G_{d,q}^{B(1)}$, 
but we instead confirmed that $G_{d,q}^B$ satisfies the equivalent requirements 
(\hyl{P3-1}--\hyl{P3-5}) in th following Lemma~\ref{lemma:Confirm2}.
On the basis of the uniqueness, this completes the proof of the latter half of the statement.
\end{proof}

\begin{lemma}
\label{lemma:Confirm2}
The kernel $G_{d,q}^B$ satisfies all the conditions \hyl{P3-1}--\hyl{P3-5} with $G=G_{d,q}^B$.
\end{lemma}
\begin{proof}
The only difference between \hyl{P2-1}--\hyl{P2-5} and \hyl{P3-1}--\hyl{P3-5} is in \hyl{P2-3} and \hyl{P3-3}.
One can show that $G_{d,q}^B$ satisfies \hyl{P3-3}
in a similar argument as the proof of Lemma~\ref{lemma:Confirm2}.
\end{proof}

According to these lemmas,
one can see that Problem~\ref{problem:Prob4} gives optimal solution $-\int_x^\infty G_{d,q}^{B(1)}(t)\,dt=G_{d,q}^B(x)$,
whose integral constant is determined from the end-point condition \hyl{P5-5}.
This is also a solution of Problem~\ref{problem:Prob22} due to the equivalence between the problems.
This completes the proof of Theorem~\ref{theorem:thm22}.

\section{Supplemental Information on Appeared Kernels}
\label{section:Kernel-information}
\setcounter{equation}{0}
\setcounter{definition}{0}
\setcounter{theorem}{0}
\setcounter{proposition}{0}
\setcounter{lemma}{0}
\setcounter{problem}{0}
\setcounter{remark}{0}
Here, we provide supplemental information on the kernels used in our simulation experiments
in Section~\ref{section:Simulation}.
More specifically, we give the expressions of $B_{d,q}$, $V_{d,0}$, and $V_{d,1}$ for those kernels.

\begin{proposition}
\label{proposition:B-kernel}
For the kernel $G_{d,q}^B$, 
the functionals $B_{d,q}$, $V_{d,0}$, and $V_{d,1}$ become
\begin{align}
\begin{split}
	&B_{d,q}(G_{d,q}^B)%
	=(-1)^{\frac{q}{2}+1}
	\frac{\Gamma(\frac{d+q}{2})\Gamma(\frac{d+q+4}{2})}
	{2\pi^{\frac{d}{2}}\Gamma(\frac{d+2q+4}{2})},\quad
	V_{d,0}(G_{d,q}^B)%
	=\frac{8(3d+4q+4)\Gamma(\frac{d+q}{2})\Gamma(\frac{d+q+4}{2})}
	{\pi^dd(d+2q)_{2,3}\Gamma(\frac{q}{2})^2},\\
	&V_{d,1}(G_{d,q}^B)%
	=\frac{16\Gamma(\frac{d+q+2}{2})\Gamma(\frac{d+q+4}{2})}
	{\pi^d(d+2)(d+2q+2)\Gamma(\frac{q}{2})^2}.
\end{split}
\end{align}
\end{proposition}

The following is on the information about the kernels $K_{d,q}^E$, $K_{d,q}^G$, and $K_{d,q}^L$, 
which are designed as higher-order extensions of Epanechnikov, Gaussian, Laplace kernels
according to the jackknife \citep{schucany1977improvement, wand1990gaussian}.
These kernels respectively belong to Epanechnikov, Gaussian, and Laplace hierarchies, 
and take the forms of product between polynomials and Epanechnikov, Gaussian, and Laplace kernels.

\begin{proposition}
\label{proposition:E-kernel}
The kernel $K_{d,q}^E(\bx)=G_{d,q}^E(\|\bx\|)$
is a $d$-variate, $q$-th order kernel, where
\begin{align}
\begin{split}
	G_{d,q}^E(u)
	&=\frac{\Gamma(\frac{d+q}{2})}{\pi^{\frac{d}{2}}\Gamma(\frac{q}{2})}
	P_{\frac{q}{2}}^{(\frac{d}{2},-1)}(1-2u^2)\mathbbm{1}_{u\le1}
	=(-1)^{\frac{q}{2}+1}
	\frac{\Gamma(\frac{d+q}{2}+1)}{\pi^{\frac{d}{2}}\Gamma(\frac{q}{2}+1)}
	(1-u^2)P_{\frac{q}{2}-1}^{(1,\frac{d}{2})}(2u^2-1)\mathbbm{1}_{u\le1}\\
	&=(-1)^{\frac{q}{2}+1}
	\frac{\Gamma(\frac{d+q}{2}+1)}{\Gamma(\frac{d}{2}+2)\Gamma(\frac{q}{2}+1)}
	P_{\frac{q}{2}-1}^{(1,\frac{d}{2})}(2u^2-1)
	\cdot G_{d,2}^E(u),
\end{split}
\end{align}
and the functionals $B_{d,q}$, $V_{d,0}$, and $V_{d,1}$ of $G_{d,q}^E$ become
\begin{align}
\begin{split}
 	&B_{d,q}(G_{d,q}^E)%
	=(-1)^{\frac{q}{2}+1}
	\frac{\Gamma(\frac{d+q}{2})\Gamma(\frac{d+q+2}{2})}
	{2\pi^{\frac{d}{2}}\Gamma(\frac{d+2q+2}{2})},\quad
	V_{d,0}(G_{d,q}^E)%
	=\frac{4\Gamma(\frac{d+q}{2})\Gamma(\frac{d+q+2}{2})}
	{\pi^dd(d+2q)\Gamma(\frac{q}{2})^2},\\
	&V_{d,1}(G_{d,q}^E)%
	=\frac{4\Gamma(\frac{d+q+2}{2})^2}
	{\pi^d(d+2)\Gamma(\frac{q}{2})^2}.
\end{split}
\end{align}
\end{proposition}

\begin{proposition}
\label{proposition:G-kernel}
The kernel $K_{d,q}^G(\bx)=G_{d,q}^G(\|\bx\|)$
is a $d$-variate, $q$-th order kernel, where
\begin{align}
	G_{d,2}^G(u)
	=(2\pi)^{-\frac{d}{2}} \exp(-u^2/2),\quad
	G_{d,q}^G(u)
	=L_{\frac{q}{2}-1}^{(\frac{d}{2})}\left(\frac{u^2}{2}\right)G_{d,2}^G(u),
\end{align}
where $L_n^{(\alpha)}$ is the Laguerre polynomial \citep{szeg1939orthogonal}.
Also, the functionals $B_{d,q}$, $V_{d,0}$, and $V_{d,1}$ of $G_{d,q}^G$ become 
\begin{align}
\begin{split}
	&B_{d,q}(G_{d,q}^G)%
	=(-2)^{\frac{q}{2}-1}\frac{\Gamma(\frac{d+q}{2})}{\pi^{\frac{d}{2}}},\\
	&V_{d,0}(G_{d,q}^G)
	=\frac{\{\Gamma(\frac{d+q}{2})\}^2}{2^{d+q-1}\pi^d}
	\sum_{i_1,i_2=0}^{q/2-1}(-2)^{i_1+i_2}
	\frac{\Gamma(\frac{d}{2}+q-2-i_1-i_2)}{\Gamma(\frac{d+q}{2}-i_1)\Gamma(\frac{d+q}{2}-i_2)
	\Gamma(\frac{q}{2}-i_1)\Gamma(\frac{q}{2}-i_2)i_1!i_2!},\\
	&V_{d,1}(G_{d,2}^G)%
	=\frac{\Gamma(\frac{d}{2}+2)}{(d+2)2^{d}\pi^d},\quad
	V_{d,1}(G_{d,4}^G)%
	=\frac{(d+10)\Gamma(\frac{d+6}{2})}{(d+2)2^{d+3}\pi^d},\\
	&V_{d,1}(G_{d,6}^G)
	=\frac{(d^2+26d+176)\Gamma(\frac{d+8}{2})}{(d+2)2^{d+8}\pi^d}.
\end{split}
\end{align}
\end{proposition}

\begin{proposition}
\label{proposition:L-kernel}
The kernel $K_{d,q}^L(\bx)=G_{d,q}^L(\|\bx\|)$
is a $d$-variate, $q$-th order kernel, where
\begin{align}
	G_{d,2}^L(u)
	=\frac{\Gamma(\frac{d}{2})}{2\pi^{\frac{d}{2}}\Gamma(d)}\exp(-u),\quad
	G_{d,q}^L(u)
	=\biggl(\sum_{i\in I_{q-2}}b_{d,q,i}^L u^i\biggr)G_{d,2}^L(u),
\end{align}
and where $b_{d,q,i}^L$, $i\in I_{q-2}$, are determined as
\begin{align}
	\begin{bmatrix}
	b_{d,q,0}^L\\ b_{d,q,2}^L\\\vdots\\ b_{d,q,q-2}^L\\
	\end{bmatrix}=
	\begin{bmatrix}
	1&(d)_{2}&\cdots&(d)_{q-2}\\
	1&(d+2)_{2}&\cdots&(d+2)_{q-2}\\
	\vdots&\vdots&\ddots&\vdots\\
	1&(d+q-2)_{2}&\cdots&(d+q-2)_{q-2}\\
	\end{bmatrix}^{-1}
	\begin{bmatrix}
	1\\0\\\vdots\\0\\
	\end{bmatrix}.
\end{align}
Also, the functionals $B_{d,q}$, $V_{d,0}$, and $V_{d,1}$ of $G_{d,q}^L$ become
\begin{align}
\begin{split}
	&B_{d,2}(G_{d,2}^L)%
	=\frac{(d)_2\Gamma(\frac{d}{2})}{2\pi^{\frac{d}{2}}};\quad
	B_{d,4}(G_{d,4}^L)%
	=-\frac{(d)_4(2d+7)\Gamma(\frac{d}{2})}{2(2d+3)\pi^{\frac{d}{2}}};\\
	&B_{d,6}(G_{d,6}^L)%
	=\frac{196200}{149}, \frac{2215080}{307\pi}, \frac{1335600}{109\pi},\;d=1,2,3;\\
	&V_{d,0}(G_{d,2}^L)%
	=\frac{\Gamma(\frac{d}{2})^2}{2^{d+2}\pi^d\Gamma(d)};\quad
	V_{d,0}(G_{d,4}^L)%
	=\frac{(d+2)_2 (9d^2+73d+96)\Gamma(\frac{d}{2})^2}{2^{d+8}\pi^d(2d+3)^2\Gamma(d)};\\
	&V_{d,0}(G_{d,6}^L)%
	=\frac{4327215}{22733824}, \frac{28174335}{193021952\pi^2}, \frac{5292791}{389316608\pi^2},\;d=1,2,3;\\
	&V_{d,1}(G_{d,2}^L)%
	=\frac{\Gamma(\frac{d}{2})^2}{2^{d+2}\pi^d\Gamma(d)};\quad
	V_{d,1}(G_{d,4}^L)%
	=\frac{(d+1)(9d^3+133d^2+534d+576)\Gamma(\frac{d}{2})^2}{2^{d+8}\pi^d(2d+3)^2\Gamma(d)};\\
	&V_{d,1}(G_{d,6}^L)%
	=\frac{5572345}{22733824}, \frac{38442103}{193021952\pi^2}, \frac{7306271}{389316608\pi^2},\;d=1,2,3.
\end{split}
\end{align}
\end{proposition}

\section{Proof of Proposition~\protect\ref{prop:Prop6}}
\label{appendix:ProofProp6}
\setcounter{equation}{0}
\setcounter{definition}{0}
\setcounter{theorem}{0}
\setcounter{proposition}{0}
\setcounter{lemma}{0}
\setcounter{problem}{0}
\setcounter{remark}{0}
This appendix provides a proof of Proposition~\ref{prop:Prop6}.
In the case $(d,q)=(2,2)$,
the third-order term of the Taylor expansion of $f$
around its mode $\bm{x}=\bm{\theta}$ is represented
as a real binary cubic form, i.e.,
a homogeneous degree-3 polynomial of $\bm{u}=\bm{x}-\bm{\theta}$:
$au_1^3+bu_1^2u_2+cu_1u_2^2+du_2^3$,
where $a,b,c,d\in\R$ are proportional to
appropriate third-order partial derivatives of $f$ at $\bm{\theta}$.
Properties of such a real binary cubic form that are
invariant under reversible linear transforms of $\bm{u}$
is characterized by the algebraic curve defined as
the set of $\bm{u}$ satisfying $au_1^3+bu_1^2u_2+cu_1u_2^2+du_2^3=0$.
Putting aside the trivial case $a=b=c=d=0$,
classification of the algebraic curve $au_1^3+bu_1^2u_2+cu_1u_2^2+du_2^3=0$
can be done on the basis of classification of the set of solutions
of the cubic equation $at^3+bt^2+ct+d=0$ of a real-valued variable $t$:
it has either:
\begin{itemize}
\item a single real solution, which is triply degenerate,
\item two real distinct solutions, one of which is doubly degenerate,
\item three real distinct solutions, or
\item a real and two complex solutions.
\end{itemize}
Accordingly, the third-order term of the Taylor expansion of $f$
at the mode is either:
\begingroup
\renewcommand{\theenumi}{(\Alph{section}.\arabic{enumi})}
\renewcommand{\labelenumi}{\theenumi}
\begin{enumerate}\setlength{\leftskip}{0.3cm}\setlength{\parskip}{0cm}\setlength{\itemsep}{0cm}
\item 0 identically,
\item of the form $(\bm{a}^\top\bm{u})^3$,
\item of the form $(\bm{a}_1^\top\bm{u})^2(\bm{a}_2^\top\bm{u})$
  with $\bm{a}_1,\bm{a}_2$ linearly independent,
\item of the form $\prod_{i=1}^3\bm{a}_i^\top\bm{u}$
  with $\bm{a}_1,\bm{a}_2,\bm{a}_3$, any two of which are linearly independent, or
\item of the form $(\bm{a}^\top\bm{u})g(\bm{u})$,
  where $g(\bm{u})=\frac{1}{2}\bm{u}^\top\mathrm{R}\bm{u}$
  with $\mathrm{R}$ positive definite.
\end{enumerate}
\endgroup
\addtocounter{equation}{5}
In the case (D.1), the AB is trivially 0 for any choice of $\rmQ$.
The cases (D.4) and (D.5) have already been covered
in Propositions~\ref{prop:Prop3} and \ref{prop:Prop4}, respectively.
What remains is to consider the cases (D.2) and (D.3),
which are dealt with in the following lemmas,
completing the proof of Proposition~\ref{prop:Prop6}.

\begin{lemma}
  $\nabla(\nabla^\top\rmQ\nabla)^{q/2}(\bm{a}^\top(\bm{x}-\bm{\theta}))^{q+1}$
  with $\bm{a}\not=\mathbf{0}_d$ 
  is non-zero irrespective of the choice of the positive definite $\rmQ$. 
\end{lemma}
\begin{proof}
One has
\begin{align}
  \nabla(\nabla^\top\rmQ\nabla)^{q/2}(\bm{a}^\top(\bm{x}-\bm{\theta}))^{q+1}
  =(q+1)!(\bm{a}^\top\rmQ\bm{a})^{q/2}\bm{a},
\end{align}
and the coefficient of $\bm{a}$ in the above expression
is non-zero because of the positive-definiteness of $\rmQ$.
\end{proof}

\begin{lemma}
  $\nabla(\nabla^\top\rmQ\nabla)^{q/2}(\bm{a}_1^\top(\bm{x}-\bm{\theta}))^q
  (\bm{a}_2^\top(\bm{x}-\bm{\theta})$
  with linearly independent $\bm{a}_1,\bm{a}_2$ 
  is non-zero irrespective of the choice of the positive definite $\rmQ$. 
\end{lemma}
\begin{proof}
One has
\begin{align}
  \nabla(\nabla^\top\rmQ\nabla)^{q/2}(\bm{a}_1^\top(\bm{x}-\bm{\theta}))^q
  (\bm{a}_2^\top(\bm{x}-\bm{\theta})
  =q\cdot q!(\bm{a}_1^\top\rmQ\bm{a}_1)^{q/2-1}(\bm{a}_1^\top\rmQ\bm{a}_2)\bm{a}_1
  +q!(\bm{a}_1^\top\rmQ\bm{a}_1)^{q/2}\bm{a}_2,
\end{align}
and the coefficient of $\bm{a}_2$ in the above expression
is non-zero because of the positive-definiteness of $\rmQ$.
\end{proof}

The general treatment should require a higher-order analog
of Sylvester's law of inertia for quadratic forms,
that is, classification of homogeneous degree-$q$ polynomials
of $d$ variables under the equivalence relation
where two polynomials are equivalent if and only if
one of them can be transformed into the other
via applying an invertible linear transform on $\R$
to its variables.
This problem seems to be less explored compared with
that on $\mathbb{C}$,
and is already complicated even when $d=3$: see e.g.~\cite{Banchi2015}.
We therefore do not attempt to extend Proposition~\ref{prop:Prop6}
to more general cases in this paper,
and defer the problem to a future investigation.

\section{Proof of Asymptotic Behaviors of the Modal Linear Regression}
\label{section:ProofMLR}
\setcounter{equation}{0}
\setcounter{assumption}{0}
\renewcommand{\theequation}{\Alph{section}.\arabic{equation}}
\renewcommand{\theassumption}{\Alph{section}.\arabic{assumption}}
Theorem~\ref{theorem:MLR} on asymptotic behaviors of the MLR parameter estimator
can be proved in the manner almost same as that of Theorem~\ref{theorem:AN}, 
under the following regularity conditions:

\begin{assumption}[Regularity conditions for Theorem~\ref{theorem:MLR}]
  \label{assumption:asmA.2}
For finite $d_\Y$ and even $q$,
\begin{itemize}\setlength{\leftskip}{0.3cm}\setlength{\parskip}{0cm}\setlength{\itemsep}{0cm}
\item[\hyt{\Alph{section}.1}]
	$\{(\bX_i,\bY_i)\in\calX\times\calY\}_{i=1}^n$ is a sample of i.i.d.~observations from $f_\XY$.
\item[\hyt{\Alph{section}.2}]
	$\mathrm{E}[\|\bX\|^{d_\Y+4+\tau}]<\infty$ for some $\tau>0$.
\item[\hyt{\Alph{section}.3}]
	The parameter space $\mathcal{A}$ is compact subset of $\R^{d_\Y\times d_\X}$,
	and has a nonempty interior $\mathrm{int}(\mathcal{A})$;
	$\Theta\in\mathrm{int}(\mathcal{A})$.
\item[\hyt{\Alph{section}.4}]
	$\Pr(\Omega\bX_i=0)<1$ for any fixed $\Omega\in\R^{d_\Y\times d_\X}$ such that $\Omega\neq\mathrm{0}$ (no multicollinearity).
\item[\hyt{\Alph{section}.5}]
	$f_\YX(\cdot|\bx)$ is $(q+2)$ times differentiable in $\R^{d_\Y}$ for all $\bx$.
\item[\hyt{\Alph{section}.6}]
	$f_\YX(\by|\bx)$ has a unique and isolated maximizer at $\by\neq\Theta\bx$ for all $\bx$
	(i.e., $f_\YX(\by|\bx)<f_\YX(\Theta\bx|\bx)$ for all $\by\neq\Theta\bx$, 
	$\nabla f_\YX(\Theta\bx|\bx)=\bm{0}_{d_\Y}$,
	and $\sup_{\by\in N_\bx} f_\YX(\by|\bx)<f_\YX(\Theta\bx|\bx)$ for a neighborhood $N_\bx$ of $\Theta\bx$, 
	for all $\bx$). 
\item[\hyt{\Alph{section}.7}]
	$\partial^\bi f_\YX(\cdot|\bx)$, $|\bi|=2$ satisfies $\int|\partial^\bi f_\YX(\by|\bx)|\,d\by<\infty$ for all $\bx$,
	and $\rmA$ in~\eqref{eq:MLR-abbre} is non-singular.
\item[\hyt{\Alph{section}.8}]
	$|\partial^\bi f_\YX(\Theta\bx|\bx)|<\infty$ for all $\bx$ and $\bi$ s.t.\,$|\bi|=2,\ldots,q+1$, and
	$\partial^\bi f_\YX(\cdot|\bx)$ is bounded in $\R^{d_\Y}$ for all $\bx$ and $\bi$ s.t.\,$|\bi|=q+2$.
\item[\hyt{\Alph{section}.9}]
	$\partial^\bi f_\YX$, $|\bi|=q+1$, and $K$ is such that
	$\bb$ in~\eqref{eq:MLR-abbre} is non-zero.
\item[\hyt{\Alph{section}.10}]
	$K$ is bounded and twice differentiable in $\R^{d_\Y}$
	and satisfies the covering number condition, 
	$\int |K(\by)|\,d\by<\infty$, and $\lim_{\|\by\|\to\infty}\|\by\||K(\by)|=0$.
\item[\hyt{\Alph{section}.11}]
	$\calB_{d_\Y,\bm{0}_{d_\Y}}(K)=1$.
\item[\hyt{\Alph{section}.12}]
	$\calB_{d_\Y,\bi}(K)=0$ for all $\bi$ s.t.\,$|\bi|=1,\ldots,q-1$.
\item[\hyt{\Alph{section}.13}]
	$|\calB_{d_\Y,\bi}(K)|<\infty$ for all $\bi$ s.t.\,$|\bi|=q$, and
	$\calB_{d_\Y,\bi}(K)\neq0$ for some $\bi$ s.t.\,$|\bi|=q$.
\item[\hyt{\Alph{section}.14}]
	$|\calB_{d_\Y,\bi}(K)|<\infty$ for all $\bi$ s.t.\,$|\bi|=q+1$.
\item[\hyt{\Alph{section}.15}]%
	$\partial_i K$ is bounded and satisfies $\int |\partial_i K(\by)\partial_j K(\by)|\,d\by<\infty$, 
	$\lim_{\|\by\|\to\infty}\|\by\| |\partial_i K(\by)\partial_j K(\by)|=0$, 
	and $\int |\partial_i K(\by)|^{2+\delta}\,d\by<\infty$ for some $\delta>0$,
	for all $i,j=1,\ldots,d_\Y$.
\item[\hyt{\Alph{section}.16}]
	$\int \nabla K(\by)\nabla K(\by)^\top\,d\by=\calV_{d_\Y}(K)$ has a finite determinant .
\item[\hyt{\Alph{section}.17}]
	$\partial^\bi K$, $|\bi|=2$ satisfies the covering number condition.
\item[\hyt{\Alph{section}.18}]
	$\lim_{n\to\infty}(nh_n^{d_\Y+2q+2})^{\frac{1}{2}}=c<\infty$.
\item[\hyt{\Alph{section}.19}]%
	$\lim_{n\to\infty}nh_n^{d_\Y+4}/\ln n=\infty$.
\end{itemize}
\end{assumption}

\section*{Acknowledgements}
This work was supported by Grant-in-Aid for JSPS Fellows, Number 20J23367.



\bibliography{bibtex}
\end{document}